%% file: main.tex
\pdfoutput=1
\documentclass{article}

\usepackage[final]{neurips_2024}

\input{packages}
\input{macros}

\title{S$^{2}$FT: Efficient, Scalable and Generalizable LLM Fine-tuning by Structured Sparsity}

\author{
    \vspace{-2.5em}\\
  \textbf{Xinyu Yang$^{1}$, Jixuan Leng$^{1}$, Geyang Guo$^{2}$,  Jiawei Zhao$^{3}$, Ryumei Nakada$^{4}$,} \\
  \textbf{Linjun Zhang$^{4}$, Huaxiu Yao$^{5}$, Beidi Chen$^{1}$}\\
  $^{1}$CMU, $^{2}$Georgia Tech, $^{3}$Caltech, $^{4}$Rutgers, $^{5}$UNC-Chapel Hill\\
  \texttt{xinyuya2, beidic@andrew.cmu.edu}\\\\
\vspace{-1.5em}\\
\url{https://infini-ai-lab.github.io/S2FT-Page}
}

\begin{document}

\maketitle

\input{text/abstract}

\input{text/introduction}

\input{text/observation}

\input{text/method}

\input{text/theory}

\input{text/experiment}

\input{text/analysis}

\input{text/related}

\input{text/conclusion}

\input{text/acknowledgement}

We would like to thank Songlin Yang, Kaustubh Ponkshe, Raghav Singhal, Jinqi Luo, Tianqi Chen, Hanshi Sun, and Chris De Sa for their helpful discussions, and the authors of LLM-Adapters, ReFT, and DoRA for providing detailed results.

{
\small
\bibliographystyle{plain}
\bibliography{reference}
}

\clearpage
\input{text/appendix}


\end{document}

%% file: packages.tex
\usepackage{color,xcolor}
\usepackage{epsfig}
\usepackage{graphicx}

\usepackage{adjustbox}
\usepackage{array}
\usepackage{booktabs}
\usepackage{colortbl}
\usepackage{float,wrapfig}
\usepackage{hhline}
\usepackage{multirow}
\usepackage{subcaption} 
\usepackage[toc,page]{appendix}

\usepackage{amsmath,amsfonts,amsthm,amssymb}
\usepackage{bm}
\usepackage{nicefrac}
\usepackage{microtype}
\usepackage{inconsolata}
\usepackage{pifont}

\usepackage{changepage}
\usepackage{extramarks}
\usepackage{fancyhdr}
\usepackage{lastpage}
\usepackage{setspace}
\usepackage{soul}
\usepackage{xspace}
\usepackage{indentfirst}
\usepackage{paralist}
\usepackage{booktabs,arydshln}

\usepackage[pagebackref=true,breaklinks=true,letterpaper=true,colorlinks,citecolor=citecolor,bookmarks=false]{hyperref}
\usepackage{url}

\usepackage{algorithm, algorithmic}
\usepackage{enumerate}
\usepackage{lipsum}
\usepackage{physics}
\usepackage{cleveref}

\usepackage{listings}
\usepackage{xcolor}
\usepackage[numbers]{natbib}
\usepackage{enumitem}

\definecolor{commentgray}{rgb}{0.5, 0.5, 0.5}

%% file: macros.tex
\newcolumntype{L}[1]{>{\raggedright\let\newline\\\arraybackslash\hspace{0pt}}m{#1}}
\newcolumntype{C}[1]{>{\centering\let\newline\\\arraybackslash\hspace{0pt}}m{#1}}
\newcolumntype{R}[1]{>{\raggedleft\let\newline\\\arraybackslash\hspace{0pt}}m{#1}}

\newcommand{\ignorethis}[1]{}

\DeclareMathOperator*{\argmin}{arg\,min}

\makeatletter
\DeclareRobustCommand\onedot{\futurelet\@let@token\@onedot}
\def\@onedot{\ifx\@let@token.\else.\null\fi\xspace}

\makeatother

\makeatletter
\def\adl@drawiv#1#2#3{%
        \hskip.5\tabcolsep
        \xleaders#3{#2.5\@tempdimb #1{1}#2.5\@tempdimb}%
                #2\z@ plus1fil minus1fil\relax
        \hskip.5\tabcolsep}
\newcommand{\cdashlinelr}[1]{%
  \noalign{\vskip\aboverulesep
           \global\let\@dashdrawstore\adl@draw
           \global\let\adl@draw\adl@drawiv}
  \cdashline{#1}
  \noalign{\global\let\adl@draw\@dashdrawstore
           \vskip\belowrulesep}}
\makeatother

\definecolor{citecolor}{HTML}{0071bc}
\definecolor{mydarkblue}{rgb}{0,0.08,1}
\definecolor{mydarkgreen}{rgb}{0.02,0.6,0.02}
\definecolor{mydarkred}{rgb}{0.8,0.02,0.02}
\definecolor{mydarkorange}{rgb}{0.40,0.2,0.02}
\definecolor{mypurple}{RGB}{111,0,255}
\definecolor{myred}{rgb}{1.0,0.0,0.0}
\definecolor{mygold}{rgb}{0.75,0.6,0.12}
\definecolor{mydarkgray}{rgb}{0.66, 0.66, 0.66}

\definecolor{darkblue}{rgb}{0,0.08,1}
\definecolor{darkgreen}{rgb}{0.02,0.6,0.02}
\definecolor{darkred}{rgb}{0.8,0.02,0.02}
\definecolor{darkorange}{rgb}{0.40,0.2,0.02}
\definecolor{darkpurple}{RGB}{111,0,255}

\def\model{S$^2$FT\xspace}

\def\oW{\overline{W}}
\def\uW{\underline{W}}
\def\oG{\overline{G}}

\def\pre{\textnormal{pre}}
\def\lora{\textnormal{LoRA}}
\def\sft{\textnormal{\model}}
\def\op{\textnormal{op}}
\def\F{\textnormal{F}}
\def\id{\textnormal{i}}
\def\ood{\textnormal{o}}

\def\spn{\textnormal{span}}
\def\column{\textnormal{colspan}}
\def\full{\textnormal{full}}

\def\bias{\textnormal{bias}}
\def\variance{\textnormal{variance}}
\def\R{\mathbb{R}}

\def\P {\mathbb{P}}
\def\E {\mathbb{E}}
\def\N{\mathbb{N}^+}

\newcommand{\SVD}{\textrm{SVD}}
\newcommand{\diag}{\operatorname{diag}}

\theoremstyle{plain}
\newtheorem{theorem}{Theorem}[section]

\newtheorem{lemma}[theorem]{Lemma}
\newtheorem{corollary}[theorem]{Corollary}
\theoremstyle{definition}

\newtheorem{assumption}[theorem]{Assumption}
\theoremstyle{remark}
\newtheorem{remark}[theorem]{Remark}

%% file: text/abstract.tex
\vspace{-1.5em}
\begin{abstract}
\vspace{-1em}
Current PEFT methods for LLMs can achieve high quality, efficient training, or scalable serving, but not all three simultaneously.  
To address this limitation, we investigate sparse fine-tuning and observe a remarkable improvement in generalization ability. 
Utilizing this key insight, we propose a family of \underline{S}tructured \underline{S}parse \underline{F}ine-\underline{T}uning (\textbf{\model}) methods for LLMs, which \textit{concurrently achieve state-of-the-art fine-tuning performance, training efficiency, and inference scalability}. \model accomplishes this by ``selecting sparsely and computing densely". Based on the coupled structures in LLMs, \model selects a few attention heads and channels in the MHA and FFN modules for each Transformer block, respectively. Next, it co-permutes the weight matrices on both sides of all coupled structures to connect the selected subsets in each layer into a dense submatrix. Finally, \model performs in-place gradient updates on all selected submatrices.
Through theoretical analyses and empirical results, our method prevents forgetting while simplifying optimization, delivers SOTA performance on both commonsense and arithmetic reasoning with 4.6$\%$ and 1.3$\%$ average improvements compared to LoRA, and surpasses full FT by 11.5$\%$ when generalizing to various domains after instruction tuning. 
Using our partial back-propagation algorithm, \model saves training memory up to 3$\times$ and improves latency by 1.5-2.7$\times$ compared to full FT, while achieving an average 10\% improvement over LoRA on both metrics. We further demonstrate that the weight updates in \model can be decoupled into adapters, enabling effective fusion, fast switch, and efficient parallelism when serving multiple fine-tuned models.

\end{abstract}
\vspace{-1em}

%% file: text/introduction.tex
\section{Introduction}
\label{sec:introduction}
\vspace{-0.8em}

Recently, Large Language Models (LLMs) have achieved significant success~\cite{llama3, achiam2023gpt, team2023gemini}. With these models being applied in diverse domains, full fine-tuning (FT) is commonly employed to enhance their downstream capabilities~\cite{roziere2023code, azerbayev2023llemma, yunxiang2023chatdoctor}.
However, \mbox{retraining all parameters comes with three} drawbacks: (i) Full FT suffers from catastrophic forgetting, where a model \mbox{forgets pre-trained knowledge while} acquiring new information~\cite{luo2023empirical, biderman2024lora}. (ii) As the model and dataset sizes grow at scale, full FT becomes increasingly computation-demanding and memory-intensive~\cite{xu2023parameter}. (iii) It is impractical to store and serve thousands of fine-tuned LLMs on modern GPUs if each requires full parameter storage~\cite{zhao2024lora, slora}.

Parameter-efficient fine-tuning (PEFT) methods propose to address these bottlenecks by updating a small fraction of parameters~\cite{han2024parameter}.  Rather than merely reducing the number of learnable parameters, an ideal PEFT method should possess three key properties to be practically effective and efficient:
\vspace{-0.2em}

\textbf{High Quality}: It should exhibit both memorization and generalization capabilities, balancing the acquisition of new information from fine-tuning tasks with the retention of pre-trained knowledge.
\vspace{-0.3em}

\textbf{Efficient Training}: It should minimize the memory footprint for model gradient and optimization states, and further translate such memory efficiency into less computation and fine-tuning speedup.
\vspace{-0.3em}

\textbf{Scalable Serving}: It should avoid adding inference overhead when serving a single PEFT model. For multiple models, new parameters should be partially stored as adapters to save memory, and allows for effective fusion~\cite{zhang2023composing}, fast switch~\cite{kong2024lora}, and efficient parallelism~\cite{slora} among thousands of adapters.

However, achieving all the aforementioned goals simultaneously is challenging. Common PEFT approaches, such as LoRA~\cite{lora}, DoRA~\cite{dora}, and Galore~\cite{zhao2024galore}, project the model's weights or gradients onto a low-rank subspace. While this significantly reduces memory footprint, their performance lags behind full fine-tuning in most large-scale scenarios. Recent state-of-the-art PEFT methods have aimed to improve performance but at the cost of serving efficiency. ReFT operates on a frozen base model and learns task-specific interventions on hidden representations that cannot be merged into the original model, leading to a $2.2 \times$ increase in inference latency. LISA~\cite{pan2024lisa} employs a coarse-grained selective method by randomly freezing most Transformer blocks during optimization, which requires significantly more trainable parameters. Consequently, in scaled serving settings like S-LoRA~\cite{slora}, LISA can only serve at most $\frac{1}{10}$ as many fine-tuned models as LoRA under the same memory budget.

\input{figures/intro}

Prior to the era of LLMs, PEFT methods based on unstructured sparse fine-tuning (SpFT) have shown a strong trade-off between low number of parameters and high model performance without sacrificing serving efficiency~\cite{fishmask, lt-sft, xu2021raise}. We hypothesize that SpFT, which selectively updates a small subset of model parameters, can outperform LoRA and its variants in generalization capabilities. In Figure~\ref{fig:observation}, our findings across various generalization tasks support this hypothesis. However, the unstructured nature of SpFT  necessitates sparse operations in computation, hindering its efficient training and scalable serving on modern hardware. This makes SpFT less practical for adapting LLMs at scale. 

In this work, we propose a family of \underline{S}tructured \underline{S}parse \underline{F}ine-\underline{T}uning (\textbf{\model}) methods to ``select sparsely and compute densely" (See Figure~\ref{fig:intro}), thereby closing the efficiency gap in SpFT. Inspired by structured weight pruning techniques~\cite{ma2023llm, liu2023deja}, we first identify several coupled structures inherent in LLMs that are connected by intermediate activations. For example, in the multi-head attention (MHA) module, each attention head in the query, key, and value projections is linked to only a few rows in the output projection. Similarly, in the feed-forward network (FFN) module, each column in the up and gate projections corresponds to a single row in the down projection. By co-permuting the matrices on both sides of these coupled structures, we can preserve the original output of these structures, with only the order of the intermediate activations changed. Exploiting this property, our \model strategically selects a subset of attention heads for the MHA module and a subset of channels for the FFN module. We then permute the coupled structures to connect the selected components within each linear layer into a dense submatrix. Finally, through our partial back-propagation algorithm with only two-line code modification, \model performs in-place gradient updates exclusively \mbox{for all selected submatrices,} boosting training efficiency by eliminating redundant forward activations and backward calculation.

Through our theoretical analysis, \model mitigates forgetting under distribution shifts while simplifying optimization. 
Empirically, \model outperforms other PEFT methods on LLaMA and Mistral family models, improving 1.2-4.1\% on commonsense reasoning tasks and 0.6-1.9\% on arithmetic reasoning ones. It also surpasses full FT by 11.5\% when generalize to various domains after instruction
tuning.

Finally, we conduct a comprehensive analysis to verify the training efficiency and serving scalability of \model. Compared to existing PEFT methods, \model not only saves 1.4-3.0$\times$ memory, but also increases latency by 1.5 to 2.7$\times$, making LLM fine-tuning more accessible. Additionally, \model's parameter updates can be decomposed into adapters, enabling adapter fusion with smaller performance drop than LoRA. Our method also results in more scalable and efficient adapter switch and parallelism through \mbox{reduced matrix multiplications, showcasing strong potential for large-scale LLM serving scenarios.}

%% file: figures/intro.tex
\begin{figure}[t]
    \centering
    \vspace{-1.8em}
    \includegraphics[width=\linewidth]{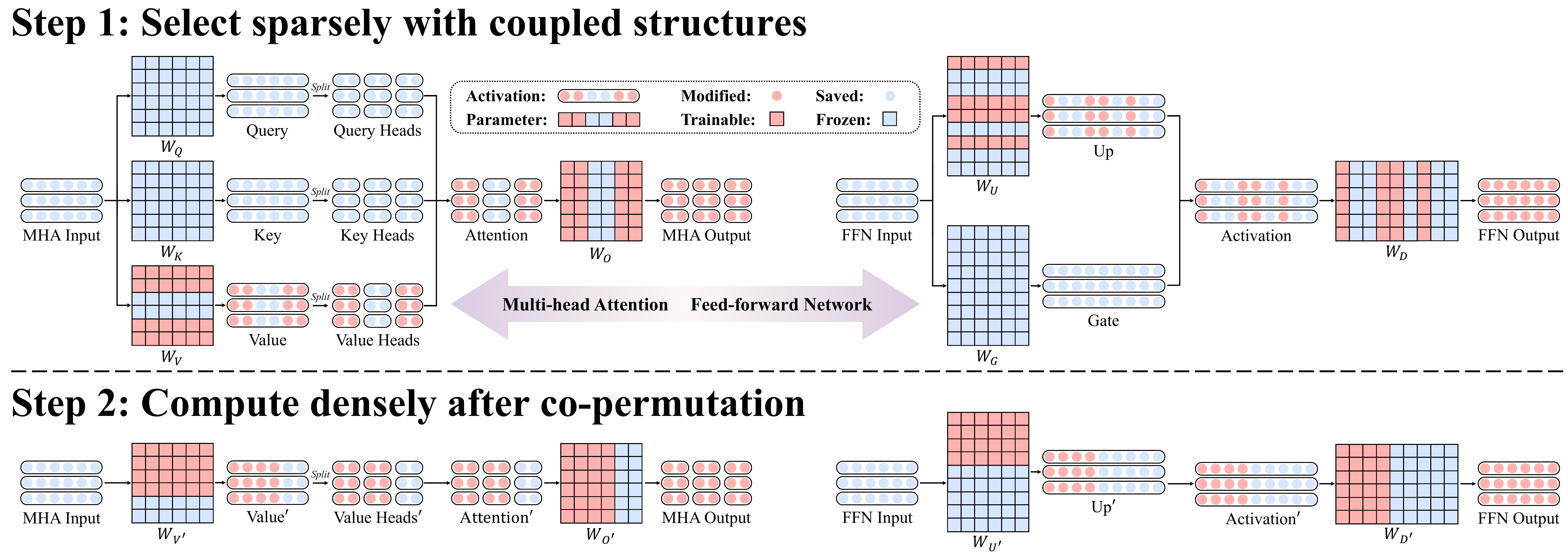}
    \caption{\textbf{An Overview of the \model Family for LLMs}: First, we perform sparse selection of specific attention heads and channels within the coupled structures of the MHA and FFN modules. Next, we apply co-permutation to the weight matrices on both sides of these structures, enabling dense gradient computation only for the selected components. While we demonstrate \model by selecting the same heads/channels on both sides for clarity, our approach also  supports asymmetric selection strategies.}
    \label{fig:intro}
    \vspace{-3em}
\end{figure}

%% file: text/observation.tex
\vspace{-1.2em}
\section{Memorization or Generalization?}
\label{sec:observation}
\vspace{-0.8em}
In this section, we evaluate the memorization and generalization capabilities of various fine-tuning methods, including full FT, LoRA, and SpFT. We hypothesize that SpFT can generalize better to downstream tasks. To support this hypothesis, we present detailed observations and analyses. Further theoretical analysis about the generalization capabilities of the \model family can be found in Section~\ref{sec:theory}.

\vspace{-0.2em}
\textbf{Hypothesis.} We hypothesize that SpFT offers superior generalization than both full FT and LoRA, while maintaining comparable memorization to LoRA with the same number of trainable parameters.

\vspace{-0.2em}
\textbf{Experimental Setup.} We fine-tune the \texttt{Llama3-8B}
on the Math10K data~\cite{llm-adapters} using SpFT, LoRA, and full FT. In addition to training losses, accuracies are measured on downstream tasks in LLM-Adapters, including near out-of-distribution (OOD) generalization on both easy (i.e, MultiArith, AddSub, SingleEq, MAWPS) and hard (i.e, GSM8K, AQuA, SVAMP) arithmetic reasoning tasks, and far OOD generalization on commonsense reasoning ones. For PEFT methods, we set three ratios of \mbox{trainable parameters  ($p=10\%, 1\%, 0.1\%$) and search for the optimal hyperparameters on the valid} set. In SpFT, trainable parameters are selected randomly with given ratios. See details in Appendix~\ref{app:observation}.
\vspace{-1.7em}
\begin{figure}[!ht]
    \centering
\includegraphics[width=\linewidth]{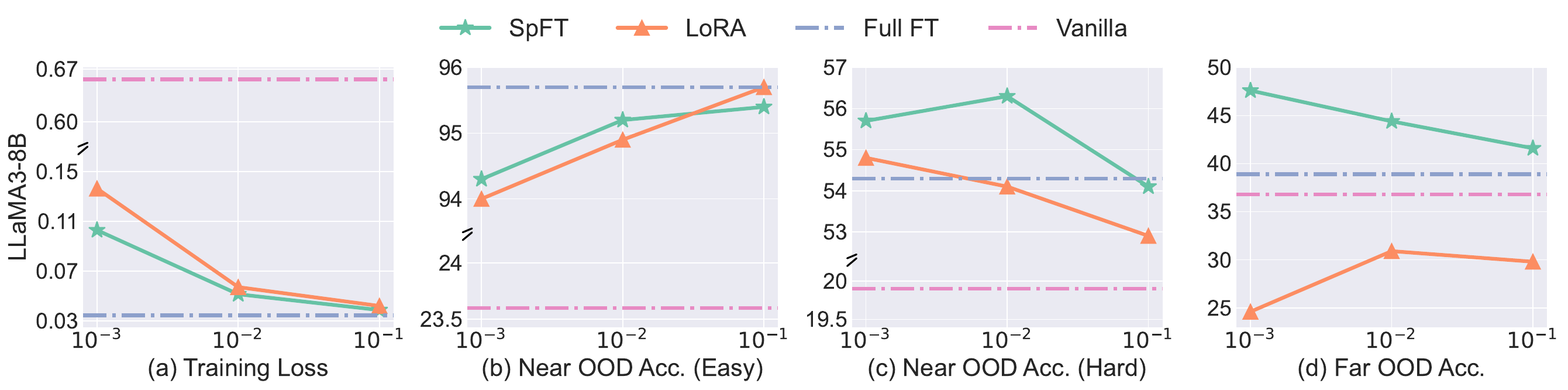}
    \caption{Accuracy comparison of SpFT, LoRA and Full FT at varying ratios of trainable parameters in various settings. SpFT exhibits strong generalization ability while full FT excels in memorization.}
    \label{fig:observation}
\end{figure}
\vspace{-1.5em}

\textbf{Observations.} Figure~\ref{fig:observation} indicates several key findings. First, SpFT achieves lower training losses than LoRA when using the same ratio of trainable parameters, especially at very small ratios. This gap arises from the more complex optimization process in LoRA, which requires the simultaneous updating of two matrices~\cite{hayou2024lora+}. Second, we observe both elevated training loss and reduced average accuracy on easier math tasks as the ratio decreases, suggesting a positive correlation between memorization abilities and trainable parameters. Notably, with only 10\% of the parameters updated, PEFT \mbox{methods learn comparable memorization abilities to full FT when trained on a 10k-sample dataset.}

\vspace{-0.2em}
When generalizing to complex mathematical problems or commonsense reasoning tasks, the performance ranking emerges as: SpFT $>$ Full FT $>$ LoRA. SpFT effectively transfers reasoning abilities to commonsense domains, while LoRA exhibits significant performance drops in far OOD generalization. This indicates (i) freezing a larger fraction of the parameters can retain more pre-trained abilities, and (ii) approximating high-dimensional gradients with low-rank decomposition may overfit fine-tuned data and hinder the model from generalization.
Since LLMs are pre-trained on high-quality data, SpFT emerges as the preferred choice for fine-tuning on task-specific data of varying quality.
\vspace{-0.5em}

%% file: text/method.tex
\vspace{-0.4em}
\section{The S$^{2}$FT family of methods}
\label{sec:method}
\vspace{-0.8em}

While SpFT demonstrates strong generalization ability and good overall performance in Section~\ref{sec:observation}, its unstructured nature poses challenges for efficient training and scalable serving on modern hardware (e.g., GPU). This is because of the need for sparse operations when storing and computing weights, gradients, and optimization states, which are significantly slower than their dense variants on GPU. This motivates our investigation into structured sparsity approaches that utilize only dense operations:

\vspace{-0.4em}
\emph{Can structured sparsity improve hardware efficiency while preserving performance by selecting sparsely but computing densely? If so, how far can the flexibility of selection be pushed in \mbox{this context?}}
\vspace{-1.4em}

To answer this question, we design a family of \underline{S}tructured \underline{S}parse \underline{F}ine-\underline{T}uning (\textbf{\model}) methods with dense-only computations, making PEFT effective, efficient and scalable. We begin by discovering the coupled structure in LLMs in Section~\ref{sec:3.1}. Leveraging this property, Section~\ref{sec:3.2} introduce the selection and permutation strategies of \model, with overall pipeline illustrated in Figure~\ref{fig:intro}\textcolor{red}{b}. In Section~\ref{sec:3.3}, we present our partial back-propagation algorithm that enables end-to-end training latency reduction.

\vspace{-0.8em}
\subsection{Discover Coupled Structures in LLMs}
\label{sec:3.1}
\vspace{-0.5em}

We initiate our pursuit of flexible structured sparsity by examining the coupled structures in LLMs.
\vspace{-0.2em}

\input{figures/structure}
\vspace{-0.2em}
\textbf{Structure Dependency in LLMs.} Inspired by prior work on structured pruning~\cite{ma2023llm, fang2023depgraph},  our study start by building the dependencies between activations and weights for LLMs. Let $A$ denote an activation and $W$ denote a weight in the model. We define $\operatorname{In}(A)$ as the set of parameters that directly contribute to the computation of $A$, and $\operatorname{Out}(A)$ as the set of parameters that depend on $A$ in the computation of \mbox{subsequent activations. The dependency between structures can be defined as follows:}
\begin{align}
W_1 \in \operatorname{In}(A) \wedge \operatorname{Deg}^+(W_1) = 1 &\Rightarrow A \text{ is dependent on } W_1 \\
W_2 \in \operatorname{Out}(A) \wedge \operatorname{Deg}^-(W_2) = 1 &\Rightarrow W_2 \text{ is dependent on } A
\end{align}
where $\operatorname{Deg}^+(W_1)$ represents the out-degree of weight $W_1$, and $\operatorname{Deg}^-(W_2)$ represents the in-degree of weight $W_2$. Each equation represents a unqiue directional dependency between activations and weights. When both equations hold simultaneously, a coupled structure exists between $W_1$ and $W_2$. In Figure~\ref{fig:structure}, we employ deep linear networks to illustrate two types of coupled structures in LLMs: 

\vspace{-0.2em}
\textit{Basic Structures}: In Figure~\ref{fig:structure1}, these structures exist in both the multi-head attention (MHA) and feed-forward network (FFN) modules. Taking LLaMA as an example, in the MHA module, we consider the Query ($\bm{Q}$), Key ($\bm{K}$), and Value ($\bm{V}$) projections as $W_1$, and the Output ($\bm{O}$) projection \mbox{as $W_2$, while $\texttt{Softmax}(\bm{Q}\bm{K}^\top)\bm{V}(x)$ acting as the activation between weight matrices. Similarly, in} the FFN module, the Up ($\bm{U}$) and Gate ($\bm{G}$) projections function as $W_1$, with the Down ($\bm{D}$) projection corresponding to $W_2$. Here, $\bm{U}(x) \cdot \texttt{SwiGLU}(\bm{G}(x))$ serves as the activations connecting $W_1$ and $W_2$.

\vspace{-0.2em}
\textit{Residual Structures}: In Figure~\ref{fig:structure2}, this type of coupled structures exists between the MHA and FFN modules. We further consider how residual connections influence the activations in these structures.

\vspace{-0.2em}
\textbf{Permutation Invariance of Coupled Structures.} Figure~\ref{fig:structure} demonstrates that $W_1$ and $W_2$ can be co-permuted using the same order, which only affects the order of activations between them while preserving the original output from the coupled structure. Since residual dependencies require an additional run-time step to permute the residuals, we will focus on basic dependencies in our method.

\vspace{-0.6em}
\subsection{Sparse Selection and Permutation}
\label{sec:3.2}
\vspace{-0.6em}

At this point, all coupled structures within the model have been identified. The subsequent sparse selection and permutation processes are straightforward, with overall pipeline illustrated in Figure~\ref{fig:intro}\textcolor{red}{b}.

\vspace{-0.2em}
\textbf{MHA Module}: There are four linear layers in a MHA module: $Q, K, V, O \in \R^{d \times d}$. For a model with $h$ attention heads, each head $i \in [h]$ has its own projections denoted as $Q_i \in \mathbb{R}^{d \times d_h}$, $K_i \in \mathbb{R}^{d \times d_h}$, $V_i \in \mathbb{R}^{d \times d_h}$, and $O_i \in \mathbb{R}^{d_h \times d}$, where $d_h = d/h$ is the dimension per head. Let $S_{\text{MHA}} \subseteq [h]$ denote a small subset of attention heads. By permuting $S_{\text{MHA}}$ to the beginning of each weight matrix, we are able to update these selected heads using dense-only operations, while keeping the other ones frozen.

\vspace{-0.2em}
\textbf{FFN Module}: There are three linear layers in an FFN module: $U, G \in \R^{k \times d}$ and $D \in \R^{d \times k}$. In \model, only a few channels require gradient updates. Let $S_{\text{FFN}} \subseteq [d]$ denote the selected channels. We can permute $S_{\text{FFN}}$ to the beginning of each weight matrix and only fine-tune this compact subset.

\vspace{-0.2em}

Next, we provide several strategies for identifying and selecting important subsets in each module.
\begin{enumerate}    
[itemsep=0.0pt,topsep=0pt,leftmargin=*]
    \item \textbf{\model-R (\model)}: In this strategy, a subset of channels is randomly selected and set to be trainable.
    \item \textbf{\model-W}: This variant selects subsets based on the magnitude of the weights for linear layers.
    \item \textbf{\model-A}: This variant selects subsets based on the magnitude of activations on a calibration set.
    \item \textbf{\model-S}: Top-K subsets are ranked and selected by the product of weight and activation magnitudes.
    \item \textbf{\model-G}: This variant selects subsets based on the magnitude of gradients on a calibration set.
\end{enumerate}

\vspace{-0.2em}
Here, 1 and 2 can be applied directly without pre-processing. 3 and 4 only require a forward pass on a small calibration dataset. While 5 necessitates a\mbox{ backward pass on this dataset, it does not} store optimization states and can mitigate memory footprints for activations through gradient checkpointing~\cite{feng2021optimal}. By default, we use \model-R for a fair comparison and discuss other variants in Section~\ref{sec:5.4}.

\vspace{-0.5em}
\subsection{Partial Back-propagation Algorithm}
\label{sec:3.3}
\vspace{-0.7em}

Finally, we introduce our partial back-propagation algorithm with only two line modifications in PyTorch. our algorithm stores trainable channels based on their start and end positions, thereby improving training efficiency by eliminating redundant forward activations and backward calculations.

\begin{lstlisting}[language=Python, numbers=none, basewidth={0.5em,0.5em}]
def setup_context(ctx, inputs, output):
   activation, weight, bias, start, end = inputs
   # only save partial input tensors for gradient calculation in forward
   ctx.save_for_backward(activation[:, start:end], weight, bias, start, end)

def gradient_update(parameter, gradient, start, end):
    # only modify the assigned positions of weight matrices during optimization 
    parameter[:, start:end].add_(gradient)
\end{lstlisting}
\vspace{-1em}

%% file: figures/structure.tex
\begin{figure}[h]
    \vspace{-1.5em}
    \centering
    \begin{subfigure}{0.42\textwidth}
        \centering
        \includegraphics[width=\textwidth]{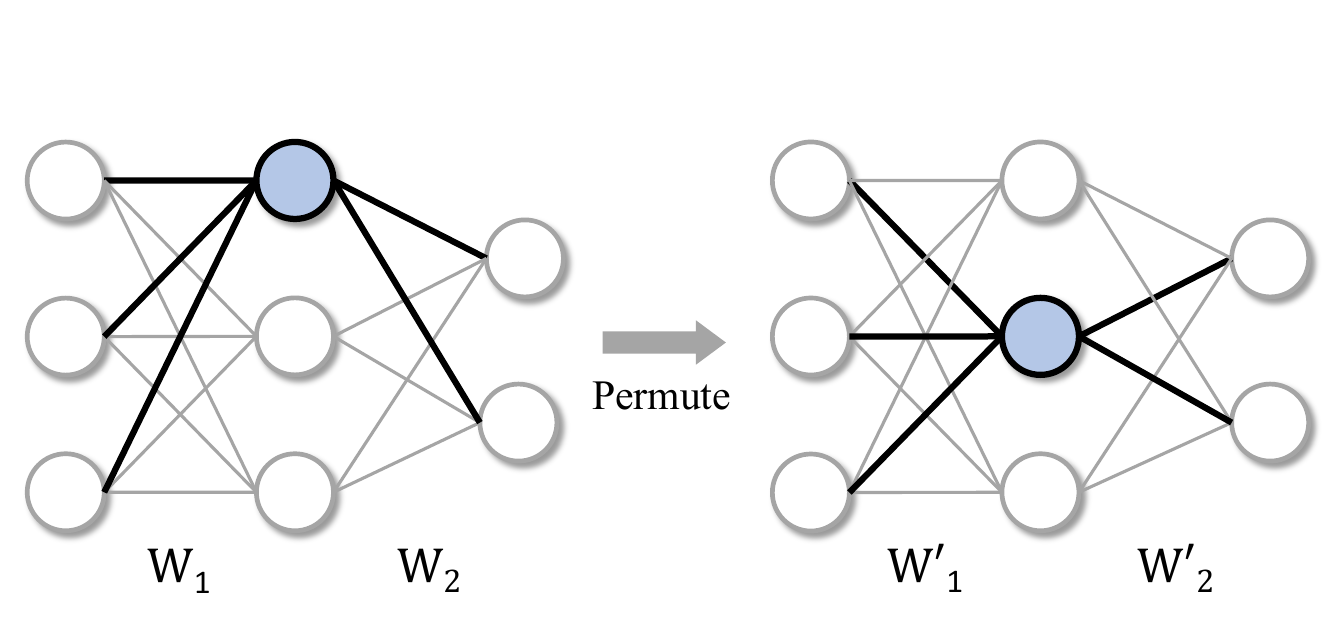}
        \caption{Basic Structure}
        \label{fig:structure1}
    \end{subfigure}
    \hfill
    \begin{subfigure}{0.54\textwidth}
        \centering
        \includegraphics[width=\textwidth]{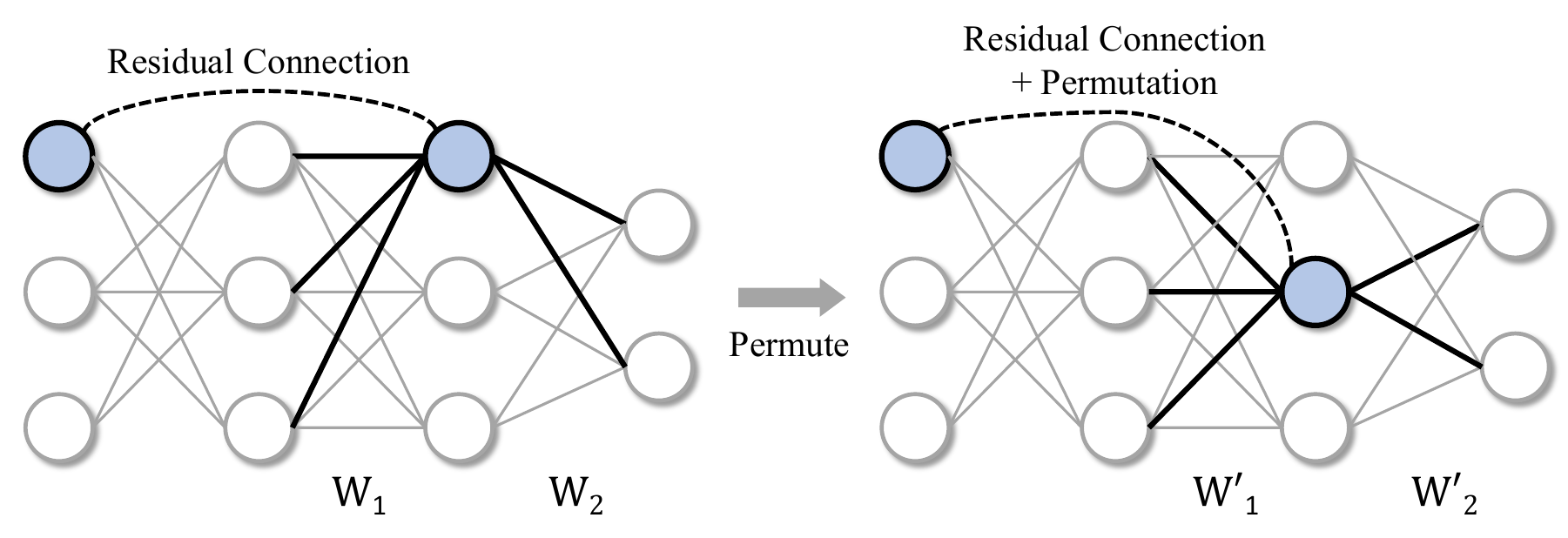}
        \caption{Residual Structure}        \label{fig:structure2}
    \end{subfigure}
    \vspace{-0.2em}
    \caption{Grouped model weights with basic structure and residual structure. All highlighted weights must be permuted simultaneously. Residual structures require additional permutation during runtime.}
    \vspace{-1.4em}
    \label{fig:structure}
\end{figure}

%% file: text/theory.tex
\section{Theoretical Analysis}
\vspace{-0.6em}
\label{sec:theory}
\vspace{-0.2em}

In this section, we theoretically explain why \model demonstrates stronger generalization capabilities compared to LoRA. Following previous work~\citep{hayou2024lora+, zhang2023lora, ponkshe2024initialization, po2024sbora}, we further show that \model is simple and efficient in optimization by maintaining stability in both the magnitude and direction of updates.

\vspace{-0.6em}
\subsection{Stronger Generalization Capability}
\vspace{-0.6em}

First, we theoretically explore why \model demonstrates stronger generalization capabilities compared to LoRA. We consider a pre-trained $L$-layer deep linear network, which has been widely used to facilitate the theoretical analysis of complex DNNs~\citep{saxe2013exact,kawaguchi2016deep,lu2017depth,hardt2016identity,laurent2018deep,arora2019implicit}. Let $f^\pre(x) := W_L^\pre W_{L-1}^\pre \dots W_1^\pre x$ be the pre-trained deep linear network, where $W_\ell^\pre \in \R^{d_\ell \times d_{\ell-1}}$, with $d_0 = p$ and $d_L = q$. 
We fine-tune the $\ell$-th layer with low-rankness level $r \leq \min\{d_\ell, d_{\ell-1}\}$ or sparsity level $s = \lfloor r \cdot \frac{d_\ell + d_{\ell-1}}{d_{\ell-1}} \rfloor$ . Denote a class of adaptation with parameters $U \in \R^{d_\ell\times d}$ and $V \in \R^{d_{\ell-1} \times d}$ as
\vspace{-0.5em}
\begin{align}
    f_{\ell,U,V}(x) := \oW^\pre_{\ell+1} (W^\pre_\ell + U V^\top) \uW_{\ell-1}^\pre x,\label{eq: rank s adaptation}
\end{align}
where $\oW_{\ell}^\pre := W_L^\pre W_{L-1}^\pre \dots W_{\ell}^\pre \in \R^{d_L \times d_{\ell-1}}$ and $\uW_\ell^\pre := W_\ell^\pre W_{\ell-1}^\pre \dots W_1^\pre \in \R^{d_\ell \times d_0}$ with $\uW_0^\mathrm{pre} = I_p$ and $\oW_L^\mathrm{pre} = I_q$. In a transformer-based LLM, each row of $W_{\ell}$ can represent the parameters in a single attention head for the MHA module or in a single channel for the FFN module.


Given $n$ observations $(x_i^{(\id)}, y_i^{(\id)}) \subset \R^p \times \R^q$, we fine-tune $f^\pre$ by minimizing the empirical risk $\mathcal{R}_n^{(\id)}(f_{\ell,U,V}) := (1/n) \sum_{i \in [n]} \|y_i^{(\id)} - f_{\ell,U,V}(x_i^{(\id)})\|^2$ via gradient descent. 
For LoRA, we train both low-rank matrices $(U, V) $ in Equation \eqref{eq: rank s adaptation} with $d \gets r$.
For \model, we train only $V$ in Equation \eqref{eq: rank s adaptation} with $d \gets s$ and fixed $U \gets U_S^\sft := [e_{a_1}; e_{a_2}; \dots; e_{a_s}]$, where $S = \{a_1, \dots, a_s\} \subset [d_\ell]$ and $e_a$ is the $a$-th standard basis. Similar conclusions hold when we fine-tune only $U$.
Motivated by the implicit regularization in gradient descent~\citep{zhang2021understanding,gunasekar2017implicit,arora2019implicit}, we directly consider minimum norm solutions.

We consider a multiple linear regression setting. 
\mbox{Assume that the in-distribution training data $(x^{(\id)},$} $y^{(\id)}) \in \R^{p+q}$ and out-of-distribution test data $(x^{(\ood)}, y^{(\ood)}) \in \R^{p+q}$ are generated i.i.d. according to 
\begin{align*}
    y^{(k)} &= B^{(k)} x^{(k)} + \epsilon^{(k)}, \ \ k \in \{\id, \ood\},
\end{align*}
\vspace{-0.3em}where $B^{(k)} \in \R^{q \times p}$ is the coefficient matrix, $x^{(k)}$ and $\epsilon^{(k)}$ are mean zero sub-Gaussian signal and noise with covariance matrices $\Sigma_x^{(k)}$ and $\Sigma_\epsilon^{(k)}$, respectively.
The generalization capacity is measured by the fine-tuned model's excess risk $\mathcal{E}(f) := \E[\|y^{(o)} - f(x^{(o)})\|^2] - \inf_{f'} \E[\|y^{(o)} - f'(x^{(o)})\|^2]$.

For these OOD data, LoRA suffers from forgetting, while \model can maintain pre-training knowledge.

\begin{assumption}[Distribution Shift]\label{asm: distribution shift}
    Assume that $\Sigma_x^{(\id)} = \Sigma_x^{(\ood)} = \Sigma_x$ for some $\Sigma_x \in \R^{p \times p}$, and $\|(\oW^\pre_{\ell+1} U_S^\sft) (\oW^\pre_{\ell+1} U_S^\sft)^\dag (B^{(\ood)} - B^{(\id)}) \Sigma_x^{1/2}\|_\F^2 \leq \varepsilon^2 \mathcal{E}^{(\ood)}(f^\pre)$ for some $\varepsilon > 0$. 
\end{assumption}
Assumption~\ref{asm: distribution shift} states that while the covariate distribution remains unchanged, the label distribution conditioned on covariates may shift, but not exceeding a factor of $\epsilon^2$ of the OOD risk of $f^\pre$. This holds for fine-tuning with proper channel selection, where primarily the output distribution is changed.



\begin{theorem}[Out-of-distribution Excess Risk, Informal]\label{thm: sft and lora ood informal}
    Suppose Assumption~\ref{asm: distribution shift} holds.
    Consider $n \to \infty$. 
    If $B^{(\id)} = \oW^\pre_{\ell+1} \tilde B^{(\id)} \uW^\pre_{\ell-1}$ holds for some $\tilde B^{(\id)} \in \R^{d_\ell \times d_{\ell-1}}$, and $s \leq \rank(\Sigma_f^{(\id)})$, then,
    \begin{align*}
        \mathcal{E}^{(\ood)}(f_{\ell,U_S^\sft,V^\sft}) &\leq (1 + 3\varepsilon^2) \mathcal{E}^{(\ood)}(f^\pre),\ \ 
        \mathcal{E}^{(\ood)}(f_{\ell,U^\lora,V^\lora}) \geq \|(B^{(\ood)} - B^{(\id)}) \Sigma_x^{1/2}\|_\F^2.
    \end{align*}
\end{theorem}
\vspace{-0.5em}
Theorem~\ref{thm: sft and lora ood informal} indicates that the OOD risk of \model \mbox{is bounded above by that of $f^\pre$, while that of} LoRA is bounded below by the label shift magnitude. If $f^\pre$ already has a low risk for OOD tasks, and the label shift is significant, \model is expected \mbox{to outperform LoRA. Essentially, when the OOD task} deviates significantly from the FT distribution, LoRA may forget pre-trained knowledge and overfit to the FT data, compromising its generalization capabilities. See formal statements in Theorem~\ref{thm: sft and lora ood}.


\vspace{-1.0em}
\subsection{Simple and Efficient Optimization}
\vspace{-0.6em}

Next, we explain why \model is a simple and efficient optimization method. In Equation \eqref{eq: rank s adaptation}, \model can be viewed as a LoRA variant that fixes $U_S^\sft$ as a combination of multiple orthogonal standard basis vectors while optimizing $V^\sft$ with zero initialization. The gradient is given by $\frac{\partial \mathcal{L}}{\partial V^\sft} =  (\uW_{\ell-1}^\pre x)^\top \frac{\partial \mathcal{L}}{\partial  \oW^\pre_{\ell+1}} U_S^\sft$. Ignore $\uW_{\ell-1}^\pre$, $\oW_{\ell-1}^\pre$ and denote  $\frac{\partial \mathcal{L}}{\partial  \oW^\pre_{\ell+1}}$ as $\oG$, at step $t$ with learning rate $\eta$, 
\vspace{-0.5em}
$$\Delta f_{\ell, t}(x) := f_{\ell, t}(x) - f_{\ell, t-1}(x) = U_S^\sft(V_t^\sft - V_{t-1}^\sft)^\top x
= -\eta U_S^\sft U_S^{\sft^\top} \oG^\top  ||x||^2 .$$
\vspace{-1.7em}

Since $U_S^\sft$ is an orthogonal matrix, the update simplifies to $\Delta f_{\ell, t}(x) =-\eta \oG^\top||x||^2$. Following LoRA+~\cite{hayou2024lora+},  assuming that $x = \Theta_n(1)$, where $n$ is the width of the layers in LLMs, we expect $\Delta f_{\ell, t}(x) = \Theta(1)$ to ensure stability and feature learning in the infinite-width limit~\citep{yang2021tuning}. \model can achieve this when $\eta = \Theta(n^{-1})$ while LoRA requires $\eta_U = \Theta(1)$ and $\eta_V = \Theta(n^{-1})$ for optimal performance. These rates become impractical for modern LLMs with very large $n$. Therefore, \model aligns with LoRA variants that fix one matrix~\citep{po2024sbora, zhang2023lora}, offering more stable and efficient optimization.

Furthermore, under a given sparsity level as regularization, our model simplifies optimization when approximating the full fine-tuning gradients at non-zero positions. Similar to LoRA-SB~\citep{ponkshe2024initialization}, let $G_V$ denote the gradient of $V^\sft$. The equivalent gradient $\tilde{G}$, which describes the virtual gradient of the pretrained weight matrices, can be expressed as $U_S^\sft G_V^\top$. Then, the gradient with respect to $V^\sft$ can be expressed in terms of the gradient of the pretrained weight $W^\pre$ as: $G_V^O = U_S^{\sft^\top}G$. Using this relationship, our objective is to minimize the distance between the equivalent gradient
and the full gradient as $\min_{G_V} \|\tilde{G} - G\|^2_F$, where the optimal solution is given by $G_V = (U_S^{\sft^\top} U_S^\sft)^{-1}G_V^O$. Since $U_S^\sft$ is orthogonal, we have $G_V = G_V^O$. This shows that \model can keep the optimal update directions throughout the training process, establishing it as an efficient sparse optimization method.

%% file: text/experiment.tex
\vspace{-1.0em}
\section{Experiments}
\label{sec:exp}
\vspace{-0.8em}
In this section, we conduct a series of experiments across three diverse benchmarks covering more than 20 datasets. Our goal is to provide a rich picture of how \model performs in different scenarios. Here,
we compare our method with different fine-tuning strategies and categories including: (i) Full fine-tuning (FT), (ii) \textit{reparameterized fine-tuning}: LoRA~\cite{lora}, DoRA~\cite{dora}, and Galore~\cite{zhao2024galore}, (iii) \textit{adapter-based fine-tuning}: Series Adapter~\cite{series-adapter}, Parallel Adapter~\cite{parallel-adapter}, and LoReFT~\cite{wu2024reft}, (iv) \textit{prompt-based fine-tuning}: Prefix-Tuning~\cite{Prefix-Tuning}, (v) \textit{sparse fine-tuning}: LISA~\cite{pan2024lisa}. For a fair comparison, we keep a comparable number of trainable parameters in \model to that of LoRA. The design choices for trainable parameter allocations in \model will be detailed in Section~\ref{sec:5.4}. All other hyperparameters are selected via cross-validation. Detailed setups and dataset descriptions are provided in Appendix~\ref{app:exp}.
\vspace{-0.7em}
\subsection{Commonsense Reasoning}
\label{sec:5.1}
\vspace{-0.7em}

The results of eight common sense reasoning tasks in Table~\ref{tab: commonsense} show that \model consistently outperforms existing PEFT methods in the \texttt{LLaMA-7B / 13B}, \texttt{LLaMA2-7B} and \texttt{LLaMA3-8B} models. Compared to LoRA and DoRA, it achieves average performance gains of 4.6\% and 2.8\%, respectively. Furthermore, \model also shows superior performance against recent approaches, including Galore, LoReFT, and LISA, with improvements of at least 1.0\%. Remarkably, despite using less than 1\% of trainable parameters, our method surpasses full FT by 0.5\%. The 3.0\% improvement on the \texttt{LLaMA3-8B} suggests that keeping most pre-trained parameters frozen enables better generalization to test distributions.

\input{tables/commonsense}
\vspace{-0.2em}

\vspace{-0.7em}
\subsection{Arithmetic Reasoning}
\label{sec:5.2}
\vspace{-0.3em}

\vspace{-0.3em}
As showcased in Table~\ref{tab: math}, \model consistently outperforms other PEFT methods for different base models. On average, it achieves improvements of 1.3\% and 0.9\% over LoRA and DoRA, respectively. These results highlight the versatility and effectiveness of our approach across a diverse range of tasks. Additionally, we observe substantial improvements even when compared to Full FT for the \texttt{LLaMA3-8B} model, particularly on complex tasks such as GSM8K and AQuA. This suggests that \model better preserves the original reasoning capabilities of this stronger model while acquiring new skills from the fine-tuning data, thereby validating the enhanced generalization ability of our method.

\input{tables/math}

\vspace{-0.8em}
\subsection{Instruction Following}
\vspace{-0.5em}

\vspace{-0.3em}
Table~\ref{tab:mt-bench} comprehensively compares various methods on eight tasks in the MT-Bench dataset~\cite{mt-bench}. It is observed that \model $>$ LISA $>$ Full FT $>$ LoRA/Galore $\geq$ Vanilla for both the \texttt{Mistral-7B} and \texttt{LLama2-7B} model. This is because sparse FT methods like \model and LISA retain more pre-trained knowledge while acquiring new skills on the FT dataset, thereby generalizing better to diverse tasks in the MT-Bench dataset. Moreover, our method outperforms LISA due to its fine-grained and flexible selection strategy, enabling all layers to learn to follow instructions on the full fine-tuning set.
\vspace{-1.5em}

\input{tables/instruction_follow}


\subsection{Design Choices for Trainable Parameter Allocations}
\label{sec:5.4}
\vspace{-0.5em}
\mbox{Finally, we detail how \model distribute trainable parameters across layers, modules, and channels.}
\vspace{-0.2em}

\textbf{Uniform across Layers}: Following Chen et al.~\cite{chen2023parameter}, we allocate parameters to each layer uniformly.
\vspace{-0.3em}

\textbf{Fine-tune Important Modules}: Figure~\ref{fig:design:component} analyzes the effectiveness of different components in a LLaMA-like Transformer Block for fine-tuning, including Query, Key, Value, Output, Up, Gate, and Down projections. To ensure a fair comparison, we maintain a fixed number of trainable parameters when fine-tuning each component. The results show that the effectiveness of components in fine-tuning follows the order: Query/Key $\ll$ Value/Up/Gate $<$ Output/Down. This is because Query/Key are only used to measure token similarities, while others serve as persistent memories of training data. Based on this finding, we allocate our parameter budget fairly to the Output and Down projections. For the \texttt{LLama3-8B} and \texttt{Mistral-7B} models, we only fine-tune the Down projection due to the inflexible selection in multi-query attention. Further analysis of this setting is left for future research.
\vspace{-0.3em}

\input{figures/component}

\input{tables/selection_strategy}

\textbf{Selection across Channels}: In Section~\ref{sec:3.2}, we discuss several strategies for channel selection. In our main experiments, we employ random selection to ensure fair comparisons with baseline methods, as these approaches treat all channels with equal importance. \mbox{However, the sparse structure of \model} offers controllability during fine-tuning, \mbox{allowing us to prioritize important channels in the selection} process to further boost performance. \mbox{Table~\ref{tab:selection} compared nine different strategies, incorporating five} varying selection metrics (i.e., random, weight, activation, weight-activation product, and gradient), each choosing either the largest or smallest values. For \model-A, \model-S, and \model-G, we employ 1\% of the fine-tuning data as a calibration set, introducing only negligible overhead during inference. 

\mbox{Our results demonstrate that random selection serves as a strong baseline due to its unbiased nature.} \mbox{Among heuristic metrics, selecting channels with the smallest activations (i.e., \model-A and \model-S)} outperforms random selection. This indicates that these channels contain less task-specific information, enabling us to inject new knowledge through fine-tuning while preserving pre-trained capabilities in other channels. In contrast, other strategies introduce bias that compromises model performance. \mbox{Notably, the counterintuitive accuracy decrease in \model-G (Large) suggests that channels with large} gradients contain task-related pre-trained knowledge, and modifying them will disrupt these abilities.


%% file: tables/commonsense.tex
\begin{table}[htbp] 
\small\centering
\footnotesize
\vspace{-5em}
\caption{Comparison among various fine-tuning methods for the \texttt{LLaMA-7B/13B}, \texttt{LLaMA2-7B}, and \texttt{LLaMA3-8B} models on eight commonsense reasoning tasks. Non-PEFT methods are marked in \textcolor{darkgray!50}{gray}. (\textsuperscript{1}: from DoRA paper, \textsuperscript{2}: from ReFT paper, \textsuperscript{3}: reproduced by us, \textsuperscript{\dag}: projected trainable parameters)} 
\setlength\tabcolsep{1.1pt} 
\resizebox{\linewidth}{!}{
\begin{tabular}{llcccccccccc}
\toprule
\textbf{Model} & \textbf{Method } & \textbf{\# Param(\%) } & \textbf{BoolQ } & \textbf{PIQA } & \textbf{SIQA } & \textbf{HellaSwag } & \textbf{Wino } & \textbf{ARC-e } & \textbf{ARC-c } & \textbf{OBQA } & \textbf{Avg. }$\uparrow$ \\
\midrule
\textcolor{darkgray!50}{ChatGPT\textsuperscript{1}}     & \textcolor{darkgray!50}{-}     & \textcolor{darkgray!50}{-}     & \textcolor{darkgray!50}{73.1}       & \textcolor{darkgray!50}{85.4}       & \textcolor{darkgray!50}{68.5}       & \textcolor{darkgray!50}{78.5}      & \textcolor{darkgray!50}{66.1}       & \textcolor{darkgray!50}{89.8}      & \textcolor{darkgray!50}{79.9}       & \textcolor{darkgray!50}{74.8}       & \textcolor{darkgray!50}{77.0}\\
\midrule
\multirow{11}{*}{LLaMA-7B } & \textcolor{darkgray!50}{Full FT\textsuperscript{3}} & \textcolor{darkgray!50}{100}   & \textcolor{darkgray!50}{70.3}  & \textcolor{darkgray!50}{84.2}  & \textcolor{darkgray!50}{80.1}  & \textcolor{darkgray!50}{92.3}  & \textcolor{darkgray!50}{85.4}  & \textcolor{darkgray!50}{86.6}  & \textcolor{darkgray!50}{72.8}  & \textcolor{darkgray!50}{83.4}  & \textcolor{darkgray!50}{81.9} \\
\cdashlinelr{2-12}
                          & Prefix~\cite{Prefix-Tuning}\textsuperscript{1}         & 0.11  & 64.3  & 76.8  & 73.9  & 42.1  & 72.1  & 72.9  & 54.0  & 60.6  & 64.6 \\
                          & Series~\cite{series-adapter}\textsuperscript{1}         & 0.99  & 63.0  & 79.2  & 76.3  & 67.9  & 75.7  & 74.5  & 57.1  & 72.4  & 70.8 \\
                          & Parallel~\cite{parallel-adapter}\textsuperscript{1}        & 3.54  & 67.9  & 76.4  & 78.8  & 69.8  & 78.9  & 73.7  & 57.3  & 75.2  & 72.2 \\
                          & LoRA~\cite{lora}\textsuperscript{3} & 0.83  & 69.2  & 81.7  & 78.4  & 83.4  & 80.8  & 79.0  & 62.4  & 78.4  & 76.7 \\
                          & DoRA~\cite{dora}\textsuperscript{1}         & 0.84  & 68.5  & 82.9  & 79.6  & 84.8  & 80.8  & 81.4  & 65.8  & 81.0  & 78.1 \\
                          & Galore~\cite{zhao2024galore}\textsuperscript{3}\      & \,\,0.83\textsuperscript{\dag} & 68.6  & 79.0  & 78.5  & 84.7  & 80.1  & 80.3  & 62.1  & 77.3  & 76.3 \\
                          & LoReFT~\cite{wu2024reft}\textsuperscript{2} & 0.03  & 69.3  & 84.4  & 80.3  & 93.1  & 84.2  & 83.2  & 68.2  & 78.9 & 80.2 \\ 
                          & LISA~\cite{pan2024lisa}\textsuperscript{3} & 9.91  & 70.4  & 82.1  & 78.7  & 92.4  & 82.9  & 84.9  & 70.2  & 78.4 & 80.0 \\ 
                          & \multicolumn{1}{>{\columncolor{cyan!10}}l}{\textbf{\model (Ours)}} & \multicolumn{1}{>{\columncolor{cyan!10}}c}{0.81} & \multicolumn{1}{>{\columncolor{cyan!10}}c}{\textbf{72.7}} & \multicolumn{1}{>{\columncolor{cyan!10}}c}{\textbf{83.7}} & \multicolumn{1}{>{\columncolor{cyan!10}}c}{\textbf{79.6}} & \multicolumn{1}{>{\columncolor{cyan!10}}c}{\textbf{93.4}} & \multicolumn{1}{>{\columncolor{cyan!10}}c}{\textbf{83.5}} & \multicolumn{1}{>{\columncolor{cyan!10}}c}{\textbf{86.1}} & \multicolumn{1}{>{\columncolor{cyan!10}}c}{\textbf{72.2}} & \multicolumn{1}{>{\columncolor{cyan!10}}c}{\textbf{83.4}} & \multicolumn{1}{>{\columncolor{cyan!10}}c}{\textbf{81.8}} \\
\midrule
\multirow{9}{*}{LLaMA-13B } & \textcolor{darkgray!50}{Full FT\textsuperscript{3}} & \textcolor{darkgray!50}{100}   & \textcolor{darkgray!50}{74.5} & \textcolor{darkgray!50}{86.3} & \textcolor{darkgray!50}{81.3} & \textcolor{darkgray!50}{94.4} & \textcolor{darkgray!50}{86.9} & \textcolor{darkgray!50}{89.7}	& \textcolor{darkgray!50}{77.9} & \textcolor{darkgray!50}{88.8} & \textcolor{darkgray!50}{85.0} \\
\cdashlinelr{2-12}
                          & Prefix~\cite{Prefix-Tuning}\textsuperscript{1} & 0.03 & 65.3 & 75.4 & 72.1 & 55.2 & 68.6 & 79.5 & 62.9 & 68.0 & 68.4 \\
                          & Series~\cite{series-adapter}\textsuperscript{1}      & 0.80  & 71.8 & 83.0	& 79.2 & 88.1 & 82.4 & 82.5 & 67.3	& 81.8 & 79.5 \\
                          & Parallel~\cite{parallel-adapter}\textsuperscript{1}      & 2.89  & 72.5	& 84.9 & 79.8 & 92.1 & 84.7 & 84.2	& 71.2	& 82.4	& 81.4 \\
                          & LoRA~\cite{lora}\textsuperscript{1}       & 0.67 & 72.1	& 83.5 & 80.5 & 90.5 & 83.7 & 82.8 & 68.3	& 82.4 & 80.5 \\
                          & DoRA~\cite{dora}\textsuperscript{1}         & 0.68 & 72.4	& 84.9	& 81.5 & 92.4 & 84.2 & 84.2 & 69.6	& 82.8 & 81.5 \\
                          & LoReFT~\cite{wu2024reft}\textsuperscript{2}      & 0.03 & 72.1 & \textbf{86.3} & \textbf{81.8}	& \textbf{95.1} & \textbf{87.2} & 86.2 & 73.7	& 84.2 & 83.3 \\ 
                          & \multicolumn{1}{>{\columncolor{cyan!10}}l}{\textbf{\model (Ours)}} & \multicolumn{1}{>{\columncolor{cyan!10}}c}{0.65} & \multicolumn{1}{>{\columncolor{cyan!10}}c}{\textbf{74.2}} & \multicolumn{1}{>{\columncolor{cyan!10}}c}{85.7} & \multicolumn{1}{>{\columncolor{cyan!10}}c}{80.7} & \multicolumn{1}{>{\columncolor{cyan!10}}c}{94.9} & \multicolumn{1}{>{\columncolor{cyan!10}}c}{86.4} & \multicolumn{1}{>{\columncolor{cyan!10}}c}{\textbf{88.4}} & \multicolumn{1}{>{\columncolor{cyan!10}}c}{\textbf{76.3}} & \multicolumn{1}{>{\columncolor{cyan!10}}c}{\textbf{87.8}} & \multicolumn{1}{>{\columncolor{cyan!10}}c}{\textbf{84.3}} \\
\midrule
\multirow{5}{*}{LLaMA2-7B } & \textcolor{darkgray!50}{Full FT\textsuperscript{3}} & \textcolor{darkgray!50}{100}   & \textcolor{darkgray!50}{74.7}  & \textcolor{darkgray!50}{84.9}  & \textcolor{darkgray!50}{78.7}  & \textcolor{darkgray!50}{93.7}  & \textcolor{darkgray!50}{84.1}  & \textcolor{darkgray!50}{87.5}  & \textcolor{darkgray!50}{75.2}  & \textcolor{darkgray!50}{85.0}  & \textcolor{darkgray!50}{83.0} \\
\cdashlinelr{2-12} &  LoRA~\cite{lora}\textsuperscript{1}     & 0.83 & 69.8	& 79.9 & 79.5 & 83.6 & 82.6 & 79.8 & 64.7 & 81.0 & 77.6 \\
                          & DoRA~\cite{dora}\textsuperscript{1}     & 0.84 & 71.8	& 83.7 & 76.0 & 89.1 & 82.6 & 83.7 & 68.2 & 82.4 & 79.7 \\
                          & \multicolumn{1}{>{\columncolor{cyan!10}}l}{\textbf{\model (Ours)}} & \multicolumn{1}{>{\columncolor{cyan!10}}c}{0.81} & \multicolumn{1}{>{\columncolor{cyan!10}}c}{\textbf{72.9}} & \multicolumn{1}{>{\columncolor{cyan!10}}c}{\textbf{86.1}} & \multicolumn{1}{>{\columncolor{cyan!10}}c}{\textbf{80.2}} & \multicolumn{1}{>{\columncolor{cyan!10}}c}{\textbf{94.3}} & \multicolumn{1}{>{\columncolor{cyan!10}}c}{\textbf{85.5}} & \multicolumn{1}{>{\columncolor{cyan!10}}c}{\textbf{87.2}} & \multicolumn{1}{>{\columncolor{cyan!10}}c}{\textbf{74.6}} & \multicolumn{1}{>{\columncolor{cyan!10}}c}{\textbf{83.4}} & \multicolumn{1}{>{\columncolor{cyan!10}}c}{\textbf{83.0}} \\
\midrule
\multirow{5}{*}{LLaMA3-8B } & \textcolor{darkgray!50}{Full FT\textsuperscript{3}} & \textcolor{darkgray!50}{100}   & \textcolor{darkgray!50}{73.9}  & \textcolor{darkgray!50}{86.2}  & \textcolor{darkgray!50}{79.1}  & \textcolor{darkgray!50}{93.1}  & \textcolor{darkgray!50}{85.8}  & \textcolor{darkgray!50}{88.1}  & \textcolor{darkgray!50}{78.2}  & \textcolor{darkgray!50}{84.0}  & \textcolor{darkgray!50}{83.6} \\
\cdashlinelr{2-12} &  LoRA~\cite{lora}\textsuperscript{1}       & 0.70 & 70.8 & 85.2 & 79.7	& 92.5 & 84.9 & 88.9 & 78.7 & 84.4 & 82.5 \\
                          & DoRA~\cite{dora}\textsuperscript{1}         & 0.71 & 74.6 & \textbf{89.3} & 79.9 & 95.5	& 85.6 & 90.5 & 80.4 & 85.8 & 85.2 \\
                          & \multicolumn{1}{>{\columncolor{cyan!10}}l}{\textbf{\model (Ours)}}  & \multicolumn{1}{>{\columncolor{cyan!10}}c}{0.70}  & \multicolumn{1}{>{\columncolor{cyan!10}}c}{\textbf{75.0}} & \multicolumn{1}{>{\columncolor{cyan!10}}c}{89.0} & \multicolumn{1}{>{\columncolor{cyan!10}}c}{\textbf{80.7}} & \multicolumn{1}{>{\columncolor{cyan!10}}c}{\textbf{96.5}} & \multicolumn{1}{>{\columncolor{cyan!10}}c}{\textbf{88.0}} & \multicolumn{1}{>{\columncolor{cyan!10}}c}{\textbf{92.5}} & \multicolumn{1}{>{\columncolor{cyan!10}}c}{\textbf{83.4}} & \multicolumn{1}{>{\columncolor{cyan!10}}c}{\textbf{87.8}} & \multicolumn{1}{>{\columncolor{cyan!10}}c}{\textbf{86.6}} \\
\bottomrule
\end{tabular}}
\vspace{-1.5em}
\label{tab: commonsense}
\end{table} 

%% file: tables/math.tex
\begin{table}[t] 
\small\centering
\footnotesize
\vspace{-0.5em}
\caption{Comparison among various fine-tuning methods for different models on seven math reasoning tasks. Non-PEFT methods are marked in \textcolor{darkgray!50}{gray}. (\textsuperscript{1}: from LLM-Adapters paper, \textsuperscript{2}: reproduced by us)} 
\setlength\tabcolsep{1.1pt} 
\resizebox{\linewidth}{!}{
\begin{tabular}{llccccccccc}
\toprule
\textbf{Model} & \textbf{Method } & \textbf{\# Param(\%) } & \textbf{MultiArith } & \textbf{GSM8K } & \textbf{AddSub } & \textbf{AQuA } & \textbf{SingleEq } & \textbf{SVAMP } & \textbf{MAWPS } & \textbf{Avg. }$\uparrow$ \\
\midrule
\textcolor{darkgray!50}{GPT-3.5\textsuperscript{1}}     & \textcolor{darkgray!50}{-}     & \textcolor{darkgray!50}{-}     & \textcolor{darkgray!50}{83.8}       & \textcolor{darkgray!50}{56.4}       & \textcolor{darkgray!50}{85.3}       & \textcolor{darkgray!50}{38.9}      & \textcolor{darkgray!50}{88.1}       & \textcolor{darkgray!50}{69.9}         & \textcolor{darkgray!50}{87.4}       & \textcolor{darkgray!50}{72.8}\\
\midrule
\multirow{5}{*}{LLaMA-7B } & \textcolor{darkgray!50}{Full FT\textsuperscript{2}} & \textcolor{darkgray!50}{100}   & \textcolor{darkgray!50}{98.8} & \textcolor{darkgray!50}{43.1} & \textcolor{darkgray!50}{91.1} & \textcolor{darkgray!50}{20.9} & \textcolor{darkgray!50}{94.3} & \textcolor{darkgray!50}{60.6}	& \textcolor{darkgray!50}{88.2}  & \textcolor{darkgray!50}{71.0}   \\
\cdashlinelr{2-11}
                          & LoRA~\cite{lora}\textsuperscript{2} & 0.83  & 98.0  & 40.0  & 91.2  & 21.7  & 93.1  & 56.7 & 85.3 & 69.7  \\
                          & DoRA~\cite{dora}\textsuperscript{2} & 0.84  & 97.3  & 38.9  & 89.6  & \textbf{22.4}  & \textbf{93.9}  & \textbf{58.4}  & 85.3 & 69.4  \\                          
                          & \multicolumn{1}{>{\columncolor{cyan!10}}l}{\textbf{\model (Ours)}} & \multicolumn{1}{>{\columncolor{cyan!10}}c}{0.81} & \multicolumn{1}{>{\columncolor{cyan!10}}c}{\textbf{98.8}} & \multicolumn{1}{>{\columncolor{cyan!10}}c}{\textbf{41.3}} & \multicolumn{1}{>{\columncolor{cyan!10}}c}{\textbf{91.4}} & \multicolumn{1}{>{\columncolor{cyan!10}}c}{21.3} & \multicolumn{1}{>{\columncolor{cyan!10}}c}{93.5} & \multicolumn{1}{>{\columncolor{cyan!10}}c}{\textbf{58.4}} & \multicolumn{1}{>{\columncolor{cyan!10}}c}{\textbf{86.1}} & \multicolumn{1}{>{\columncolor{cyan!10}}c}{\textbf{70.1}} \\
\midrule
\multirow{5}{*}{LLaMA-13B } & \textcolor{darkgray!50}{Full FT\textsuperscript{2}} & \textcolor{darkgray!50}{100}   & \textcolor{darkgray!50}{98.3} & \textcolor{darkgray!50}{47.6} & \textcolor{darkgray!50}{92.9} & \textcolor{darkgray!50}{26.0} & \textcolor{darkgray!50}{95.1} & \textcolor{darkgray!50}{65.7} & \textcolor{darkgray!50}{88.7} & \textcolor{darkgray!50}{73.5} \\
\cdashlinelr{2-11}
                          & LoRA~\cite{lora}\textsuperscript{2}       & 0.67 &  97.5 & 47.8 & 89.9 & 20.5 & 94.3 & 61.2 & 87.4	& 71.2 \\
                          & DoRA~\cite{dora}\textsuperscript{2}       & 0.68 &  97.2 & 48.1 & \textbf{90.6} & 20.9 & 93.9 & 63.8 & \textbf{88.2}	& 71.8 \\
                          & \multicolumn{1}{>{\columncolor{cyan!10}}l}{\textbf{\model (Ours)}} & \multicolumn{1}{>{\columncolor{cyan!10}}c}{0.65} & \multicolumn{1}{>{\columncolor{cyan!10}}c}{\textbf{97.7}} & \multicolumn{1}{>{\columncolor{cyan!10}}c}{\textbf{48.4}} & \multicolumn{1}{>{\columncolor{cyan!10}}c}{90.4} & \multicolumn{1}{>{\columncolor{cyan!10}}c}{\textbf{22.8}} & \multicolumn{1}{>{\columncolor{cyan!10}}c}{\textbf{95.5}} & \multicolumn{1}{>{\columncolor{cyan!10}}c}{\textbf{63.9}} & \multicolumn{1}{>{\columncolor{cyan!10}}c}{87.8} & \multicolumn{1}{>{\columncolor{cyan!10}}c}{\textbf{72.4}}\\
\midrule
\multirow{5}{*}{LLaMA2-7B } & \textcolor{darkgray!50}{Full FT\textsuperscript{2}} & \textcolor{darkgray!50}{100}   & \textcolor{darkgray!50}{99.3} & \textcolor{darkgray!50}{47.5} & \textcolor{darkgray!50}{91.1} & \textcolor{darkgray!50}{24.4} & \textcolor{darkgray!50}{96.7} & \textcolor{darkgray!50}{62.5}	& \textcolor{darkgray!50}{89.1}  & \textcolor{darkgray!50}{72.9}   \\
\cdashlinelr{2-11}
&  LoRA~\cite{lora}\textsuperscript{2}     & 0.83 & 97.5	& 44.0 & 91.2 & 20.9 & 94.1 & 59.2 & 85.7 & 70.4 \\
                          & DoRA~\cite{dora}\textsuperscript{2}     & 0.84 & 98.2	& 43.8 & 90.1 & 24.4 & 94.5 & 59.1 & \textbf{89.1} & 71.3 \\
                          & \multicolumn{1}{>{\columncolor{cyan!10}}l}{\textbf{\model (Ours)}} & \multicolumn{1}{>{\columncolor{cyan!10}}c}{0.81} & \multicolumn{1}{>{\columncolor{cyan!10}}c}{\textbf{98.5}} & \multicolumn{1}{>{\columncolor{cyan!10}}c}{\textbf{44.3}} & \multicolumn{1}{>{\columncolor{cyan!10}}c}{\textbf{91.1}} & \multicolumn{1}{>{\columncolor{cyan!10}}c}{\textbf{25.2}} & \multicolumn{1}{>{\columncolor{cyan!10}}c}{\textbf{94.7}} & \multicolumn{1}{>{\columncolor{cyan!10}}c}{\textbf{61.8}} & \multicolumn{1}{>{\columncolor{cyan!10}}c}{88.2} & \multicolumn{1}{>{\columncolor{cyan!10}}c}{\textbf{72.0}} \\
\midrule
\multirow{5}{*}{LLaMA3-8B } &  \textcolor{darkgray!50}{Full FT\textsuperscript{2}} & \textcolor{darkgray!50}{100}   & \textcolor{darkgray!50}{99.2} & \textcolor{darkgray!50}{62.0} & \textcolor{darkgray!50}{93.9} & \textcolor{darkgray!50}{26.8} & \textcolor{darkgray!50}{96.7} & \textcolor{darkgray!50}{74.0}	& \textcolor{darkgray!50}{91.2}  & \textcolor{darkgray!50}{77.7}   \\
\cdashlinelr{2-11}
& LoRA~\cite{lora}\textsuperscript{2}       & 0.70 & 99.5 & 61.6 & 92.7 & 25.6 & 96.3 & 73.8 & 90.8 & 77.2 \\
                          & DoRA~\cite{dora}\textsuperscript{2}         & 0.71 & 98.8 & 62.7 & 92.2 & 26.8	& 96.9 & 74.0 & 91.2 & 77.5 \\
                          & \multicolumn{1}{>{\columncolor{cyan!10}}l}{\textbf{\model (Ours)}} & \multicolumn{1}{>{\columncolor{cyan!10}}c}{0.70} & \multicolumn{1}{>{\columncolor{cyan!10}}c}{\textbf{99.7}} & \multicolumn{1}{>{\columncolor{cyan!10}}c}{\textbf{65.8}} & \multicolumn{1}{>{\columncolor{cyan!10}}c}{\textbf{93.7}} & \multicolumn{1}{>{\columncolor{cyan!10}}c}{\textbf{31.5}} & \multicolumn{1}{>{\columncolor{cyan!10}}c}{\textbf{97.8}} & \multicolumn{1}{>{\columncolor{cyan!10}}c}{\textbf{76.0}} & \multicolumn{1}{>{\columncolor{cyan!10}}c}{\textbf{92.4}} & \multicolumn{1}{>{\columncolor{cyan!10}}c}{\textbf{79.6}} \\
\bottomrule
\vspace{-3.5em}
\end{tabular}}
\label{tab: math}
\end{table} 

%% file: tables/instruction_follow.tex
\begin{table}[t]
\vspace{-4em}
\caption{Performance comparison of LLM fine-tuning methods trained on the Alpaca GPT-4 dataset. We report the MT-Bench score as the evaluation metric. All baseline results are cited from LISA.}
\small\centering
\setlength\tabcolsep{2pt}
\begin{tabular}{llcccccccccc}
\toprule
\textbf{Model} & \textbf{Method} & \textbf{Writing} & \textbf{Roleplay} & \textbf{Reasoning} & \textbf{Code} & \textbf{Math} & \textbf{Extraction} & \textbf{STEM} & \textbf{Humanities} & \textbf{Avg.} \\
\midrule
\multirow{6}{*}{Mistral-7B} & Vanilla & 5.25 & 3.20 & 4.50 & 1.60 & 2.70 & 6.50 & 6.17 & 4.65 & 4.32 \\
& Full FT & 5.50 & 4.45 & 5.45 & 2.50 & 3.25 & 5.78 & 4.75 & 5.45 & 4.64 \\
& LoRA & 5.30 & 4.40 & 4.65 & 2.35 & 3.30 & 5.50 & 5.55 & 4.30 & 4.41 \\
& Galore & 5.05 & 5.27 & 4.45 & 1.70 & 2.50 & 5.21 & 5.52 & 5.20 & 4.36 \\
& LISA & 6.84 & 3.65 & 5.45 & 2.20 & 2.75 & 5.65 & 5.95 & 6.35 & 4.85 \\
& \multicolumn{1}{>{\columncolor{cyan!10}}l}{\textbf{Ours}} & \multicolumn{1}{>{\columncolor{cyan!10}}c}{\textbf{6.95}} & \multicolumn{1}{>{\columncolor{cyan!10}}c}{4.40} & \multicolumn{1}{>{\columncolor{cyan!10}}c}{\textbf{5.50}} & \multicolumn{1}{>{\columncolor{cyan!10}}c}{\textbf{2.70}} & \multicolumn{1}{>{\columncolor{cyan!10}}c}{\textbf{3.55}} & \multicolumn{1}{>{\columncolor{cyan!10}}c}{5.95} & \multicolumn{1}{>{\columncolor{cyan!10}}c}{\textbf{6.35}} & \multicolumn{1}{>{\columncolor{cyan!10}}c}{\textbf{6.75}} & \multicolumn{1}{>{\columncolor{cyan!10}}c}{\textbf{5.27}} \\
\midrule
\multirow{6}{*}{LLaMA2-7B} & Vanilla & 2.75 & 4.40 & 2.80 & 1.55 & 1.80 & 3.20 & 5.25 & 4.60 & 3.29 \\
& Full FT & 5.55 & 6.45 & 3.60 & 1.75 & 2.00 & 4.70 & 6.45 & 7.50 & 4.75 \\
& LoRA & 6.30 & 5.65 & 4.05 & 1.60 & 1.45 & 4.17 & 6.20 & 6.20 & 4.45 \\
& Galore & 5.60 & 6.40 & 3.20 & 1.25 & 1.95 & 5.05 & 6.57 & 7.00 & 4.63 \\
& LISA & 6.55 & \textbf{6.90} & 3.45 & 1.60 & \textbf{2.16} & 4.50 & 6.75 & 7.65 & 4.94 \\
& \multicolumn{1}{>{\columncolor{cyan!10}}l}{\textbf{Ours}} & \multicolumn{1}{>{\columncolor{cyan!10}}c}{\textbf{6.75}} & \multicolumn{1}{>{\columncolor{cyan!10}}c}{6.60} & \multicolumn{1}{>{\columncolor{cyan!10}}c}{\textbf{4.15}} & \multicolumn{1}{>{\columncolor{cyan!10}}c}{\textbf{1.65}} & \multicolumn{1}{>{\columncolor{cyan!10}}c}{1.85} & \multicolumn{1}{>{\columncolor{cyan!10}}c}{4.75} & \multicolumn{1}{>{\columncolor{cyan!10}}c}{\textbf{7.45}} & \multicolumn{1}{>{\columncolor{cyan!10}}c}{\textbf{8.38}} & \multicolumn{1}{>{\columncolor{cyan!10}}c}{\textbf{5.20}} \\
\bottomrule
\end{tabular}
\vspace{-1em}
\label{tab:mt-bench}
\end{table}

%% file: figures/component.tex
\begin{figure}[!ht]
    \centering
    \vspace{-4em}
    \includegraphics[width=\linewidth]{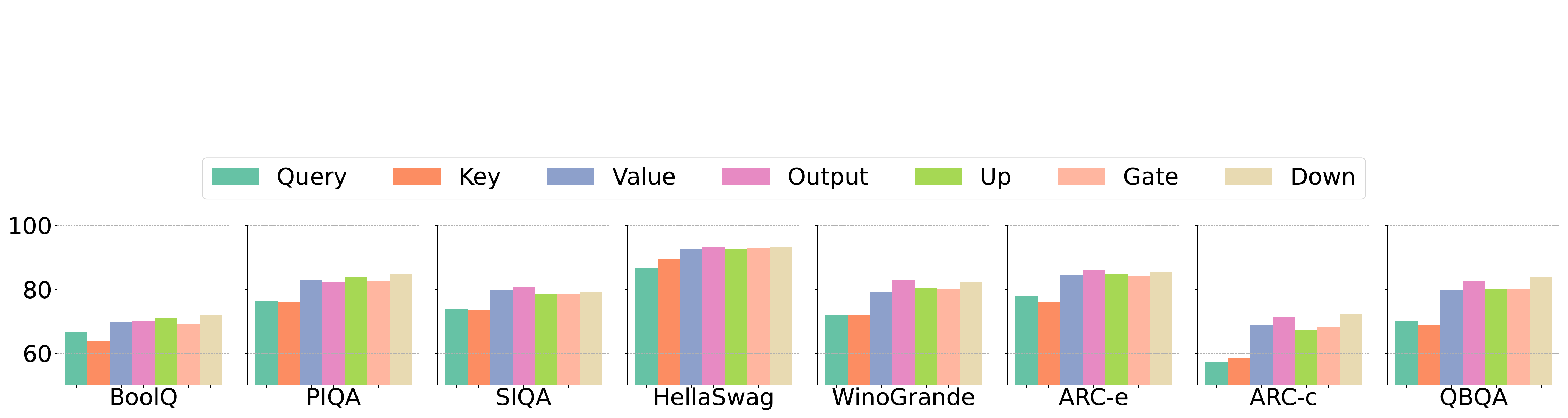}
    \caption{The impact of different components in fine-tuning, including Query, Key, Value, Output, Up, Gate, and Down projection. We fix the trainable parameter budget and only fine-tune one component.}
    \label{fig:design:component}
    \vspace{-1.3em}
\end{figure}

%% file: tables/selection_strategy.tex
\begin{table}[h]
\vspace{-1em}
\caption{Comparison of various channel selection strategies on the commonsense and arithmetic reasoning datasets for the \texttt{LLama3-8B}. We report the average accuracy (\%) as the evaluation metric.}
\small\centering
\resizebox{\linewidth}{!}{
\begin{tabular}{lcccccccccc}
\toprule
\multirow{2}{*}{\textbf{Task}} & \multirow{2}{*}{\textbf{\model-R}} & \multicolumn{2}{c}{\textbf{\model-W}} & \multicolumn{2}{c}{\textbf{\model-A}}  & \multicolumn{2}{c}{\textbf{\model-S}}  & \multicolumn{2}{c}{\textbf{\model-G}}   \\
\cmidrule(r{0.3em}){3-4}
\cmidrule(r{0.1em}){5-6}
\cmidrule(l{0.1em}){7-8}
\cmidrule(l{0.3em}){9-10}
& & Large & Small & Large & Small & Large & Small & Large & Small \\
\midrule
Commonsense & 86.6 & 85.9\textcolor{mydarkblue}{\textsubscript{(-0.7)}} & 85.3\textcolor{mydarkgray}{\textsubscript{(-1.3)}} & 84.7\textcolor{mydarkgray}{\textsubscript{(-1.9)}} & 87.3\textcolor{mydarkgreen}{\textsubscript{(+0.7)}} & 85.1\textcolor{mydarkgray}{\textsubscript{(-1.5)}} & 87.2\textcolor{mydarkgreen}{\textsubscript{(+0.6)}} & 85.4\textcolor{mydarkgray}{\textsubscript{(-1.2)}} & 86.2\textcolor{mydarkblue}{\textsubscript{(-0.4)}}\\
Arithmetic & 79.6 & 78.4\textcolor{mydarkgray}{\textsubscript{(-1.2)}} & 78.4\textcolor{mydarkgray}{\textsubscript{(-1.2)}} & 77.1\textcolor{mydarkgray}{\textsubscript{(-2.5)}} & 80.0\textcolor{mydarkgreen}{\textsubscript{(+0.4)}} & 76.8\textcolor{mydarkgray}{\textsubscript{(-2.8)}} & 79.8\textcolor{mydarkgreen}{\textsubscript{(+0.2)}} & 77.8\textcolor{mydarkgray}{\textsubscript{(-1.8)}} & 79.5\textcolor{mydarkblue}{\textsubscript{(-0.1)}} \\

\bottomrule
\end{tabular}}
\vspace{-0.7em}
\label{tab:selection}
\end{table}

%% file: text/analysis.tex
\vspace{-0.8em}
\section{Analysis}
\label{sec:efficient}
\vspace{-0.8em}

Having demonstrated the strong generalization capability and overall performance of \model, we now further explore its training efficiency and serving scalability compared to other fine-tuning techniques.
\vspace{-2em}

\subsection{Training Efficiency}
\vspace{-0.7em}
To evaluate training efficiency, we examine two crucial metrics: peak memory footprint and average training latency. These numbers are measured on a single Nvidia A100 (80G) SXM GPU. We keep a comparable number of parameters for all methods. To obtain the average latency, we fine-tune the model for 50 runs, each run including 200 iterations, with 10 warmup runs excluded in measurement. 

\input{figures/training_efficiency}

\mbox{In Figure~\ref{fig:training_efficiency}, we thoughtfully profile \model on various model sizes, sequence lengths, and batch sizes.} Compared to Full FT, \model saves 1.4-3.0$\times$ memory, and speedups fine-tuning by 1.5-2.7 times. When benchmarking against other PEFT methods, \model establishes new standards for efficiency, offering average reductions of 2\% in memory usage and 9\% in latency. Notably, \model outperforms the widely adopted LoRA, achieving about 10\% improvement in both metrics by avoiding the need to store new parameters and perform additional calculations. Our partial back-propagation algorithm further improves efficiency by saving unnecessary forward activations and backward calculations.
\vspace{-0.7em}

\subsection{Serving Scalability}
\vspace{-0.5em}
While \model avoids additional inference overhead for a single fine-tuned model through in-place gradient updates, we will now discuss its scalability for serving thousands of fine-tuned models. To begin, we introduce the unmerged computation paradigm of \model: Given a pre-trained weight matrix $W^{pre} \in \R^{d \times k}$ and its corresponding fine-tuned weight matrix $W$ with sparsity level $s$, we define the weight difference as $\Delta W = W - W^{\text{pre}}$. Similar to Section~\ref{sec:theory}, $\Delta W$ can be decomposed into the product of a weight matrix $V \in \R^{k \times s}$ and a permutation matrix $U \in \R^{d \times s}$. This decomposition allows us to ``unmerge" an adapter $\Delta W = UV^\top$ from $W$, thereby sharing similarities with other adapters during inference. Following Zhong et al.~\cite{zhong2024multi}, we consider three different adapter composition scenarios:

\textbf{Adapter Fusion.} To combine knowledge from multiple trained adapters, we employ weighted fusion when fine-tuning is impractical due to limited data access or computational resources. However, this approach degrades performance. In Table~\ref{tab:fusion}, we compare the effectiveness of LoRA and \model when combining adapters trained separately on commonsense and arithmetic reasoning tasks, where we consider both fine-tuning overlapped and non-overlapped parameters for different adapters in \model. Our results show that \model with non-overlapped parameters achieves the best performance, while the overlapped variant shows inferior results. This is because \model (non-overlap) modifies orthogonal low-rank spaces for different tasks. Similarly, LoRA largely retains task-specific capabilities during adapter fusion by optimizing low-rank projection matrices to create separate spaces for each adapter.


\input{tables/fusion}

\input{figures/scalable_inference}

\textbf{Adapter Switch.} Another way to leveraging multiple adapters is to dynamically switch between them. This process involves four steps: unfusing the old adapter, unloading it from memory, loading the new adapter, and fusing it into the model. In such setting, LoRA needs two matrix multiplications (\texttt{matmul}) and two additions (\texttt{add}) on GPU whereas \model only requires two sparse addition (\texttt{scatter\_add}). In Figure~\ref{fig:switch1}, we increase the base weight dimension while maintaining a sparsity of 32 for \model and a low-rankness of 16 for LoRA. Notably, we observe that LoRA's switching time scales quadratically, while S$^2$FT remains nearly constant. Moreover, in I/O-constrained scenarios such as deployment on CPU, \model further accelerates adapter switch by only updating a small fraction of the original weights, \mbox{reducing the volume of I/O transfers, as time compared between \texttt{scatter\_add} and \texttt{add} in Figure~\ref{fig:switch2}.} 
\vspace{-1.4em}

\textbf{Adapter Parallelism.} To serve thousands of adapters in parallel, we decompose the computation into separate batched computations for $W^{pre}$ and $\Delta W$ following S-LoRA~\cite{slora}.  While LoRA requires two \texttt{matmul} and one \texttt{add} on GPU, \model reduces this to a \texttt{matmul}, an \texttt{add}, and either a \texttt{scatter} or \texttt{gather} for $W_1$ and $W_2$ in Section~\ref{sec:3.1}. Figure~\ref{fig:parallel} shows that \model achieves up to 22\% faster inference than LoRA under the same memory constraints, with more speedup as the number of adapters scales.
\vspace{-1em}

%% file: figures/training_efficiency.tex
\begin{figure}[!ht]
    \centering
    \vspace{-0.7em}
    \includegraphics[width=\linewidth]{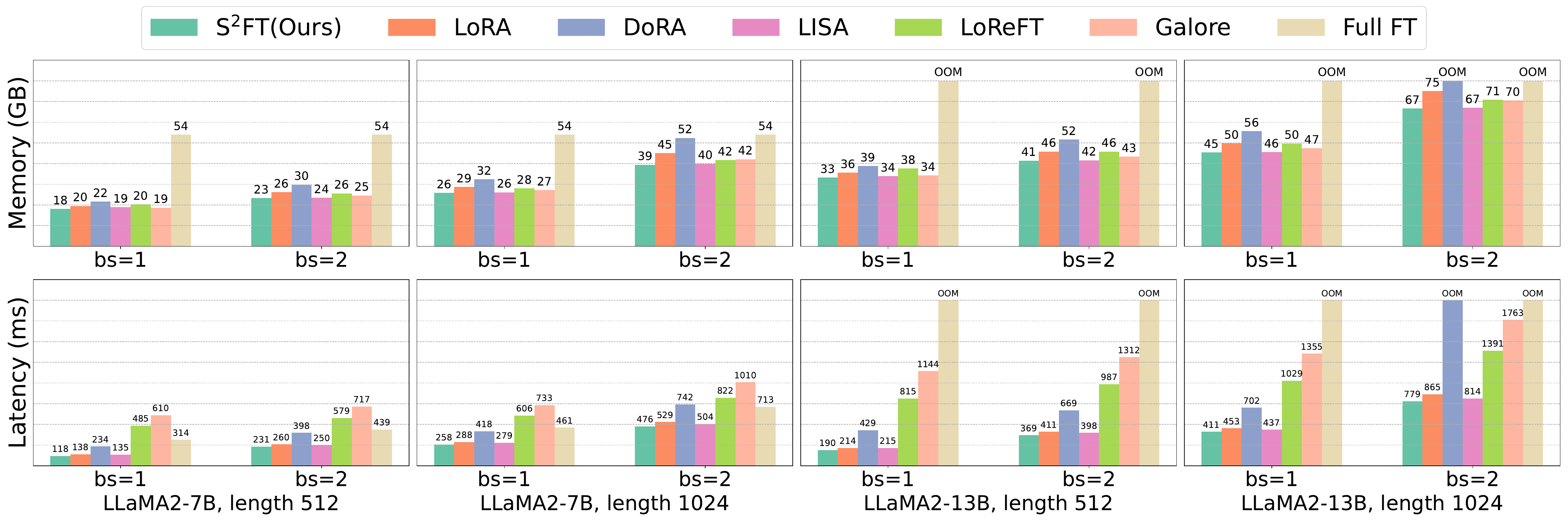}
    \caption{Comparison of memory and computation efficiency during training on the \texttt{LLaMA2-7B/13B} with varying sequence lengths and batch sizes. Average latency and peak memory usage are reported. \model significantly improves training latency while reducing memory footprint compared to baselines.}
    \label{fig:training_efficiency}
    \vspace{-0.8em}
\end{figure}

%% file: tables/fusion.tex
\begin{table}[h]
\vspace{-1.2em}
\caption{Adapter Fusion Results for LoRA and \model trained on the commonsense and arithmetic reasoning datasets using the \texttt{LLama3-8B}. We report the average accuracy (\%) as the evaluation metric.}
\small\centering
\resizebox{\linewidth}{!}{
\begin{tabular}{lccccccc}
\toprule
\multirow{2}{*}{\textbf{Task}} & \multicolumn{3}{c}{\textbf{LoRA}} & \multicolumn{4}{c}{\textbf{\model}} \\
\cmidrule(r{0.1em}){2-4}
\cmidrule(l{0.1em}){5-8}
& Commonsense & Arithmetic  & Fused & Commonsense & Arithmetic & Fused (overlap) & Fused (non-overlap) \\
\midrule
Commonsense & 83.1 & \textcolor{mydarkgray}{32.1} & 79.8\textcolor{mydarkblue}{\textsubscript{(-3.3)}} & 86.6 & \textcolor{mydarkgray}{42.3} & 82.0\textcolor{mydarkgray}{\textsubscript{(-4.6)}} & 84.0\textcolor{mydarkgreen}{\textsubscript{(-2.6)}} \\
Arithmetic & \textcolor{mydarkgray}{12.0} & 77.2 & 71.6\textcolor{mydarkblue}{\textsubscript{(-5.6)}} & \textcolor{mydarkgray}{12.8} & 79.6 & 72.2\textcolor{mydarkgray}{\textsubscript{(-7.4)}} & 75.3\textcolor{mydarkgreen}{\textsubscript{(-4.3)}} \\

\bottomrule
\end{tabular}}
\label{tab:fusion}
\end{table}

%% file: figures/scalable_inference.tex
\begin{figure}[h]
    \vspace{-1.5em}
    \centering
    \begin{subfigure}{0.32\textwidth}
        \centering
        \includegraphics[width=\textwidth]{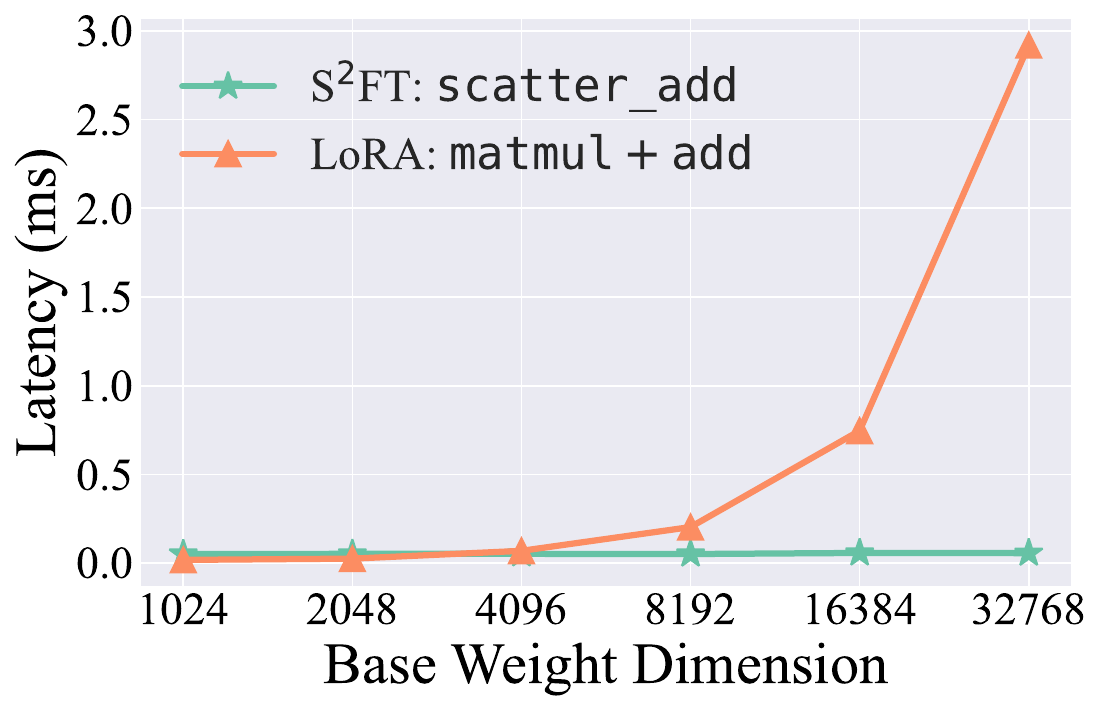}
        \caption{Switch Time on GPU}
    \label{fig:switch1}
    \end{subfigure}
    \hfill
    \begin{subfigure}{0.32\textwidth}
        \centering
        \includegraphics[width=\textwidth]{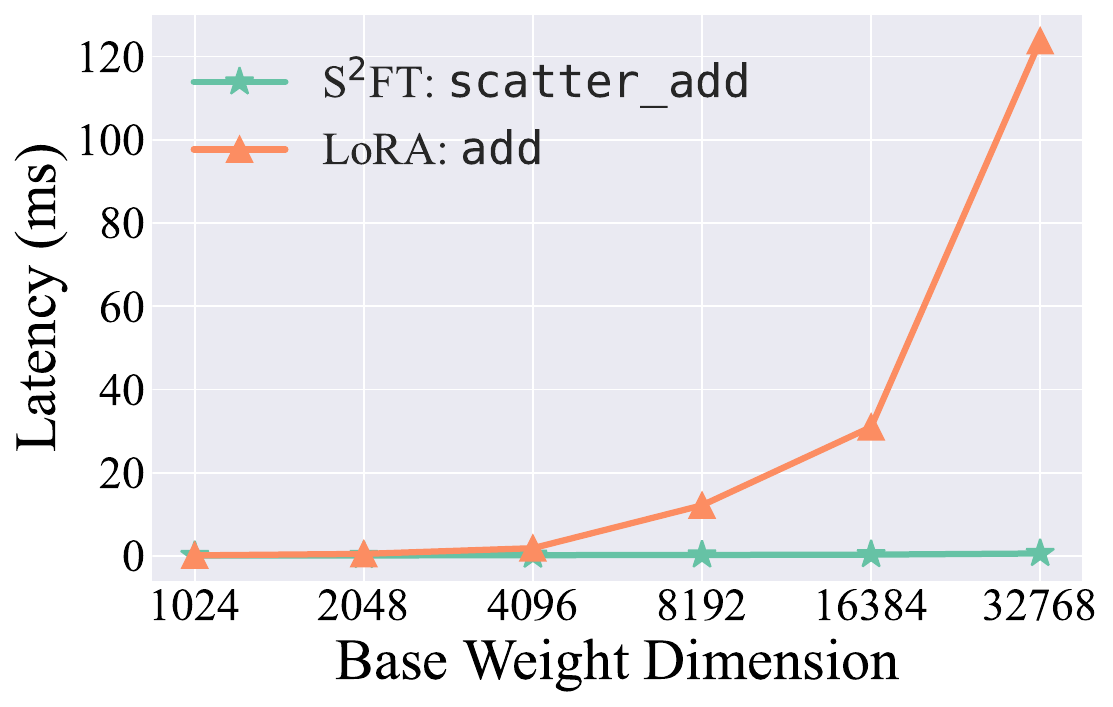}
        \caption{Switch Time on CPU}        \label{fig:switch2}
    \end{subfigure}
    \hfill
    \begin{subfigure}{0.32\textwidth}
        \centering
        \includegraphics[width=\textwidth]{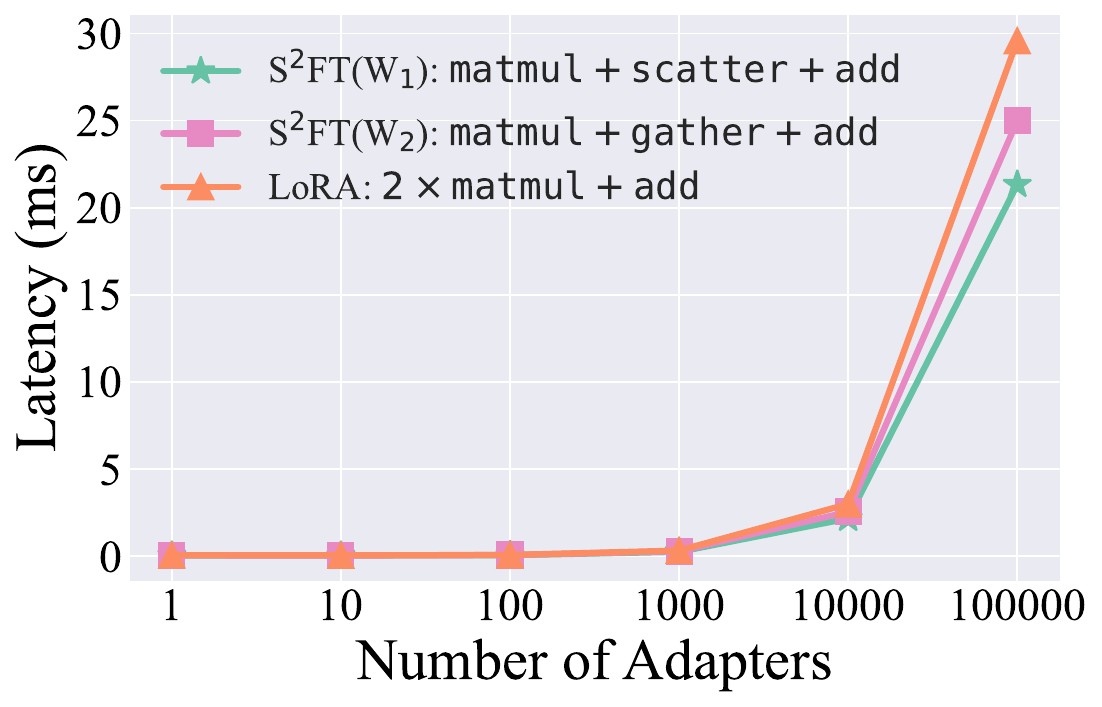}
        \caption{Parallelism Time on GPU}        \label{fig:parallel}
    \end{subfigure}
    \vspace{-0.5em}
    \caption{Comparison of latency for adapter switch and parallelism on a single linear layer. \model improves scalability for switch on GPU and CPU, while saving 22\% time during parallelism on GPU.}
    \label{fig:inference}
    \vspace{-1.8em}
\end{figure}

%% file: text/related.tex
\section{Related Work}
\label{related_work}
\vspace{-0.7em}

PEFT methods reduce the fine-tuning cost for large models, which can be categorized into 4 groups:
\vspace{-0.3em}

\textbf{Adapter-based Fine-tuning} introduces additional trainable module into the original model. Series Adapters insert components between MHA or FFN layers~\cite{pfeiffer2020mad, series-adapter}, while parallel adapters add modules alongside existing components~\cite{parallel-adapter}. Recently, ReFT~\cite{wu2024reft} was introduced to directly learn interventions on hidden representations. However,
they introduce additional latency during inference.
\vspace{-1.4em}

\textbf{Prompt-based Fine-tuning} adds randomly-initialized soft tokens to the input (usually as a prefix) and train their embeddings while freezing the model weights~\cite{Prefix-Tuning, liu2023gpt, lester2021power}. These approaches result in poor performance compared to other groups, while come at the cost of significant inference overhead.
\vspace{-1.4em}

\textbf{Reparameterized Fine-tuning} utilizes low-rank projections to reduce trainable parameters while allowing operations with high-dimensional matrices. LoRA\cite{lora} and its recent variants like DoRA\cite{dora}, AsyLoRA~\cite{zhu2024asymmetry},  and FLoRA~\cite{si2024flora}, use low-rank matrices to approximate additive weight updates during training. To alleviate the limitations of low-rank structure, other work also add or multiply orthogonal matrices to enable high-rank updating, including MoRA~\cite{jiang2024mora}, OFT~\cite{qiu2023controlling}, and BOFT~\cite{liu2023parameter}. These methods require no additional inference cost as the weight updates can be merged into models.
\vspace{-0.3em}

\textbf{Sparse Fine-tuning} aims to reduce the number of fine-tuned parameters by selecting a subset of pre-trained parameters that are critical to downstream tasks while discarding unimportant ones. This kind of methods are commonly used in the pre-LLM era~\cite{guo2020parameter, zaken2021bitfit, sung2021training}. However, they cannot reduce the memory footprints due to their unstructured nature. Recent approaches address this limitation through three directions: (1) developing structured variants that sacrifice selection flexibility for better hardware efficiency~\cite{pan2024lisa, zhu2023lift}, (2) incorporating sparsity into LoRA~\cite{wang2024roselora, ding2023sparse, liu2024alora} but yield limited efficiency gains, or (3) using sparse operators for lower memory cost but slow down training~\cite{ansell2024scaling, panda2024lottery, bhardwaj2024rapid}.
\vspace{-1.3em}

Our work is based on the last category but achieving better performance and efficiency simultaneously. Additionally, we focus on scalable inference of PEFT methods, with \model being the only approach that enables effective fusion, rapid switching, and efficient parallelism when serving multiple adapters.

\vspace{-0.5em}

%% file: text/conclusion.tex
\vspace{-0.6em}
\section{Conclusion}
\label{conclusion}
\vspace{-0.8em}

This paper introduces \model, a novel PEFT family that simultaneously achieves high quality, efficient training, and scalable serving for LLM fine-tuning. \model accomplishes this by selecting sparsely and compute densely. It selects a subset of heads and channels to be trainable for the MHA and FFN modules, respectively. The weight matrices from the two sides of the coupled structures in LLMs are co-permuted to connect the selected components into dense matrices, and only these parameters are updated using dense operations. We hope \model can be considered as a successor to LoRA for PEFT.

%% file: text/acknowledgement.tex
\section{Acknowledgement}

%% file: text/appendix.tex
\appendix

\section{Limitations}

While our work demonstrates the effectiveness of \model for LLM fine-tuning, several promising directions remain unexplored. First, extending \model to other architectures with coupled structures, such as CNNs and RNNs, can broaden its applicability. Second, verifying our approach beyond language tasks, particularly in large vision/multi-modal models, will enhance its versatility. Third, exploring more selection strategies can provide deeper insights into optimal fine-tuning protocols due to the controllability in \model. Fourth, scaling our method to larger models requires further experiments. Finally, although our work confirms the feasibility of scalable and efficient deployment during inference, developing a practical serving system for \model remains an important next step.

\section{Broader Impacts}

Since our work focuses on PEFT, it leads to a reduction in hardware resource and
energy consumption. Given the growing adoption of LLMs across diverse domains and the corresponding surge in fine-tuning demands, \model should represent an important step toward more sustainable AI development. 

\section{Detailed Experimental Setups for Section~\ref{sec:observation}}
\label{app:observation}

In this study, we used SpFT, LoRA, and Full FT to fine-tune the \texttt{LLaMA-3-8B} model on the Math10K dataset~\cite{llm-adapters}. The Math10K dataset combines training sets from GSM8K~\cite{cobbe2021training}, MAWPS~\cite{koncel2016mawps}, and AQuA~\cite{ling2017program}, augmented with chain-of-thought steps generated by language models. We conducted training for 3 epochs with a batch size of 64. For both PEFT methods--SpFT and LoRA--we fine-tune with three ratios of trainable parameters for all linear layers: $p = 10\%, 1\%, 0.1\%$.  The model's performance is evaluated on both arithmetic and commonsense reasoning tasks, representing near out-of-distribution (OOD) and far OOD generalization scenarios, respectively. The arithmetic reasoning dataset comprises seven subtasks: MultiArith~\cite{roy2016solving}, GSM8K, AddSub~\cite{hosseini2014learning}, AQuA, SingleEq~\cite{koncel2015parsing}, SVAMP~\cite{patel-etal-2021-nlp}, and MAWPS. The commonsense reasoning dataset includes eight subtasks: BoolQ~\cite{clark2019boolq}, PIQA~\cite{bisk2020piqa}, SocialQA~\cite{sap2019socialiqa}, HellaSwag~\cite{zellers2019hellaswag}, WinoGrande~\cite{sakaguchi2021winogrande}, ARC-challenge~\cite{clark2018think}, ARC-easy~\cite{clark2018think}, and OpenbookQA~\cite{mihaylov2018can}.
Based on task complexity within arithmetic reasoning (accuracy $\geq$ 90\%), we group MultiArith, AddSub, SingleEq, and MAWPS as easy subtasks, while the remaining ones are classified as hard subtasks. This stratification enables us to evaluate whether the model develops advanced reasoning abilities beyond memorizing basic arithmetic operations from the training data.

\section{Detailed Selection Strategies in Section~\ref{sec:method}}
\label{app:selection}

For the five selection strategies described in Section~\ref{sec:3.2}, we will detail the methods for identifying and selecting important subsets within each linear layer of both MHA and FFN modules in LLMs.

\begin{enumerate}    
[itemsep=0.0pt,topsep=0pt,leftmargin=*]
    \item \textbf{\model-R (\model)}: In this strategy, we will randomly select some heads for the MHA modules and select a few channels for the FFN modules. For the output projection, all channels in the selected heads will be included to enable dense-only computation. In the up and gate projections, we will select a subset of columns, while for the down projection, a few trainable rows will be chosen.
    \item \textbf{\model-W}: This variant selects subsets based on the weight magnitudes (i.e., $\Vert W \Vert_2$) in the MHA and FFN modules. We will test subsets corresponding to both the largest and smallest weights.
    \item \textbf{\model-A}: This variant selects subsets based on the magnitude of activations (i.e., $\Vert A \Vert_2$) on a calibration set, using $1\%$ of the fine-tuning data. Since collecting activations requires only forward passes, this approach maintains the same memory footprint as inference and incurs a negligible increase in training time. Similarly, we evaluate both the largest and smallest activation variants.
    \item \textbf{\model-S}: The Top-K subsets are ranked and selected by the product of the weight and activation magnitudes (i.e, $\Vert W \Vert_2 \cdot \Vert A \Vert_2$). The activation values are collected in a manner similar to \model-A.
    \item \textbf{\model-G}: This variant selects subsets based on the magnitude of gradients on the calibration set. Since gradients are collected without updating the model, we calculate and discard gradients layer by layer during back-propagation similar to Galore~\cite{zhao2024galore}, requiring minimal additional memory.
    
\end{enumerate}

\section{Detailed Experimental Setups for Section~\ref{sec:exp}}
\label{app:exp}

Detailed selection strategies and number of trainable parameters are presented in Section~\ref{sec:exp}. 

\subsection{Dataset Description}

\textbf{Commonsense Reasoning.} The commonsense reasoning dataset comprise eight subsets: BoolQ~\cite{clark2019boolq}, PIQA~\cite{bisk2020piqa}, SocialQA~\cite{sap2019socialiqa}, HellaSwag~\cite{zellers2019hellaswag}, WinoGrande~\cite{sakaguchi2021winogrande}, ARC-challenge~\cite{clark2018think}, ARC-easy~\cite{clark2018think}, and OpenbookQA~\cite{mihaylov2018can}. Following the experimental setup of LLM-Adapters~\cite{llm-adapters}, we split each dataset into training and test sets. Subsequently, we combine the training data from all eight tasks into a single fine-tuning dataset and evaluate performance on the individual test dataset for each task.

\textbf{Arithmetic Reasoning.} We followed Hu et al.~\cite{llm-adapters} and evaluated \model on seven math reasoning tasks, including MultiArith~\cite{roy2016solving}, GSM8K~\cite{cobbe2021training}, AddSub~\cite{hosseini2014learning}, AQuA~\cite{ling2017program}, SingleEq~\cite{koncel2015parsing}, SVAMP~\cite{patel-etal-2021-nlp} and MAWPS~\cite{koncel2016mawps}. Our fine-tuning employed the Math10K dataset~\cite{llm-adapters}, which combines training sets from GSM8K, MAWPS, and AQuA, augmented with LM-generated chain-of-thought steps. Therefore, these three tasks are considered ID, while the remaining four are classified as OOD tasks.

\textbf{Instruction Following.} To further showcase \model's superior generalization ability, we employ the instruction-following fine-tuning task with Alpaca GPT-4 dataset, which comprises 52k samples generated by GPT-4~\cite{gpt4} based on inputs from Alpaca~\cite{alpaca}. Performance is measured on \mbox{MT-Bench~\cite{mt-bench}}, featuring 80 high-quality, multi-turn questions designed to assess LLMs on eight different aspects.

\subsection{Hyperparameter Description}
Additional hyperparameter configurations for all tasks are provided in Table~\ref{commonsense_hp}. We maintain the same hyperparameter settings across the \texttt{LLaMA-7/13B}, \texttt{LLaMA2-7B},  \texttt{LLaMA3-8B}, and \texttt{Mistral-7B} models.

\begin{table}[htbp]
    \footnotesize
    \centering
    \caption{Hyperparameter configurations of \model on various base models across three tasks.}
    \label{commonsense_hp}
    \begin{tabular}{@{}cccccc@{}}
        \toprule
        \textbf{Hyperparameters} & \textbf{Commonsense Reasoning} & \textbf{Arithmetic Reasoning} & \textbf{Instruction Following}\\ \midrule
        Optimizer                                           & AdamW             & AdamW  & AdamW                   \\
        LR                                                  & 2e-4              & 1e-3  & 2e-5               \\
        LR Scheduler                                        & linear            & linear  & cosine        \\
        Batch size                                          & 16$\times$4              & 16$\times$4 & 16$\times$4                      \\
        Warmup Steps                                        & 100               & 100   & 0          \\
        Epochs                                              & 3                 & 3     & 1      \\
        \bottomrule
    \end{tabular}
\end{table}

\section{Proofs for Theoretical Results in Section~\ref{sec:theory}}
\label{app:theory}

Here we provide proofs for the results in Section~\ref{sec:theory}.

\subsection{Notation}

For a vector $a$, let $\|a\|$ be the $\ell_2$ norm of $a$. For $d_1 \geq d_2$, denote \mbox{a set of orthogonal matrices} by $\mathbb{O}_{d_1,d_2} := \{R \in \R^{d_1 \times d_2}: R^\top R = I_{d_2}\}$.
For a matrix $A \in \R^{d_1\times d_2}$, let $\|A\|_\F$ and $\|A\|_\op$ be the Frobenius norm and spectral norm of $A$, respectively.
Denote the condition number of $A$ by $\kappa_*(A) := \|A\|_\op / \lambda_*(A)$. Let $A^\dag$ be \mbox{Moore-Penrose inverse of $A$.
For a symmetric matrix} $A$, denote its effective rank by $r_e(A) := \tr(A) / \|A\|_\op$.
Note that $r_e(A) \leq \rank(A)$ always holds.
For $a, b \in \R$, we let $a \vee b := \max(a, b)$ and $a \wedge b := \min(a, b)$.
For a matrix $A \in \R^{d_1 \times d_2}$, let $\SVD_r(A) := \Phi_r(A) \Lambda_r(A) \Psi_r^\top(A)$ be the top-$r$ singular value decomposition of $A$, where $\Phi_r(A) \in \mathbb{O}_{d_1,r}$ and $\Psi_r(A) \in \mathbb{O}_{d_2,r}$ are top-$r$ left and right \mbox{singular vectors of $A$, respectively}, and $\Lambda_r(A) = \diag(\lambda_1(A), \dots, \lambda_r(A)) \in \R^{r \times r}$ is a diagonal matrix of singular values of $A$, where $\lambda_j(A)$ denotes the $j$-th largest singular value of $A$.
Define $\Phi_*(A) := \Phi_{\rank(A)}(A)$ and $\Psi_*(A) := \Psi_{\rank(A)}(A)$ as the left and right singular vectors of $A$ corresponding to non-zero singular values, respectively. 
Define the smallest \textit{positive} singular value of $A$ as $\lambda_*(A) = \lambda_{\rank(A)}(A)$ and let $\Lambda_*(A) = \Lambda_{\rank(A)}(A)$.
For a deep learning model fine-tuned on $n$ i.i.d. samples $(x_i^{(\id)}, y_i^{(\id)}) \subset \R^p \times \R^q$, we say an event $\mathcal{F}$ occurs with high probability when $\P(\mathcal{F}) = 1 - \exp(-\Omega(\log^2(n+p+q)))$.

\subsection{Setup}

We consider multivariate regression task.
Using $n$ i.i.d. samples $(x_i^{(\id)}, y_i^{(\id)}) \subset \R^p \times \R^q$ from in-distribution task, we fine-tune a pre-trained network $f^\pre: \R^p \to \R^q$ for better prediction.

\paragraph{Deep Linear Networks}
We consider deep linear networks of the form $x \mapsto W_L W_{L-1} \dots W_1 x: \R^d \to \R^p$, where $W_\ell \in \R^{d_\ell \times d_{\ell-1}}$, with $d_L = q$ and $d_0 = p$. 
In comparison to multi-head attention transformers, each row of $W_\ell$ can be viewed as corresponding to the parameters in a single head.
Let $f^\pre(x) = W_L^\pre W_{L-1}^\pre \dots W_1^\pre x: \R^p\to \R^q$ represent a pre-trained neural network.
We denote $\oW_{\ell}^\pre := W_L^\pre W_{L-1}^\pre \dots W_{\ell}^\pre \in \R^{d_L \times d_{\ell-1}}$ as the weights up to the $\ell$-th layer, and $\uW_\ell^\pre := W_\ell^\pre W_{\ell-1}^\pre \dots W_1^\pre \in \R^{d_\ell \times d_0}$ as the weights above the $\ell$-th layer, with the promise that $\uW_0^\mathrm{pre} = I$.
Deep linear networks have been widely used to facilitate the theoretical analysis of modern complex deep neural networks \citep{saxe2013exact,kawaguchi2016deep,lu2017depth,hardt2016identity,laurent2018deep,arora2019implicit}.

\paragraph{Fine-Tuning}
We employ $\ell_2$ distance as the error metric.
Given a pre-trained network $f^\pre$, we fine-tune its $\ell$-th layer by minimizing the empirical in-distribution risk $\mathcal{R}_n^{(\id)}(f) := (1/n) \sum_{i \in [n]} \|y_i^{(\id)} - f(x_i^{(\id)})\|^2$, where $(x_i^{(\id)}, y_i^{(\id)}) \subset \R^p \times \R^q$ are $n$ i.i.d. observations from in-distribution task.
More specifically, we consider a class of rank-$d$ adaptation defined as
\begin{align}
    f_{\ell,U,V}(x) := \oW^\pre_{\ell+1} (W^\pre_\ell + U V^\top) \uW_{\ell-1}^\pre x,\label{eq: rank s adaptation ap}
\end{align}
where $U \in \R^{d_\ell\times d}$ and $V \in \R^{d_{\ell-1} \times d}$ are parameters to fine-tune.
Note that by regarding multiple consecutive layers as a single layer, our settings can be extended to multi-layer fine-tuning.

We specifically compare two fine-tuning methods: LoRA and \model.
\begin{itemize}[itemsep=0.0pt,topsep=0pt,leftmargin=*]
    \item \textbf{LoRA.} 
        For a fixed $\ell \in [L]$, and low-rankness level $1 \leq r \leq \min\{d_\ell, d_{\ell-1}\}$, we train the low-rank matrices $(U, V)$ in \eqref{eq: rank s adaptation ap} by minimizing the empirical in-distribution risk via gradient descent.
        Motivated from the previous results that gradient descent has implicit regularization \citep{zhang2021understanding,gunasekar2017implicit,arora2019implicit}, we directly consider the minimum norm solutions:
        \begin{align}
            (U^\lora, V^\lora) &\in \argmin_{U, V} \|(U, V)\|_\F^2\ \ \text{ s.t. $(U, V)$ minimizes $\mathcal{R}_n^{(\id)}(f_{\ell,U,V})$}.\label{eq: lora adaptation matrix}
        \end{align}
    \item \textbf{\model.}
        For a fixed $\ell \in [L]$, and a sparsity level $s = \lfloor r \cdot \frac{d_\ell + d_{\ell-1}}{d_{\ell-1}} \rfloor$, we train only $V$ in \eqref{eq: rank s adaptation ap} with the fixed choice of $U \gets U_S^\sft := [e_{a_1}; e_{a_2}; \dots; e_{a_s}]$, which specifies $s$ channels to fine-tune, where $S = \{a_1, a_2, \dots, a_s\} \subset [d_\ell]$. Here $e_a$ is the standard basis vector with the $a$-th entry being $1$.
        We minimize the empirical in-distribution risk via gradient descent. 
        Similar to LoRA, we consider the following minimum norm solution:
        \begin{align}
            V^\sft = \argmin_V \|V\|_\F^2 \ \ \text{ s.t. $V$ minimizes $\mathcal{R}_n^{(\id)}(f_{\ell,U_S^\sft,V})$}.\label{eq: sft adaptation matrix}
        \end{align}
\end{itemize}

\paragraph{Data Generating Process}

As a simplification of the data generating process, we consider multiple linear regression. 
Assume that the in-distribution data $(x^{(\id)}, y^{(\id)}) \in \R^{p+q}$ and out-of-distribution data $(x^{(\ood)}, y^{(\ood)}) \in \R^{p+q}$ are generated according to 
\begin{align}
    y^{(k)} &= B^{(k)} x^{(k)} + \epsilon^{(k)}, \ \ k \in \{\id, \ood\},\label{model: linear regression}
\end{align}
where $B^{(k)} \in \R^{q \times p}$, and $\epsilon^{(k)} \in \R^q$ is the error term satisfying $\E[\epsilon^{(k)} | x^{(k)}] = 0$.
Assume that $\Sigma_\epsilon^{(k)} := \E[\epsilon^{(k)} \epsilon^{(k) \top}] \in \R^{q \times q}$ exists and $\E[x^{(k)}] = 0$. The signal covariance matrix is denoted by $\Sigma_x^{(k)} := \E[x^{(k)} x^{(k) \top}] \in \R^{p \times p}$.


We define the in-distribution and out-of-distribution risks of $f: \R^p \to \R^q$ as:
\begin{align*}
     \mathcal{R}^{(k)}(f) = \E[\|y^{(k)} - f(x^{(k)})\|], \ \ k \in \{\id, \ood\}.
\end{align*}

For notational brevity, we can write $W^\pre = \uW_L^\pre \in \R^{q \times p}$. 
Let $X^{(\id)} := (x_1^{(\id)}, \dots, x_n^{(\id)}) \in \R^{p \times n}$, \mbox{$Y^{(\id)} := (y_1^{(\id)}, \dots, y_n^{(\id)}) \in \R^{q \times n}$, and $E^{(\id)} = (\epsilon_1^{(\id)}, \dots, \epsilon_n^{(\id)}) := Y^{(\id)} - B^{(\id)} X^{(\id)} \in \R^{q \times n}$.
Denote} the in-distribution sample covariance matrices by $\hat \Sigma_x^{(\id)} := (1/n) X^{(\id)} X^{(\id) \top}$, $\hat \Sigma_\epsilon^{(\id)} := (1/n) E^{(\id)} E^{(\id) \top}$, \mbox{$\hat \Sigma_{x,\epsilon}^{(\id)} := (1/n) X^{(\id)} E^{(\id) \top}$, $\hat \Sigma_{\epsilon,x}^{(\id)} = \hat \Sigma_{x,\epsilon}^{(\id) \top}$. Define $\check \Sigma_{x,\epsilon}^{(k)} = (X^{(\id) \top})^\dag E^{(\id) \top}$, $\hat A := (\uW^\pre_{\ell-1} \hat \Sigma_x^{(\id)} \uW^{\pre \top}_{\ell-1})^{1/2}$,} $A := (\uW^\pre_{\ell-1} \Sigma_x^{(\id)} \uW^{\pre \top}_{\ell-1})^{1/2}$, $\Phi' := \Phi_*(\oW^\pre_{\ell+1})$, $\Phi_S'' := \Phi_*(\oW^\pre_{\ell+1} U_S^\sft)$, $D = B^{(\id)} - W^\pre$, $\hat D := B^{(\id)} - W^\pre + \check\Sigma_{\epsilon,x}^{(\id)}$.
Also define $M := \Phi'^\top D \Sigma_x^{(\id)} \uW^{\pre \top}_{\ell-1} A^\dag$ and $\hat M := \Phi'^\top \hat D \hat \Sigma_x^{(\id)} \uW^{\pre \top}_{\ell-1} \hat A^\dag$.
Let $\hat \Psi' := \Psi_*(\hat A)$, and $G_\ell^{(\id,\ood)} := (\uW^\pre_\ell \Sigma_x^{(\id) 1/2})^\dag \uW^\pre_\ell \Sigma_x^{(\ood) 1/2}$ be a matrix that captures the covariate shift at the $\ell$-th layer.


We consider fine-tuning the $\ell$-th ($\ell \in [L]$) layer of the pre-trained deep linear network $f^\pre(x) = W_L^\pre W_{L-1}^\pre \dots W_1^\pre x$ using in-distribution observations $(x_i^{(\id)}, y_i^{(\id)})_{i \in [n]}$.

To measure the performance of models, we define the excess risks of $f$ for the task $k \in \{\id, \ood\}$ as
\begin{align*}
    \mathcal{E}^{(k)}(f) := \E[\|y^{(k)} - f(x^{(k)})\|^2] - \inf_{f'} \E[\|y^{(k)} - f'(x^{(k)})\|^2],
\end{align*}
where the infimum is taken over all square integrable functions.

\subsection{Assumptions}\label{sec: assumptions ap}

We assume that $\uW^\pre_{\ell-1} \Sigma_x^{(\id)} \uW^{\pre \top}_{\ell-1} \neq 0$, since otherwise $\uW^\pre_{\ell-1} x^{(\id)} = 0$ almost surely and fine-tuning the $\ell$-th layer does not improve the performance of the pre-trained model. 
Define the in-distribution prediction residuals for the pre-trained model $f^\pre$ by $\Sigma^{(\id)}_f := \E[(B^{(\id)} x^{(\id)} - W^\pre x^{(\id)}) (B^{(\id)} x^{(\id)} - W^\pre x^{(\id)})^\top]$. Note that $\mathcal{E}^{(\id)}(f^\pre) = \tr(\Sigma^{(\id)}_f)$.
We also assume that $\|\Sigma^{(\id)}_f\|_\op > 0$, since otherwise $\mathcal{E}^{(\id)}(f^\pre) = \|\Sigma^{(\id)}_f\|_\F^2 = 0$ and there is no room for improvement from the pre-trained model.

Next, we introduce several assumptions.
\begin{assumption}[Sub-Gaussianity]\label{asm: sub gaussian x e ap}
    Assume that there exist some constants $c_1, c_2 \in (0, \infty)$ such that $(x^{(\id)}, \epsilon^{(\id)})$ in the model~\ref{model: linear regression} satisfies
    \begin{align*}
        \gamma^\top \Sigma_x^{(\id)} \gamma &\geq c_1 \|\gamma^\top x^{(\id)}\|_{\psi_2}^2, \ \ \text{ and } \ \ \gamma'^\top \Sigma_\epsilon^{(\id)} \gamma' \geq c_2 \|\gamma'^\top \epsilon^{(\id)}\|_{\psi_2}^2,
    \end{align*}
    for any $\gamma \in \R^p$ and $\gamma' \in \R^q$, where $\|y\|_{\psi_2}$ is the sub-Gaussian norm defined as
    \begin{align*}
        \|y\|_{\psi_2} := \inf \{\upsilon > 0: \E[\exp(y^2/\upsilon^2)] \leq 2\}
    \end{align*}
    for a random variable $y$ taking values in $\R$.
\end{assumption}

\begin{assumption}[Sufficiently Many Observations]\label{asm: regime ap}
    Assume that
    \begin{align*}
        n &\gg (\kappa_*^4(A) r_e(A^2) + \kappa_*^2(\Sigma_x^{(\id)}) r_e(\Sigma_x^{(\id)}) + r_e(D \Sigma_x^{(\id)} D^\top)) \log^2(n+p+q),\\
        n &\gg \frac{\|\Sigma_\epsilon^{(\id)}\|_\op}{\|D \Sigma_x^{(\id)} D^\top\|_\op} (r_e(\Sigma_\epsilon^{(\id)}) + r_e(A^2)) \log^2(n+p+q),
    \end{align*}
    and
    \begin{align*}
        n \gg \kappa_*^4(\Sigma_x^{(\id)}) \frac{r_e(\Sigma_x^{(\id)}) (r_e(\Sigma_\epsilon^{(\id)}) + r_e(\Sigma_x^{(\id)}))}{r_e(A^2)} \log^2(n+p+q).
    \end{align*}
\end{assumption}

\begin{assumption}[Eigengap Condition]\label{asm: eigengap ap}
    Assume that there exists some constant $C_{\textnormal{g}} > 0$ such that
    \begin{align*}
        \frac{\lambda_s(\Phi'^\top D \Sigma_x^{(\id)} \uW^{\pre \top}_{\ell-1} A^\dag)}{\lambda_s(\Phi'^\top D \Sigma_x^{(\id)} \uW^{\pre \top}_{\ell-1} A^\dag) - \lambda_{s+1}(\Phi'^\top D \Sigma_x^{(\id)} \uW^{\pre \top}_{\ell-1} A^\dag)} \lesssim C_{\textnormal{g}}
    \end{align*}
    holds.
\end{assumption}
Assumption~\ref{asm: eigengap ap} is necessary to identify the rank-$r$ approximation of $M$, which is used to derive the risk of LoRA. 

\begin{assumption}[Approximate Sparsity of Channels]\label{asm: sparsity ap}
    Assume that there exists some $S_0 \subset [d_\ell]$ with $|S_0| \leq s$ and $\delta > 0$ such that
    \begin{align*}
        \sum_{a \in [d_\ell] \setminus S_0} \|e_a^\top (\oW^\pre_{\ell+1})^\dag (B^{(\id)} - W^\pre) \Sigma_x^{(\id) 1/2}\|^2 \leq \delta^2 \|(\oW^\pre_{\ell+1})^\dag (B^{(\id)} - W^\pre) \Sigma_x^{(\id) 1/2}\|_\F^2
    \end{align*}
    holds.
\end{assumption}

\begin{assumption}[Distribution Shift]\label{asm: distribution shift ap}
    Assume that $\Sigma_x^{(\id)} = \Sigma_x^{(\ood)} = \Sigma_x$ for some  $\Sigma_x \in \R^{d \times d}$ and that $\|\Phi_*^\top(\oW^\pre_{\ell+1} U_S^\sft) (B^{(\ood)} - B^{(\id)}) \Sigma_x^{1/2}\|_\F^2 \leq \varepsilon^2 \mathcal{E}^{(\ood)}(f^\pre)$ for some $\varepsilon > 0$. 
\end{assumption}

\begin{assumption}[Condition Number]\label{asm: condition number ap}
    Assume that $\kappa_*(M) \lesssim 1$, $\kappa_*(\oW^\pre_{\ell+1}) \lesssim 1$, $\kappa_*(\Sigma^{(\id)}_f) \lesssim 1$ and $\kappa_*(\uW^\pre_{\ell-1} \Sigma_x^{(\id)} \uW^{\pre \top}_{\ell-1}) \lesssim 1$.
\end{assumption}
Note that Assumption~\ref{asm: condition number ap} is not essential to our analysis.

\subsection{Main Results}

We first demonstrate that LoRA and \model exhibit comparable memorization abilities. Next, we present a formal restatement of \ref{thm: sft and lora ood informal} that combine Theorems~\ref{thm: lora in-distribution}, \ref{thm: lora ood}, \ref{thm: sft in-distribution}, \ref{thm: sft ood}, and Lemma~\ref{lem: sparsity}.

\begin{theorem}\label{thm: sft and lora id}
    Suppose that Assumptions~\ref{asm: sub gaussian x e ap}, \ref{asm: regime ap}, \ref{asm: eigengap ap}, \ref{asm: sparsity ap}, and \ref{asm: condition number ap} hold. Choose $S$ such that $S \supset S_0$ holds.
    Let $U^\lora, V^\lora$ be the LoRA adaptation matrices defined in \eqref{eq: lora adaptation matrix}.
    Let $V^\sft$ be the \model adaptation matrices given $U_S^\sft$ defined in \eqref{eq: sft adaptation matrix}.
    Then, for all sufficiently large $n$, the following holds with probability $1 - \exp(-\Omega(\log^2(n+p+q)))$:
    for any $\eta > 0$,
    \begin{align*}
        \mathcal{E}^{(\id)}(f_{\ell,U_S^\sft,V^\sft}) &\leq (1 + \eta) (T_\bias^\sft)^2 + (1 + \eta^{-1}) (T_\variance^\sft)^2,\\
        \mathcal{E}^{(\id)}(f_{\ell,U^\lora,V^\lora}) &\leq (1 + \eta) (T_\bias^\lora)^2 + (1 + \eta^{-1}) (T_\variance^\lora)^2,
    \end{align*}
    where 
    \begin{align*}
        0 &\leq (T_\bias^\lora)^2 - \mathcal{E}^{(\id)}(f_\ell^\full) \simeq (T_\bias^\sft)^2 - \mathcal{E}^{(\id)}(f_\ell^\full) \lesssim \delta^2 \mathcal{E}^{(\id)}(f^\pre),\\
        (T_\variance^\sft)^2 &\lesssim (\|\Sigma_\epsilon^{(\id)}\|_\op + \|\Sigma^{(\id)}_f\|_\op) \frac{s d_{\ell-1} \log^2(n+p+q)}{n},\\
        (T_\variance^\lora)^2 &\lesssim (\|\Sigma_\epsilon^{(\id)}\|_\op + \|\Sigma^{(\id)}_f\|_\op) \frac{r (d_\ell + d_{\ell-1}) \log^2(n+p+q)}{n}.
    \end{align*}
\end{theorem}

\begin{theorem}[Restatement of Theorem~\ref{thm: sft and lora ood informal}]\label{thm: sft and lora ood}
    Consider the limit $n \to \infty$. Suppose that Assumption~\ref{asm: distribution shift ap} holds.
    Let $U^\lora, V^\lora$ be the LoRA adaptation matrices defined in \eqref{eq: U V infty lora}.
    Let $V^\sft$ be the \model adaptation matrices given $U_S^\sft$ defined in \eqref{eq: V infty sft}.
    If $B^{(\id)} = \oW^\pre_{\ell+1} \tilde B \uW^\pre_{\ell-1}$ holds for some $\tilde B^{(\id)} \in \R^{d_\ell \times d_{\ell-1}}$, 
    and $s, r \leq \rank(\Sigma_f^{(\id)})$, then,
    \begin{align*}
        \mathcal{E}^{(\ood)}(f_{\ell,U_S^\sft,V^\sft}) &\leq (1 + 3\varepsilon^2) \mathcal{E}^{(\ood)}(f^\pre),\\
        \mathcal{E}^{(\ood)}(f_{\ell,U^\lora,V^\lora}) &\geq \|(B^{(\ood)} - B^{(\id)}) \Sigma_x^{1/2}\|_\F^2.
    \end{align*}
\end{theorem}

\begin{proof}[Intuition of the proof of Theorem~\ref{thm: sft and lora ood}]
    LoRA forgets pre-trained tasks due to its model complexity. 
    Consider the simplest low-rank adaptation to a single-layer linear network:
    \begin{align*}
        \Delta_1 \in \argmin_{\substack{\Delta_1' \in \R^{d_1 \times d_0}\\\rank(\Delta_1')=r}} \E[\|y^{(\id)} - (W_1^\pre + \Delta_1') x^{(\id)}\|^2].
    \end{align*}
    Assume that $\Sigma_x^{(\id)} = I$, then we can show that the solution is $\Delta_1 = \SVD_r(B^{(\id)} - W_1^\pre)$.
    Under the condition that the rank of $B^{(\id)} - W_1^\pre$ is smaller than, or comparable to $r$,
    LoRA fine-tuned model can learn the in-distribution best regressor in $\ell_2$ sense, since $(W_1^\pre + \Delta_1) x \approx B^{(\id)} x = \E[y^{(\id)} | x^{(\id)} = x]$. Hence it makes LoRA fine-tuned model vulunerable to distribution shift.
    
    On the other hand, we model \model as fine-tuning only a few channels:
    \begin{align*}
        \Delta_1 \in \argmin_{\substack{\Delta_1' = \sum_{a \in S} e_a v_a^\top, v_a \in \R^{d_0}}} \E[\|y^{(\id)} - (W_1^\pre + \Delta_1') x^{(\id)}\|^2].
    \end{align*}
    Although \model is a special case of LoRA,
    the constraint on the direction of low-rank matrix prevents overfitting to the in-distribution task.
    To see this, note that a sparse fine-tuned model can be written as
    \begin{align*}
        (W_1^\pre + \Delta_1) x = W_1^\pre x + \sum_{a \in S} e_a e_a^\top (B^{(\id)} - W_1^\pre) x &= \sum_{a \in S^c} e_a e_a^\top W_1^\pre x + \sum_{a \in S} e_a e_a^\top B^{(\id)} x,
    \end{align*}
    where $S \subset [d_1]$ is a set of channels with cardinality $s$.
    Since \model keeps most of parameters from the pre-trained model, except for rows specified by $S$, the model forget less pre-training tasks.
\end{proof}

\subsection{Proofs for LoRA}

\subsubsection{Excess Risk of LoRA}

\begin{lemma}[Excess Risk]\label{lem: excess risk lora}
    Consider the minimum norm solution
    \begin{align*}
        (U^\lora, V^\lora) \in \argmin_{(U, V) \in\R^{d_\ell \times r} \times \R^{d_{\ell-1} \times r}} \|(U, V)\|_\F^2\ \ \text{ s.t. $(U, V)$ minimizes $\mathcal{R}_n^{(\id)}(f_{\ell,U,V})$. }
    \end{align*}
    Then, the low-rank adaptation matrix satisfies
    \begin{align*}
        U^\lora V^{\lora \top} &= (\oW^\pre_{\ell+1})^\dag \SVD_r(\oW^\pre_{\ell+1} (\oW^\pre_{\ell+1})^\dag \hat D \hat \Sigma_x^{(\id)} \uW^{\pre \top}_{\ell-1} \hat A^\dag) \hat A^\dag,
    \end{align*}
    and
    \begin{align*}
        \mathcal{E}^{(k)}(f_{\ell,U^\lora,V^\lora}) &= \tr\biggl(\qty(B^{(k)} - W^\pre - \SVD_r(\oW^\pre_{\ell+1} (\oW^\pre_{\ell+1})^\dag \hat D \hat \Sigma_x^{(\id)} \uW^{\pre \top}_{\ell-1} \hat A^\dag) \hat A^\dag \uW^\pre_{\ell-1}) \Sigma_x^{(k)}\\
        &\quad\cdot \qty(B^{(k)} - W^\pre - \SVD_r(\oW^\pre_{\ell+1} (\oW^\pre_{\ell+1})^\dag \hat D \hat \Sigma_x^{(\id)} \uW^{\pre \top}_{\ell-1} \hat A^\dag) \hat A^\dag \uW^\pre_{\ell-1})^\top\biggr)
    \end{align*}
    for $k \in \{\id, \ood\}$.
\end{lemma}

\begin{proof}[Proof of Lemma~\ref{lem: excess risk lora}]
    The empirical risk of $f_{\ell,U,V}$ for the in-distribution task can be written as
    \begin{align}
        \mathcal{R}_n^{(\id)}(f_{\ell,U,V}) &= \frac{1}{n} \sum_{i \in [n]} \|(B^{(\id)} - W^\pre) x_i^{(\id)} + \epsilon_i^{(\id)} - \oW^\pre_{\ell+1} U V^\top \uW^\pre_{\ell-1} x_i^{(\id)}\|^2\nonumber\\
        &= \tr((B^{(\id)} - W^\pre - \oW^\pre_{\ell+1} U V^\top \uW^\pre_{\ell-1}) \hat \Sigma_x^{(\id)} (B^{(\id)} - W^\pre - \oW^\pre_{\ell+1} U V^\top \uW^\pre_{\ell-1})^\top)\nonumber\\
        &\quad+ 2\tr((B^{(\id)} - W^\pre - \oW^\pre_{\ell+1} U V^\top \uW^\pre_{\ell-1}) \hat \Sigma_{x,\epsilon}^{(\id)}) + \tr(\hat \Sigma_\epsilon^{(\id)})\nonumber\\
        &= \tr(V^\top \uW^\pre_{\ell-1} \hat \Sigma_x^{(\id)} \uW^{\pre \top}_{\ell-1} V U^\top \oW^{\pre \top}_{\ell+1} \oW^\pre_{\ell+1} U)\nonumber\\
        &\quad- 2 \tr(\oW^\pre_{\ell+1} U V^\top \uW^\pre_{\ell-1} \qty{\hat \Sigma_x^{(\id)} (B^{(\id)} - W^\pre)^\top + \hat \Sigma_{x,\epsilon}^{(\id)}})\nonumber\\
        &\quad+ \tr((B^{(\id)} - W^\pre) \hat \Sigma_x^{(\id)} (B^{(\id)} - W^\pre)^\top) + 2 \tr( (B^{(\id)} - W^\pre) \hat \Sigma_{x,\epsilon}^{(\id)}) + \tr(\hat \Sigma_\epsilon^{(\id)}).\label{eq: min lora risk}
    \end{align}
    Since $\hat \Sigma_{x,\epsilon}^{(\id)} = \hat \Sigma_x^{(\id)} (X^{(\id) \top})^\dag E^{(\id) \top} = \hat \Sigma_x^{(\id)} \check \Sigma_{x,\epsilon}^{(\id)}$,
    \begin{align}
        \mathcal{R}_n^{(\id)}(f_{\ell,U,V}) &= \tr(\hat A V U^\top \oW^{\pre \top}_{\ell+1} \oW^\pre_{\ell+1} U V^\top \hat A) - 2 \tr(\oW^\pre_{\ell+1} U V^\top \hat A \hat A^\dag \uW^\pre_{\ell-1} \hat \Sigma_x^{(\id)} \hat D^\top)\nonumber\\
        &\quad- 2 \tr(\oW^\pre_{\ell+1} U V^\top (I - \hat A \hat A^\dag) \uW^\pre_{\ell-1} \hat \Sigma_x^{(\id)} \hat D^\top)\nonumber\\
        &\quad+ \tr(D \hat \Sigma_x^{(\id)} D^\top) + 2 \tr( D \hat \Sigma_{x,\epsilon}^{(\id)}) + \tr(\hat \Sigma_\epsilon^{(\id)})\nonumber\\
        &= \|\oW^\pre_{\ell+1} U V^\top \hat A - \hat D \hat \Sigma_x^{(\id)} \uW^{\pre \top}_{\ell-1} \hat A^\dag\|_\F^2 - \|\hat D \hat \Sigma_x^{(\id)} \uW^{\pre \top}_{\ell-1} \hat A^\dag\|_\F^2\nonumber\\
        &\quad+ \tr(D \hat \Sigma_x^{(\id)} D^\top) + 2 \tr( D \hat \Sigma_{x,\epsilon}^{(\id)}) + \tr(\hat \Sigma_\epsilon^{(\id)}),\label{eq: risk lora equivalent}
    \end{align}
    where we used $(I - \hat A \hat A^\dag) \uW^\pre_{\ell-1} \hat \Sigma_x^{(\id) 1/2} = 0$.
    From \eqref{eq: risk lora equivalent}, minimizing $\mathcal{R}_n^{(\id)}(f_{\ell,U,V})$ is equivalent to minimizing the norm:
    \begin{align*}
        \|\oW^\pre_{\ell+1} U V^\top \hat A - \hat D \hat \Sigma_x^{(\id)} \uW^{\pre \top}_{\ell-1} \hat A^\dag\|_\F^2 &= \|\oW^\pre_{\ell+1} U V^\top \hat A - \oW^\pre_{\ell+1} (\oW^\pre_{\ell+1})^\dag \hat D \hat \Sigma_x^{(\id)} \uW^{\pre \top}_{\ell-1} \hat A^\dag\|_\F^2\\
        &\quad+ \|(I - \oW^\pre_{\ell+1} (\oW^\pre_{\ell+1})^\dag) \hat D \hat \Sigma_x^{(\id)} \uW^{\pre \top}_{\ell-1} \hat A^\dag\|_\F^2.
    \end{align*}
    This is minimized by $(U', V')$ satisfying
    \begin{align}
        U' V'^\top &= (\oW^\pre_{\ell+1})^\dag \SVD_r(\oW^\pre_{\ell+1} (\oW^\pre_{\ell+1})^\dag \hat D \hat \Sigma_x^{(\id)} \uW^{\pre \top}_{\ell-1} \hat A^\dag) \hat A^\dag\nonumber\\
        &\quad+ (I - (\oW^\pre_{\ell+1})^\dag \oW^\pre_{\ell+1}) A_1 + A_2 (I - \hat \Psi' \hat \Psi'^\top),\label{eq: U prime V prime}
    \end{align}
    where $A_1, A_2 \in \R^{d_\ell \times d_{\ell-1}}$ are arbitrary matrices.
    Since we particularly consider the minimum norm solution, we must have $A_1 = 0$ and $A_2 = 0$. Hence
    \begin{align*}
        \oW^\pre_{\ell+1} U^\lora V^{\lora \top} \uW^\pre_{\ell-1} = \SVD_r(\oW^\pre_{\ell+1} (\oW^\pre_{\ell+1})^\dag \hat D \hat \Sigma_x^{(\id)} \uW^{\pre \top}_{\ell-1} \hat A^\dag) \hat A^\dag \uW^\pre_{\ell-1}.
    \end{align*}
    Therefore, the excess risk for $k \in \{\id, \ood\}$ becomes
    \begin{align*}
        \mathcal{E}^{(k)}(f_{\ell,U^\lora,V^\lora}) &= \E\qty[\qty(B^{(k)} x^{(k)} - \oW^\pre_{\ell+1} (W^\pre_{\ell} + U^\lora V^{\lora \top}) \uW^\pre_{\ell-1} x^{(k)})^2]\\
        &= \tr\biggl(\qty(B^{(k)} - W^\pre - \SVD_r(\oW^\pre_{\ell+1} (\oW^\pre_{\ell+1})^\dag \hat D \hat \Sigma_x^{(\id)} \uW^{\pre \top}_{\ell-1} \hat A^\dag) \hat A^\dag \uW^\pre_{\ell-1}) \Sigma_x^{(k)}\\
        &\quad\cdot \qty(B^{(k)} - W^\pre - \SVD_r(\oW^\pre_{\ell+1} (\oW^\pre_{\ell+1})^\dag \hat D \hat \Sigma_x^{(\id)} \uW^{\pre \top}_{\ell-1} \hat A^\dag) \hat A^\dag \uW^\pre_{\ell-1})^\top\biggr).
    \end{align*}
    This concludes the proof.
\end{proof}

\subsubsection{In-distribution Excess Risk of LoRA}

Let $\mathcal{E}^{(\id)}(f_\ell^\full)$ denote the excess risk of $f^\pre$ after fine-tuning all the parameters of the $\ell$-th layer under \textit{population} in-distribution risk.

\begin{theorem}[Restatement of Theorem~\ref{thm: sft and lora id}: LoRA Part]\label{thm: lora in-distribution}
    Suppose that Assumptions~\ref{asm: sub gaussian x e ap}, \ref{asm: regime ap} and \ref{asm: eigengap ap} hold.
    Then, the following holds with probability $1 - \exp(-\Omega(\log^2(n+p+q)))$.
    For any $\eta > 0$,
    \begin{align*}
        \mathcal{E}^{(\id)}(f_{\ell,U^\lora,V^\lora}) \leq (1 + \eta) (T_\bias^\lora)^2 + (1 + \eta^{-1}) (T_\variance^\lora)^2,
    \end{align*}
    where
    \begin{align}
        (T_\bias^\lora)^2 &\leq \frac{0 \vee (\rank(D \Sigma_x^{(\id)} D^\top) - r)}{\rank(D \Sigma_x^{(\id)} D^\top)} \kappa_*^2(D \Sigma_x^{(\id)} D^\top) \mathcal{E}^{(\id)}(f^\pre) + \mathcal{E}^{(\id)}(f_\ell^\full),\label{eq: bias lora}\\
        (T_\variance^\lora)^2 &\lesssim C^2 \kappa_*^4(M) \|\Sigma_\epsilon^{(\id)}\|_\op \kappa_*^2(A) \frac{r(r_e(\Phi'^\top \Sigma_\epsilon^{(\id)} \Phi') + r_e(A^2)) \log^2(n+p+q)}{n}\nonumber\\
        &\quad+ C^2 \kappa_*^4(M) \|D \Sigma_x^{(\id)} D^\top\|_\op \frac{r(\kappa_*^2(A) r_e(\Phi'^\top D \Sigma_x^{(\id)} D^\top \Phi') + \kappa_*^6(A) r_e(A^2)) \log^2(n+p+q)}{n}.\nonumber
    \end{align}
\end{theorem}
Note that the first term on the right hand side of \eqref{eq: bias lora} depends on the rank of residual matrix $\Sigma_f^{(\id)} = D \Sigma_x^{(\id)} D^\top$.
It becomes zero when $\rank(\Sigma_f^{(\id)}) \leq r$ and small when $r/\rank(\Sigma_f^{(\id)}) \approx 1$. 

\begin{proof}[Proof of Theorem~\ref{thm: lora in-distribution}]
    Let $\oW_\ell^\lora := \oW^\pre_{\ell+1} U^\lora V^{\lora \top}$.
    From Lemma~\ref{lem: excess risk lora}, we have
    \begin{align*}
        \mathcal{E}^{(\id)}(f_{\ell,U^\lora,V^\lora}) &= \tr((D - \oW_\ell^\lora \uW^\pre_{\ell-1}) \Sigma_x^{(\id)} (D - \oW_\ell^\lora \uW^\pre_{\ell-1})^\top)\\
        &= \|(\oW_\ell^\lora A A^\dag \uW^\pre_{\ell-1} - D) \Sigma_x^{(\id) 1/2}\|_\F^2,
    \end{align*}
    where we used $(I - A A^\dag) \uW^\pre_{\ell-1} \Sigma_x^{(\id) 1/2} = 0$. From Lemma~\ref{lem: excess risk lora}
    \begin{align*}
        \oW_\ell^\lora A = \SVD_r(\oW^\pre_{\ell+1} (\oW^\pre_{\ell+1})^\dag \hat D \hat \Sigma_x^{(\id)} \uW^{\pre \top}_{\ell-1} \hat A^\dag) \hat A^\dag A.
    \end{align*}
    This gives
    \begin{align*}
        \|(\oW_\ell^\lora A A^\dag \uW^\pre_{\ell-1} - D) \Sigma_x^{(\id) 1/2}\|_\F &\leq \|(\oW_\ell^\lora A - \SVD_r(\oW^\pre_{\ell+1} (\oW^\pre_{\ell+1})^\dag D \Sigma_x^{(\id)} \uW^{\pre \top}_{\ell-1} A^\dag)) A^\dag \uW^\pre_{\ell-1} \Sigma_x^{(\id) 1/2}\|_\F\\
        &\quad+ \|\SVD_r(\oW^\pre_{\ell+1} (\oW^\pre_{\ell+1})^\dag D \Sigma_x^{(\id)} \uW^{\pre \top}_{\ell-1} A^\dag) A^\dag \uW^\pre_{\ell-1} \Sigma_x^{(\id) 1/2} - D \Sigma_x^{(\id) 1/2}\|_\F\\
        &=: T_\variance^\lora + T_\bias^\lora.
    \end{align*}
    We bound $T_\variance^\lora$ and $T_\bias^\lora$ separately.
    
    For the term $T_\variance^\lora$, since $A^\dag \uW^\pre_{\ell-1} \Sigma_x^{(\id)} \uW^{\pre \top}_{\ell-1} A^\dag = A^\dag A^2 A^\dag$,
    \begin{align*}
        T_\variance^\lora
        &= \|\SVD_r(\oW^\pre_{\ell+1} (\oW^\pre_{\ell+1})^\dag \hat D \hat \Sigma_x^{(\id)} \uW^{\pre \top}_{\ell-1} \hat A^\dag) \hat A^\dag A - \SVD_r(\oW^\pre_{\ell+1} (\oW^\pre_{\ell+1})^\dag D \Sigma_x^{(\id)} \uW^{\pre \top}_{\ell-1} A^\dag) A^\dag A\|_\F.
    \end{align*}
    Therefore,
    \begin{align*}
        T_\variance^\lora &\leq \|\SVD_r(\oW^\pre_{\ell+1} (\oW^\pre_{\ell+1})^\dag D \Sigma_x^{(\id)} \uW^{\pre \top}_{\ell-1} A^\dag) A^\dag A - \SVD_r(\oW^\pre_{\ell+1} (\oW^\pre_{\ell+1})^\dag \hat D \hat \Sigma_x^{(\id)} \uW^{\pre \top}_{\ell-1} \hat A^\dag) A^\dag A\|_\F\\
        &\quad+ \|\SVD_r(\oW^\pre_{\ell+1} (\oW^\pre_{\ell+1})^\dag \hat D \hat \Sigma_x^{(\id)} \uW^{\pre \top}_{\ell-1} \hat A^\dag) (\hat A^\dag A - A^\dag A)\|_\F\\
        &=: T_{\variance,1}^\lora + T_{\variance,2}^\lora,
    \end{align*}
    We first bound $T_{\variance,1}^\lora$.
    From Lemma~\ref{lem: rank s approx} and Assumption~\ref{asm: eigengap ap}, we have
    \begin{align*}
        T_{\variance,1}^\lora &\leq \|\SVD_r(\hat M) - \SVD_r(M)\|_\F\\
        &\leq \kappa_*^2(M) \frac{\lambda_r(M)}{\lambda_r(M) - \lambda_{r+1}(M)} \sqrt{r} \|\hat M - M\|_\op\\
        &\leq \kappa_*^2(M) C \sqrt{r} \|\hat M - M\|_\op,
    \end{align*}
    where $\hat M = \oW^\pre_{\ell+1} (\oW^\pre_{\ell+1})^\dag \hat D \hat \Sigma_x^{(\id)} \uW^{\pre \top}_{\ell-1} \hat A^\dag$ and $M = \oW^\pre_{\ell+1} (\oW^\pre_{\ell+1})^\dag D \Sigma_x^{(\id)} \uW^{\pre \top}_{\ell-1} A^\dag$.
    From Lemma~\ref{lem: good event},
    \begin{align*}
        \|\hat M - M\|_\op &\leq \|\Phi'^\top \hat D \hat \Sigma_x^{(\id)} \uW^{\pre \top}_{\ell-1} - \Phi'^\top D \Sigma_x^{(\id)} \uW^{\pre \top}_{\ell-1}\|_\op \|\hat A^\dag\|_\op\\
        &\quad+ \|D \Sigma_x^{(\id)} \uW^{\pre \top}_{\ell-1}\|_\op \|\hat A^\dag - A^\dag\|_\op\\
        &\lesssim \|\Sigma_\epsilon^{(\id)}\|_\op^{1/2} \kappa_*(A) \sqrt{\frac{(r_e(\Phi'^\top \Sigma_\epsilon^{(\id)} \Phi') + r_e(A^2)) \log^2(n+p+q)}{n}}\\
        &\quad\quad+ \|D \Sigma_x^{(\id)} D^\top\|_\op^{1/2} \kappa_*(A) \sqrt{\frac{(r_e(\Phi'^\top D \Sigma_x^{(\id)} D^\top \Phi') + r_e(A^2)) \log^2(n+p+q)}{n}}\\
        &\quad\quad+ \|D \Sigma_x^{(\id)} \uW^{\pre \top}_{\ell-1}\|_\op \frac{\kappa_*(A)}{\lambda_*(A)} \sqrt{\frac{r_e(A^2) \log^2(n+p+q)}{n}}\\
        &\lesssim \|\Sigma_\epsilon^{(\id)}\|_\op^{1/2} \kappa_*(A) \sqrt{\frac{(r_e(\Phi'^\top \Sigma_\epsilon^{(\id)} \Phi') + r_e(A^2)) \log^2(n+p+q)}{n}}\\
        &\quad\quad+ \|D \Sigma_x^{(\id)} D^\top\|_\op^{1/2} \sqrt{\frac{(\kappa_*^2(A) r_e(\Phi'^\top D \Sigma_x^{(\id)} D^\top \Phi') + \kappa_*^4(A) r_e(A^2)) \log^2(n+p+q)}{n}}
    \end{align*}
    holds on the event $\mathcal{F}$, where we used $\|D \Sigma_x^{(\id)} \uW^{\pre \top}_{\ell-1}\|_\op \leq \|D \Sigma_x^{(\id) 1/2}\|_\op \|A\|_\op$.
    Hence
    \begin{align*}
        T_{\variance,1}^\lora &\lesssim C_{\textnormal{g}} \kappa_*^2(M) \|\Sigma_\epsilon^{(\id)}\|_\op^{1/2} \kappa_*(A) \sqrt{\frac{r(r_e(\Phi'^\top \Sigma_\epsilon^{(\id)} \Phi') + r_e(A^2)) \log^2(n+p+q)}{n}}\\
        &\quad+ C_{\textnormal{g}} \kappa_*^2(M) \|D \Sigma_x^{(\id)} D^\top\|_\op^{1/2} \sqrt{\frac{r(\kappa_*^2(A) r_e(\Phi'^\top D \Sigma_x^{(\id)} D^\top \Phi') + \kappa_*^4(A) r_e(A^2)) \log^2(n+p+q)}{n}}.
    \end{align*}
    Next we bound $T_{\variance,2}^\lora$. Again from Lemma~\ref{lem: good event},
    \begin{align*}
        T_{\variance,2}^\lora &\leq \sqrt{r} \|\hat D \hat \Sigma_x^{(\id)} \uW^{\pre \top}_{\ell-1}\|_\op \|\hat A^\dag\|_\op \|\hat A^\dag - A^\dag\|_\op \|A\|_\op\\
        &\lesssim \|D \Sigma_x^{(\id) 1/2}\|_\op \|\Sigma_x^{(\id) 1/2} \uW^{\pre \top}_{\ell-1}\|_\op \frac{\kappa_*^2(A)}{\lambda_*(A)} \sqrt{\frac{r \cdot r_e(A^2) \log^2(n+p+q)}{n}}\\
        &= \|D \Sigma_x^{(\id) 1/2}\|_\op \kappa_*^3(A) \sqrt{\frac{r \cdot r_e(A^2) \log^2(n+p+q)}{n}}
    \end{align*}
    holds on the event $\mathcal{F}$. Therefore,
    \begin{align}
        T_\variance^\lora &\lesssim C_{\textnormal{g}} \kappa_*^2(M) \|\Sigma_\epsilon^{(\id)}\|_\op^{1/2} \kappa_*(A) \sqrt{\frac{r(r_e(\Phi'^\top \Sigma_\epsilon^{(\id)} \Phi') + r_e(A^2)) \log^2(n+p+q)}{n}}\nonumber\\
        &\quad+ C_{\textnormal{g}} \kappa_*^2(M) \|D \Sigma_x^{(\id)} D^\top\|_\op^{1/2} \sqrt{\frac{r(\kappa_*^2(A) r_e(\Phi'^\top D \Sigma_x^{(\id)} D^\top \Phi') + \kappa_*^6(A) r_e(A^2)) \log^2(n+p+q)}{n}}\label{eq: T lora}
    \end{align}
    hold with high probability.
    \paragraph{Bound $T_\bias^\lora$.}
    Note that
    \begin{align*}
        (T_\bias^\lora)^2 &= \|\SVD_r(M) A^\dag \uW^\pre_{\ell-1} \Sigma_x^{(\id) 1/2} - D \Sigma_x^{(\id) 1/2}\|_\F^2\\
        &= \|\underbrace{\SVD_r(M) A^\dag \uW^\pre_{\ell-1} \Sigma_x^{(\id) 1/2} - \Phi' \Phi'^\top D \Sigma_x^{(\id)} \uW^{\pre \top}_{\ell-1} (A^2)^\dag \uW^\pre_{\ell-1} \Sigma_x^{(\id) 1/2}}_{=: T_1}\|_\F^2\\
        &\quad+ \|\underbrace{D \Sigma_x^{(\id) 1/2} (I - \Sigma_x^{(\id) 1/2} \uW^{\pre \top}_{\ell-1} (A^2)^\dag \uW^\pre_{\ell-1} \Sigma_x^{(\id) 1/2})}_{=: T_2}\|_\F^2\\
        &\quad+ \|\underbrace{(I - \Phi' \Phi'^\top) D \Sigma_x^{(\id)} \uW^{\pre \top}_{\ell-1} (A^2)^\dag \uW^\pre_{\ell-1} \Sigma_x^{(\id) 1/2}}_{=: T_3}\|_\F^2
    \end{align*}
    where the second equality follows from the fact that cross terms are zero, i.e., $\tr(T_1 T_2^\top) = \tr(T_2 T_3^\top) = \tr(T_3 T_1^\top) = 0$ since $\Psi_*(\uW^\pre_{\ell-1} \Sigma_x^{(\id) 1/2}) \Psi_*^\top(\uW^\pre_{\ell-1} \Sigma_x^{(\id) 1/2}) = \Sigma_x^{(\id) 1/2} \uW^{\pre \top}_{\ell-1} (A^2)^\dag \uW^\pre_{\ell-1} \Sigma_x^{(\id) 1/2}$ and
    \begin{align*}
        (I - \Phi' \Phi'^\top) \Phi_*(\SVD_r(M)) = 0,\ \ 
        \uW^\pre_{\ell-1} \Sigma_x^{(\id) 1/2} (I - \Psi_*(\uW^\pre_{\ell-1} \Sigma_x^{(\id) 1/2}) \Psi_*^\top(\uW^\pre_{\ell-1} \Sigma_x^{(\id) 1/2})) = 0
    \end{align*}
    hold.
    Thus from Lemma~\ref{lem: excess risk full},
    \begin{align}
        (T_\bias^\lora)^2 &= \|\SVD_r(\Phi' \Phi'^\top D \Sigma_x^{(\id)} \uW^{\pre \top}_{\ell-1} A^\dag) - \Phi' \Phi'^\top D \Sigma_x^{(\id)} \uW^{\pre \top}_{\ell-1} A^\dag\|_\F^2 + \mathcal{E}^{(\id)}(f_\ell^\full).\label{eq: B square lora}
    \end{align}
    Notice that
    \begin{align}
        &\|\SVD_r(\Phi' \Phi'^\top D \Sigma_x^{(\id)} \uW^{\pre \top}_{\ell-1} A^\dag) - \Phi' \Phi'^\top D \Sigma_x^{(\id)} \uW^{\pre \top}_{\ell-1} A^\dag\|_\F^2\nonumber\\
        &\quad\leq \{0 \vee (\rank(\Phi' \Phi'^\top D \Sigma_x^{(\id)} \uW^{\pre \top}_{\ell-1} A^\dag) - r)\} \|\Phi' \Phi'^\top D \Sigma_x^{(\id)} \uW^{\pre \top}_{\ell-1} A^\dag\|_\op^2\nonumber\\
        &\quad\leq \{0 \vee (\rank(\Phi' \Phi'^\top D \Sigma_x^{(\id)} \uW^{\pre \top}_{\ell-1} A^\dag) - r)\} \|D \Sigma_x^{(\id) 1/2}\|_\op^2\nonumber\\
        &\quad\leq \frac{0 \vee (\rank(D \Sigma_x^{(\id)} D^\top) - r)}{\rank(D \Sigma_x^{(\id) 1/2})} \kappa_*^2(D \Sigma_x^{(\id)} D^\top) \mathcal{E}^{(\id)}(f^\pre),\label{eq: M - svd M}
    \end{align}
    where the last inequality follows since
    \begin{align*}
        \|D \Sigma_x^{(\id) 1/2}\|_\F^2 = \|\Lambda_*(D \Sigma_x^{(\id) 1/2})\|_\F^2 \geq \rank(D \Sigma_x^{(\id) 1/2}) \lambda_*^2(D \Sigma_x^{(\id) 1/2}) = \frac{\rank(D \Sigma_x^{(\id) 1/2}) }{\kappa_*^2(D \Sigma_x^{(\id) 1/2})} \|D \Sigma_x^{(\id) 1/2}\|_\op^2.
    \end{align*}
    
    \paragraph{Summary}
    Note that for any $\eta > 0$, $(T_\variance^\lora + T_\bias^\lora)^2 \leq (1 + \eta) (T_\bias^\lora)^2 + (1 + 1/\eta) (T_\variance^\lora)^2$ holds. Therefore,
    \begin{align*}
        \mathcal{E}^{(\id)}(f_{\ell,U^\lora,V^\lora}) &\leq (1 + \eta) (T_\bias^\lora)^2 + (1 + \eta^{-1}) (T_\variance^\lora)^2.
    \end{align*}
    Combined with \eqref{eq: T lora}, \eqref{eq: B square lora}, and \eqref{eq: M - svd M}, this concludes the proof.
\end{proof}

\subsubsection{Out-of-distribution Excess Risk of LoRA}\label{sec: lora ood}

We define the low-rank matrix obtained by LoRA under population in-distribution risk as
\begin{align}
    (U_\infty^\lora, V_\infty^\lora) &\in \argmin_{U, V} \|(U, V)\|_\F^2\ \ \text{ s.t. $(U, V)$ minimizes $\mathcal{R}^{(\id)}(f_{\ell,U,V})$}.\label{eq: U V infty lora}
\end{align}

\begin{theorem}[Restatement of Theorem~\ref{thm: sft and lora ood}: LoRA Part]\label{thm: lora ood}
    For $(U_\infty^\lora, V_\infty^\lora)$, defined in \eqref{eq: U V infty lora}
    \begin{align*}
        \mathcal{E}^{(\ood)}(f_{\ell,U_\infty^\lora,V_\infty^\lora}) &\lesssim \|(I - \Phi' \Phi'^\top) B^{(\ood)} \Sigma_x^{(\ood) 1/2}\|_\F^2 + \|(B^{(\ood)} - B^{(\id)}) \Sigma_x^{(\id) 1/2}\|_\F^2 \|G^{(\id,\ood)}_{\ell-1}\|_\op^2\\
        &\quad+ \|(B^{(\ood)} - W^\pre) (\Sigma_x^{(\ood) 1/2} - \Sigma_x^{(\id) 1/2} G^{(\id,\ood)}_{\ell-1})\|_\F\\
        &\quad+ \frac{0 \vee (\rank(D \Sigma_x^{(\id)} D^\top) - r)}{\rank(D \Sigma_x^{(\id)} D^\top)} \kappa_*^2(D \Sigma_x^{(\id)} D^\top) \|G^{(\id,\ood)}_{\ell-1}\|_\op^2 \mathcal{E}^{(\id)}(f^\pre).
    \end{align*}
    Furthermore, for any $\eta \in (0, 1)$,
    \begin{align}
        \mathcal{E}^{(\ood)}(f_{\ell,U_\infty^\lora,V_\infty^\lora}) &\geq (1 - \eta) \norm{(B^{(\ood)} - B^{(\id)}) \Sigma_x^{(\ood) 1/2}}_\F^2 - 3 (\eta^{-1} - 1) \|(I - \Phi' \Phi'^\top) B^{(\id)} \Sigma_x^{(\ood) 1/2}\|_\F^2\nonumber\\
        &\quad- 3 (\eta^{-1} - 1)\|(B^{(\id)} - W^\pre) (\Sigma_x^{(\ood) 1/2} - \Sigma_x^{(\id) 1/2} G_{\ell-1}^{(\id,\ood)})\|_\F^2\nonumber\\
        &\quad- 3 (\eta^{-1} - 1) \frac{0 \vee (\rank(D \Sigma_x^{(\id)} D^\top) - r)}{\rank(D \Sigma_x^{(\id)} D)} \kappa_*^2(D \Sigma_x^{(\id)} D^\top) \|G^{(\id,\ood)}_{\ell-1}\|_\op \mathcal{E}^{(\id)}(f^\pre).\label{eq: lora ood lb}
    \end{align}
\end{theorem}

\begin{proof}[Proof of Theorem~\ref{thm: lora ood}]
    With a slight modification to the proof of Lemma~\ref{lem: excess risk lora}, it follows that
    \begin{align}
        \mathcal{E}^{(\ood)}(f_{\ell,U_\infty^\lora,V_\infty^\lora}) &= \tr\biggl(\qty(B^{(\ood)} - W^\pre - \SVD_r(\oW^\pre_{\ell+1} (\oW^\pre_{\ell+1})^\dag D \Sigma_x^{(\id)} \uW^{\pre \top}_{\ell-1} A^\dag) A^\dag \uW^\pre_{\ell-1}) \Sigma_x^{(\ood)}\nonumber\\
        &\quad\cdot \qty(B^{(\ood)} - W^\pre - \SVD_r(\oW^\pre_{\ell+1} (\oW^\pre_{\ell+1})^\dag D \Sigma_x^{(\id)} \uW^{\pre \top}_{\ell-1} A^\dag) A^\dag \uW^\pre_{\ell-1})^\top\biggr)\nonumber\\
        &= \norm{(B^{(\ood)} - W^\pre) \Sigma_x^{(\ood) 1/2} - \SVD_r(\Phi' \Phi'^\top D \Sigma_x^{(\id)} \uW^{\pre \top}_{\ell-1} A^\dag) A^\dag \uW^\pre_{\ell-1} \Sigma_x^{(\ood) 1/2}}_\F^2.\label{eq: excess risk lora infty}
    \end{align}
    Recall that $M := \Phi' \Phi'^\top D \Sigma_x^{(\id)} \uW^{\pre \top}_{\ell-1} A^\dag$. Then,
    \begin{align*}
        &\norm{(B^{(\ood)} - W^\pre) \Sigma_x^{(\ood) 1/2} - \SVD_r(\Phi' \Phi'^\top D \Sigma_x^{(\id)} \uW^{\pre \top}_{\ell-1} A^\dag) A^\dag \uW^\pre_{\ell-1} \Sigma_x^{(\ood) 1/2}}_\F\\
        &\quad\leq \norm{(B^{(\ood)} - W^\pre) \Sigma_x^{(\ood) 1/2} - \Phi' \Phi'^\top D \Sigma_x^{(\id)} \uW^{\pre \top}_{\ell-1} (A^2)^\dag \uW^\pre_{\ell-1} \Sigma_x^{(\ood) 1/2}}_\F\\
        &\quad\quad+ \|M A^\dag \uW^\pre_{\ell-1} \Sigma_x^{(\ood) 1/2} - \SVD_r(M) A^\dag \uW^\pre_{\ell-1} \Sigma_x^{(\ood) 1/2}\|_\F\\
        &\quad= \norm{(B^{(\ood)} - W^\pre) \Sigma_x^{(\ood) 1/2} - \Phi' \Phi'^\top D \Sigma_x^{(\id) 1/2} (\uW^\pre_{\ell-1} \Sigma_x^{(\id) 1/2})^\dag \uW^\pre_{\ell-1} \Sigma_x^{(\ood) 1/2}}_\F\\
        &\quad\quad+ \|M A^\dag \uW^\pre_{\ell-1} \Sigma_x^{(\ood) 1/2} - \SVD_r(M) A^\dag \uW^\pre_{\ell-1} \Sigma_x^{(\ood) 1/2}\|_\F\\
        &\quad\leq \|(I - \Phi' \Phi'^\top) B^{(\ood)} \Sigma_x^{(\ood) 1/2}\|_\F + \|\Phi' \Phi'^\top (B^{(\ood)} - B^{(\id)}) \Sigma_x^{(\id) 1/2} G^{(\id,\ood)}_{\ell-1}\|_\F\\
        &\quad\quad+ \|\Phi' \Phi'^\top (B^{(\ood)} - W^\pre) (\Sigma_x^{(\ood) 1/2} - \Sigma_x^{(\id) 1/2} G^{(\id,\ood)}_{\ell-1})\|_\F\\
        &\quad\quad+ \|M A^\dag \uW^\pre_{\ell-1} \Sigma_x^{(\ood) 1/2} - \SVD_r(M) A^\dag \uW^\pre_{\ell-1} \Sigma_x^{(\ood) 1/2}\|_\F\\
        &\quad\leq \|(I - \Phi' \Phi'^\top) B^{(\ood)} \Sigma_x^{(\ood) 1/2}\|_\F + \|(B^{(\ood)} - B^{(\id)}) \Sigma_x^{(\id) 1/2}\|_\F \|G^{(\id,\ood)}_{\ell-1}\|_\op\\
        &\quad\quad+ \|(B^{(\ood)} - W^\pre) (\Sigma_x^{(\ood) 1/2} - \Sigma_x^{(\id) 1/2} G^{(\id,\ood)}_{\ell-1})\|_\F + \|M - \SVD_r(M)\|_\F \|A^\dag \uW^\pre_{\ell-1} \Sigma_x^{(\ood) 1/2}\|_\op,
    \end{align*}
    where we used $\Phi' \Phi'^\top W^\pre = W^\pre$.
    From \eqref{eq: M - svd M}, we have
    \begin{align*}
        \{\mathcal{E}^{(\ood)}(f_{\ell,U_\infty^\lora,V_\infty^\lora})\}^{1/2} &\leq \|(I - \Phi' \Phi'^\top) B^{(\ood)} \Sigma_x^{(\ood) 1/2}\|_\F + \|(B^{(\ood)} - B^{(\id)}) \Sigma_x^{(\id) 1/2}\|_\F \|G^{(\id,\ood)}_{\ell-1}\|_\op\\
        &\quad+ \|(B^{(\ood)} - W^\pre) (\Sigma_x^{(\ood) 1/2} - \Sigma_x^{(\id) 1/2} G^{(\id,\ood)}_{\ell-1})\|_\F\\
        &\quad+ \|G^{(\id,\ood)}_{\ell-1}\|_\op \kappa_*(D \Sigma_x^{(\id)} D^\top) \sqrt{\frac{0 \vee (\rank(D \Sigma_x^{(\id)} D^\top) - r)}{\rank(D \Sigma_x^{(\id) 1/2})} \mathcal{E}^{(\id)}(f^\pre)},
    \end{align*}
    where we used $\|A^\dag \uW^\pre_{\ell-1} \Sigma_x^{(\ood) 1/2}\|_\op = \|G^{(\id,\ood)}_{\ell-1}\|_\op$.
    This gives the first claim.

    Using $2 \tr(A B^\top) \geq - \eta \|A\|_\F^2 - (1/\eta) \|B\|_\F^2$ for any $\eta > 0$ and any matrices $A, B$ of the same shape, \eqref{eq: excess risk lora infty} can be rewritten as
    \begin{align}
        \mathcal{E}^{(\ood)}(f_{\ell,U_\infty^\lora,V_\infty^\lora}) &= \Bigl\|(B^{(\ood)} - B^{(\id)}) \Sigma_x^{(\ood) 1/2} + \underbrace{(I - \Phi' \Phi'^\top) (B^{(\id)} - W^\pre) \Sigma_x^{(\ood) 1/2}}_{=: T_1}\nonumber\\
        &\quad+ \underbrace{\Phi' \Phi'^\top (B^{(\id)} - W^\pre) (\Sigma_x^{(\ood) 1/2} - \Sigma_x^{(\id) 1/2} G_{\ell-1}^{(\id,\ood)})}_{=: T_2}\nonumber\\
        &\quad+ \underbrace{M A^\dag \uW^\pre_{\ell-1} \Sigma_x^{(\ood) 1/2} - \SVD_r(M) A^\dag \uW^\pre_{\ell-1} \Sigma_x^{(\ood) 1/2}}_{=: T_3}\Bigr\|_\F^2\nonumber\\
        &= \norm{(B^{(\ood)} - B^{(\id)}) \Sigma_x^{(\ood) 1/2}}_\F^2 + 2\tr((B^{(\ood)} - B^{(\id)}) \Sigma_x^{(\ood) 1/2} (T_1 + T_2 + T_3)^\top)\nonumber\\
        &\quad+ \|T_1 + T_2 + T_3\|_\F^2\nonumber\\
        &\geq (1 - \eta) \norm{(B^{(\ood)} - B^{(\id)}) \Sigma_x^{(\ood) 1/2}}_\F^2 + (1 - \eta^{-1}) \|T_1 + T_2 + T_3\|_\F^2.\label{eq: lora lb}
    \end{align}
    Choose $\eta \in (0, 1)$.
    By a similar argument as above, and using $\Phi' \Phi'^\top W^\pre = W^\pre$, we can show that
    \begin{align*}
        \|T_1+T_2+T_3\|_\F^2 &\leq 3 \|T_1\|_\F^2 + 3 \|T_2\|_\F^2 + 3 \|T_3\|_\F^2\\
        &\leq 3 \|(I - \Phi' \Phi'^\top) B^{(\id)} \Sigma_x^{(\ood) 1/2}\|_\F^2 + 3 \|(B^{(\id)} - W^\pre) (\Sigma_x^{(\ood) 1/2} - \Sigma_x^{(\id) 1/2} G_{\ell-1}^{(\id,\ood)})\|_\F^2\\
        &\quad+ 3 \frac{0 \vee (\rank(D \Sigma_x^{(\id)} D^\top) - r)}{\rank(D \Sigma_x^{(\id) 1/2})} \kappa_*^2(D \Sigma_x^{(\id)} D^\top) \|G^{(\id,\ood)}_{\ell-1}\|_\op \mathcal{E}^{(\id)}(f^\pre),
    \end{align*}
    where we used \eqref{eq: M - svd M} again.
    This concludes the proof.
\end{proof}

\subsection{Proofs for Structured Sparse Fine-tuning}

\subsubsection{Excess Risk of Structured Sparse Fine-tuning}

\begin{lemma}[Excess Risk]\label{lem: excess risk sft}
    Given $S \subset [d_\ell]$, consider the minimum norm solution
    \begin{align*}
        V^\sft \in \argmin_{V \in \R^{d_{\ell-1} \times s}} \|V\|_\F^2\ \ \text{ s.t. $V$ minimizes $\mathcal{R}_n^{(\id)}(f_{\ell,U_S^\sft,V})$}.
    \end{align*}
    Then, the structured sparse adaptation matrix satisfies
    \begin{align}
        U_S^\sft V^{\sft \top} = U_S^\sft (\oW^\pre_{\ell+1} U_S^\sft)^\dag \hat D \hat \Sigma_x^{(\id)} \uW^{\pre \top}_{\ell-1} (\hat A^\dag)^2,\label{eq: V star}
    \end{align}
    and
    \begin{align*}
        \mathcal{E}^{(k)}(f_{\ell,U_S^\sft,V^\sft}) &= \tr\biggl(\qty(B^{(k)} - W^\pre - \oW^\pre_{\ell+1} U_S^\sft (\oW^\pre_{\ell+1} U_S^\sft)^\dag \hat D \hat \Sigma_x^{(\id)} \uW^{\pre \top}_{\ell-1} (\hat A^\dag)^2 \uW^\pre_{\ell-1} ) \Sigma_x^{(k)}\\
        &\quad\cdot \qty(B^{(k)} - W^\pre - \oW^\pre_{\ell+1} U_S^\sft (\oW^\pre_{\ell+1} U_S^\sft)^\dag \hat D \hat \Sigma_x^{(\id)} \uW^{\pre \top}_{\ell-1} (\hat A^\dag)^2 \uW^\pre_{\ell-1} )^\top\biggr)
    \end{align*}
    for $k \in \{\id, \ood\}$.
\end{lemma}

\begin{proof}   
    Since $\hat \Sigma_{x,\epsilon}^{(k)} = (1/n) X^{(k)} E^{(k) \top}$ and $\hat \Sigma_x^{(k)} = (1/n) X^{(k)} X^{(k) \top}$, we have $\hat \Sigma_{x,\epsilon}^{(k)} = \hat \Sigma_x^{(k)} (X^{(k) \top})^\dag E^{(k) \top} =: \hat \Sigma_x^{(k)} \check \Sigma_{x,\epsilon}^{(k)}$. Similar to \eqref{eq: risk lora equivalent}, we have
    \begin{align*}
        \mathcal{R}_n^{(\id)}(f_{\ell,U_S^\sft,V}) &= \|\oW^\pre_{\ell+1} U_S^\sft V^\top \hat A - \hat D \hat \Sigma_x^{(\id)} \uW^{\pre \top}_{\ell-1} \hat A^\dag\|_\F^2 - \|\hat D \hat \Sigma_x^{(\id)} \uW^{\pre \top}_{\ell-1} \hat A^\dag\|_\F^2\\
        &\quad+ \tr(D \hat \Sigma_x^{(\id)} D^\top) + 2 \tr( D \hat \Sigma_{x,\epsilon}^{(\id)}) + \tr(\hat \Sigma_\epsilon^{(\id)}).
    \end{align*}
    Thus minimizing $\mathcal{R}_n^{(\id)}(f_{\ell,U_S^\sft,V})$ is equivalent to minimizing the norm
    \begin{align}
        &\|\oW^\pre_{\ell+1} U_S^\sft V^\top \hat A - \hat D \hat \Sigma_x^{(\id)} \uW^{\pre \top}_{\ell-1} \hat A^\dag\|_\F^2\label{eq: risk quadratic sft}\\
        &\quad= \|\oW^\pre_{\ell+1} U_S^\sft V^\top \hat A - \oW^\pre_{\ell+1} U_S^\sft (\oW^\pre_{\ell+1} U_S^\sft)^\dag \hat D \hat \Sigma_x^{(\id)} \uW^{\pre \top}_{\ell-1} \hat A^\dag\|_\F^2\nonumber\\
        &\quad\quad + \|(I - (\oW^\pre_{\ell+1} U_S^\sft) (\oW^\pre_{\ell+1} U_S^\sft)^\dag) \hat D \hat \Sigma_x^{(\id)} \uW^{\pre \top}_{\ell-1} \hat A^\dag\|_\F^2.\nonumber
    \end{align}
    Using the same argument as in the proof of Lemma~\ref{lem: excess risk lora},
    the minimum norm solution $V^\sft$ is obtained by
    \begin{align*}
        V^\sft = (\hat A^\dag)^2 \uW^\pre_{\ell-1} \hat \Sigma_x^{(\id)} \hat D^\top (U_S^{\sft \top} \oW^{\pre \top}_{\ell+1})^\dag.
    \end{align*}
    The excess risk for $k \in \{\id, \ood\}$ becomes
    \begin{align*}
        \mathcal{E}^{(k)}(f_{\ell,U_S^\sft,V^\sft}) &= \E\qty[\qty(B^{(k)} x^{(k)} - \oW^\pre_{\ell+1} (W^\pre_{\ell} + U_S^\sft V^{\sft \top}) \uW^\pre_{\ell-1} x^{(k)})^2]\\
        &= \tr\biggl(\qty(B^{(k)} - W^\pre - \oW^\pre_{\ell+1} U_S^\sft (\oW^\pre_{\ell+1} U_S^\sft)^\dag \hat D \hat \Sigma_x^{(\id)} \uW^{\pre \top}_{\ell-1} (\hat A^\dag)^2 \uW^\pre_{\ell-1} ) \Sigma_x^{(k)}\\
        &\quad\cdot \qty(B^{(k)} - W^\pre - \oW^\pre_{\ell+1} U_S^\sft (\oW^\pre_{\ell+1} U_S^\sft)^\dag \hat D \hat \Sigma_x^{(\id)} \uW^{\pre \top}_{\ell-1} (\hat A^\dag)^2 \uW^\pre_{\ell-1} )^\top\biggr).
    \end{align*}
    This concludes the proof.
\end{proof}

\subsubsection{In-distribution Excess Risk of Structured Sparse Fine-tuning}

\begin{theorem}[Restatement of Theorem~\ref{thm: sft and lora id}: \model Part]\label{thm: sft in-distribution}
    Suppose that Assumptions~\ref{asm: sub gaussian x e ap} and \ref{asm: regime ap} hold. Fix $S \subset [d_\ell]$ with $|S| = s$.
    Then, the following holds with probability $1 - \exp(-\Omega(\log^2(n+p+q)))$.
    For any $\eta > 0$,
    \begin{align*}
        \mathcal{E}^{(\id)}(f_{\ell,U_S^\sft,V^\sft}) \leq (1 + \eta) (T_\bias^\sft)^2 + (1 + \eta^{-1}) (T_\variance^\sft)^2,
    \end{align*}
    where
    \begin{align}
        (T_\bias^\sft)^2 &\leq \|(\Phi' \Phi'^\top - \Phi_S'' \Phi_S''^\top) \Phi_*(D \Sigma_x^{(\id) 1/2})\|_\op^2 \mathcal{E}^{(\id)}(f_\ell^\pre) + \mathcal{E}^{(\id)}(f_\ell^\full),\label{eq: bias sft}\\
        (T_\variance^\sft)^2 &\lesssim \|\Sigma_\epsilon^{(\id)}\|_\op \kappa_*^2(A) \frac{s (r_e(\Phi_S''^\top \Sigma_\epsilon^{(\id)} \Phi_S'') + r_e(A^2)) \log^2(n+p+q)}{n}\nonumber\\
        &\quad+ \|D \Sigma_x^{(\id)} D^\top\|_\op \frac{s (\kappa_*^2(A) r_e(\Phi_S''^\top D \Sigma_x^{(\id)} D^\top \Phi_S'') + \kappa_*^8(A) r_e(A^2)) \log^2(n+p+q)}{n}.\nonumber
    \end{align}
\end{theorem}
Note that the term $\|(\Phi' \Phi'^\top - \Phi_S'' \Phi_S''^\top) \Phi_*(D \Sigma_x^{(\id) 1/2})\|_\op$ in \eqref{eq: bias sft} measures the distance between subspaces spanned by $\Phi'$ and $\Phi_S''$ in a label space, weighted by $\Phi_*(\Sigma_f^{(\id)})$.
In high level, this quantity shows the closeness between the $\ell$-th layer full fine-tuning and \model.
It takes small values when the important channels for residual prediction are sparsely distributed among all channels.
This aligns with the intuition that \model only selectively fine-tunes small number of coordinates, and thus relying on the information contained in those coordinates.

\begin{proof}[Proof of Theorem~\ref{thm: sft in-distribution}]
    Using the same argument as in the proof of Theorem~\ref{thm: lora in-distribution} combined with Lemma~\ref{lem: excess risk sft}, we have
    \begin{align*}
        \mathcal{E}^{(\id)}(f_{\ell,U_S^\sft,V^\sft}) &= \|(\oW^\pre_{\ell+1} U_S^\sft V^{\sft \top} A A^\dag \uW^\pre_{\ell-1} - D) \Sigma_x^{(\id) 1/2}\|_\F^2,
    \end{align*}
    and
    \begin{align*}
        &\|(\oW^\pre_{\ell+1} U_S^\sft V^{\sft \top} A A^\dag \uW^\pre_{\ell-1} - D) \Sigma_x^{(\id) 1/2}\|_\F\\
        &\quad\leq \|\oW^\pre_{\ell+1} U_S^\sft (\oW^\pre_{\ell+1} U_S^\sft)^\dag (\hat D \hat \Sigma_x^{(\id)} \uW^{\pre \top}_{\ell-1} (\hat A^2)^\dag - D \Sigma_x^{(\id)} \uW^{\pre \top}_{\ell-1} (A^2)^\dag) A\|_\F\\
        &\quad\quad+ \|\oW^\pre_{\ell+1} U_S^\sft (\oW^\pre_{\ell+1} U_S^\sft)^\dag D \Sigma_x^{(\id)} \uW^{\pre \top}_{\ell-1} (A^2)^\dag \uW^\pre_{\ell-1} \Sigma_x^{(\id) 1/2} - D \Sigma_x^{(\id) 1/2}\|_\F\\
        &\quad=: T_\variance^\sft + T_\bias^\sft.
    \end{align*}
    We bound $T_\variance^\sft$ and $T_\bias^\sft$ separately.
    \paragraph{Bound $T_\variance^\sft$.}
    Note that
    \begin{align*}
        T_\variance^\sft &= \|\oW^\pre_{\ell+1} U_S^\sft (\oW^\pre_{\ell+1} U_S^\sft)^\dag \hat D \hat \Sigma_x^{(\id)} \uW^{\pre \top}_{\ell-1} (\hat A^\dag)^2 A - \oW^\pre_{\ell+1} U_S^\sft (\oW^\pre_{\ell+1} U_S^\sft)^\dag D \Sigma_x^{(\id)} \uW^{\pre \top}_{\ell-1} A^\dag\|_\F\\
        &\leq \|\oW^\pre_{\ell+1} U_S^\sft (\oW^\pre_{\ell+1} U_S^\sft)^\dag D \Sigma_x^{(\id)} \uW^{\pre \top}_{\ell-1} A^\dag - \oW^\pre_{\ell+1} U_S^\sft (\oW^\pre_{\ell+1} U_S^\sft)^\dag \hat D \hat \Sigma_x^{(\id)} \uW^{\pre \top}_{\ell-1} A^\dag\|_\F\\
        &\quad+ \|\oW^\pre_{\ell+1} U_S^\sft (\oW^\pre_{\ell+1} U_S^\sft)^\dag \hat D \hat \Sigma_x^{(\id)} \uW^{\pre \top}_{\ell-1} ((\hat A^\dag)^2 - (A^\dag)^2) A\|_\F\\
        &=: T_{\variance,1}^\sft + T_{\variance,2}^\sft.
    \end{align*}
    For the term $T_{\variance,1}^\sft$, using Lemma~\ref{lem: good event},
    \begin{align*}
        T_{\variance,1}^\sft &\leq 2 \sqrt{s} \|\Phi_S''^\top D \Sigma_x^{(\id)} \uW^{\pre \top}_{\ell-1} - \Phi_S''^\top \hat D \hat \Sigma_x^{(\id)} \uW^{\pre \top}_{\ell-1}\|_\op \|A^\dag\|_\op\\
        &\lesssim \|\Sigma_\epsilon^{(\id)}\|_\op^{1/2} \kappa_*(A) \sqrt{\frac{s (r_e(\Phi_S''^\top \Sigma_\epsilon^{(\id)} \Phi_S'') + r_e(A^2)) \log^2(n+p+q)}{n}}\\
        &\quad\quad+ \|D \Sigma_x^{(\id)} D^\top\|_\op^{1/2} \kappa_*(A) \sqrt{\frac{s (r_e(\Phi_S''^\top D \Sigma_x^{(\id)} D^\top \Phi_S'') + r_e(A^2)) \log^2(n+p+q)}{n}}
    \end{align*}
    holds on the event $\mathcal{F}$, where the first inequality follows since the term inside the norm is at most rank-$2s$.
    Again from Lemma~\ref{lem: good event},
    \begin{align*}
        T_{\variance,2}^\sft &\leq \sqrt{s} \|\Phi_S''^\top \hat D \hat \Sigma_x^{(\id)} \uW^{\pre \top}_{\ell-1}\|_\op \|(\hat A^\dag)^2 - (A^\dag)^2\|_\op \|A\|_\op\\
        &\lesssim \|D \Sigma_x^{(\id) 1/2}\|_\op \|\Sigma_x^{(\id) 1/2} \uW^{\pre \top}_{\ell-1}\|_\op \frac{\kappa_*^3(A)}{\lambda_*(A)} \sqrt{\frac{s r_e(A^2) \log^2(n+d+p)}{n}}\\
        &= \|D \Sigma_x^{(\id) 1/2}\|_\op \kappa_*^4(A) \sqrt{\frac{s r_e(A^2) \log^2(n+d+p)}{n}}
    \end{align*}
    holds on the event $\mathcal{F}$.
    Therefore,
    \begin{align}
        T_\variance^\sft &\lesssim \|\Sigma_\epsilon^{(\id)}\|_\op^{1/2} \kappa_*(A) \sqrt{\frac{s (r_e(\Phi_S''^\top \Sigma_\epsilon^{(\id)} \Phi_S'') + r_e(A^2)) \log^2(n+p+q)}{n}}\nonumber\\
        &\quad+ \|D \Sigma_x^{(\id)} D^\top\|_\op^{1/2} \sqrt{\frac{s (\kappa_*^2(A) r_e(\Phi_S''^\top D \Sigma_x^{(\id)} D^\top \Phi_S'') + \kappa_*^8(A) r_e(A^2)) \log^2(n+p+q)}{n}}.\label{eq: T sft}
    \end{align}
    \paragraph{Bound $T_\bias^\sft$.}
    By the same argument as in the proof of Theorem~\ref{thm: lora in-distribution},
    \begin{align}
        (T_\bias^\sft)^2 &= \|\Phi_S'' \Phi_S''^\top D \Sigma_x^{(\id)} \uW^{\pre \top}_{\ell-1} A^\dag - \Phi' \Phi'^\top D \Sigma_x^{(\id)} \uW^{\pre \top}_{\ell-1} A^\dag\|_\F^2 + \mathcal{E}^{(\id)}(f_\ell^\full)\label{eq: B square sft}\\
        &\leq \|(\Phi_S'' \Phi_S''^\top - \Phi' \Phi'^\top) \Phi_*(D \Sigma_x^{(\id) 1/2})\|_\op^2 \|D \Sigma_x^{(\id)} \uW^{\pre \top}_{\ell-1} A^\dag\|_\F^2 + \mathcal{E}^{(\id)}(f_\ell^\full)\nonumber\\
        &= \|(\Phi_S'' \Phi_S''^\top - \Phi' \Phi'^\top) \Phi_*(D \Sigma_x^{(\id) 1/2})\|_\op^2 \mathcal{E}^{(\id)}(f_\ell^\pre) + \mathcal{E}^{(\id)}(f_\ell^\full),\label{eq: B square sft 2}
    \end{align}
    where we used $\|D \Sigma_x^{(\id)} \uW^{\pre \top}_{\ell-1} A^\dag\|_\F^2 \leq \|D \Sigma_x^{(\id) 1/2}\|_\F^2 = \mathcal{E}^{(\id)}(f_\ell^\pre)$. We hypothesize that $T_\bias^\sft \simeq T_\bias^\lora$ by comparing \eqref{eq: B square lora} and \eqref{eq: B square sft}, Here, $\SVD_s(\Phi' \Phi'^\top D \Sigma_x^{(\id)} \uW^{\pre \top}_{\ell-1} A^\dag)$ is the best rank-$s$ approximation of $\Phi' \Phi'^\top D \Sigma_x^{(\id)} \uW^{\pre \top}_{\ell-1} A^\dag$ and $\Phi_S'' \Phi_S''^\top D \Sigma_x^{(\id)} \uW^{\pre \top}_{\ell-1} A^\dag$ benefits from a rank-$r$ approximation, where $r>s$.
    
    
    \paragraph{Summary}
    Note that for any $\eta > 0$, $(T_\variance^\sft + T_\bias^\sft)^2 \leq (1 + \eta) (T_\bias^\sft)^2 + (1 + 1/\eta) (T_\variance^\sft)^2$ holds. Thus
    \begin{align*}
        \mathcal{E}^{(\id)}(f_{\ell,U_S^\sft,V^\sft}) &\leq (1 + \eta) (T_\bias^\sft)^2 + (1 + \eta^{-1}) (T_\variance^\sft)^2.
    \end{align*}
    Combined with \eqref{eq: T sft} and \eqref{eq: B square sft 2}, this concludes the proof.
\end{proof}

Next we characterize the bias terms $T_\bias^\lora$ and $T_\bias^\sft$ under sparsity assumption.

\begin{lemma}\label{lem: sparsity}
    Suppose that Assumption~\ref{asm: sparsity ap} holds.
    Then, for a sparse fine-tuned network with the choice $S \supset S_0$, it follows that 
    \begin{align*}
        \mathcal{E}^{(\id)}(f_\ell^\full) \leq (T_\bias^\lora)^2 \leq (T_\bias^\sft)^2 \leq \mathcal{E}^{(\id)}(f_\ell^\full) + \delta^2 \kappa_*^2(\oW^\pre_{\ell+1}) \mathcal{E}^{(\id)}(f^\pre).
    \end{align*}
\end{lemma}

\begin{proof}
    Note that $\Phi_S'' \Phi_S''^\top$ is a projection into a subspace, which is contained in a subspace projected by $\Phi' \Phi'^\top$. Thus
    \begin{align*}
        &\|\Phi_S'' \Phi_S''^\top D \Sigma_x^{(\id)} \uW^{\pre \top}_{\ell-1} A^\dag - \Phi' \Phi'^\top D \Sigma_x^{(\id)} \uW^{\pre \top}_{\ell-1} A^\dag\|_\F^2\\
        &\quad= \|(\Phi_S'' \Phi_S''^\top - I) \Phi' \Phi'^\top D \Sigma_x^{(\id)} \uW^{\pre \top}_{\ell-1} A^\dag\|_\F^2\\
        &\quad= \|(\Phi_S'' \Phi_S''^\top - I) \oW^\pre_{\ell+1} (\oW^\pre_{\ell+1})^\dag D \Sigma_x^{(\id)} \uW^{\pre \top}_{\ell-1} A^\dag\|_\F^2\\
        &\quad= \|(\Phi_S'' \Phi_S''^\top - I) \oW^\pre_{\ell+1} ((I - U_S^\sft U_S^{\sft \top}) + U_S^\sft U_S^{\sft \top}) (\oW^\pre_{\ell+1})^\dag D \Sigma_x^{(\id)} \uW^{\pre \top}_{\ell-1} A^\dag\|_\F^2\\
        &\quad= \|(\Phi_S'' \Phi_S''^\top - I) \oW^\pre_{\ell+1} (I - U_S^\sft U_S^{\sft \top}) (\oW^\pre_{\ell+1})^\dag D \Sigma_x^{(\id)} \uW^{\pre \top}_{\ell-1} A^\dag\|_\F^2,
    \end{align*}
    where the last equality follows since $(\Phi_S'' \Phi_S''^\top - I) \oW^\pre_{\ell+1} U_S^\sft = 0$ by definition of $\Phi_S'' = \Phi_*(\oW^\pre_{\ell+1} U_S^\sft)$.
    Thus
    \begin{align*}
        &\|\Phi_S'' \Phi_S''^\top D \Sigma_x^{(\id)} \uW^{\pre \top}_{\ell-1} A^\dag - \Phi' \Phi'^\top D \Sigma_x^{(\id)} \uW^{\pre \top}_{\ell-1} A^\dag\|_\F^2\\
        &\quad\leq \|\oW^\pre_{\ell+1}\|_\op^2 \|(I - U_S^\sft U_S^{\sft \top})(\oW^\pre_{\ell+1})^\dag D \Sigma_x^{(\id) 1/2}\|_\F^2 \|\Sigma_x^{(\id) 1/2} \uW^{\pre \top}_{\ell-1} A^\dag\|_\op^2\\
        &\quad= \|\oW^\pre_{\ell+1}\|_\op^2 \|\Sigma_x^{(\id) 1/2} \uW^{\pre \top}_{\ell-1} A^\dag\|_\op^2 \sum_{a \in [d_\ell]\setminus S} \|e_a^\top (\oW^\pre_{\ell+1})^\dag D \Sigma_x^{(\id) 1/2}\|^2\\
        &\quad\leq \delta^2 \|\oW^\pre_{\ell+1}\|_\op^2 \|(\oW^\pre_{\ell+1})^\dag D \Sigma_x^{(\id) 1/2}\|_\F^2\\
        &\quad\leq \delta^2 \kappa_*^2(\oW^\pre_{\ell+1}) \|D \Sigma_x^{(\id) 1/2}\|_\F^2,
    \end{align*}
    where the second inequality follows from $\|\Sigma_x^{(\id) 1/2} \uW^{\pre \top}_{\ell-1} A^\dag\|_\op \leq 1$, Assumption~\ref{asm: sparsity ap} and $S \supset S_0$.
    The conclusion follows from \eqref{eq: B square lora} and \eqref{eq: B square sft}.
\end{proof}

\subsubsection{Out-of-distribution Excess Risk of Structured Sparse Fine-tuning}\label{sec: sft ood}

Given $S \subset [d_\ell]$ with $|S|=s$, we define the structured sparse adaptation matrix obtained by \model under population in-distribution risk as
\begin{align}
    V_\infty^\sft = \argmin_V \|V\|_\F^2 \ \ \text{ s.t. $V$ minimizes $\mathcal{R}^{(\id)}(f_{\ell,U_S^\sft,V})$}.\label{eq: V infty sft}
\end{align}

\begin{theorem}[Restatement of Theorem~\ref{thm: sft and lora ood}: \model Part]\label{thm: sft ood}
    Fix $S \subset [d_\ell]$ with $|S| = s$.
    For $V_\infty^\sft$ defined in \eqref{eq: V infty sft},
    \begin{align*}
        \mathcal{E}^{(\ood)}(f_{\ell,U_S^\sft,V_\infty^\sft}) &\leq \mathcal{E}^{(\ood)}(f^\pre) + 3 \bigl\|\Phi_S'' \Phi_S''^\top (B^{(\ood)} - B^{(\id)}) \Sigma_x^{(\ood) 1/2} \bigr\|_\F^2\\
        &\quad+ 3 \|B^{(\id)} (\Sigma_x^{(\ood) 1/2} - \Sigma_x^{(\id) 1/2} G_{\ell-1}^{(\id,\ood)}) \|_\F^2\\
        &\quad+ 3 \|\oW^\pre_\ell\|_\op^2 \|\uW^\pre_{\ell-1} \Sigma_x^{(\ood) 1/2} - \uW^\pre_{\ell-1} \Sigma_x^{(\id) 1/2} G_{\ell-1}^{(\id,\ood)}\|_\F^2.
    \end{align*}
\end{theorem}

\begin{remark}
    If there is no covariate shift, i.e., $\Sigma_x^{(\id)} = \Sigma_x^{(\ood)} = \Sigma_x$ for some $\Sigma_x$, Theorem~\ref{thm: sft ood} further gives the bound
    \begin{align*}
        \mathcal{E}^{(\ood)}(f_{\ell,U_S^\sft,V_\infty^\sft}) &\leq \mathcal{E}^{(\ood)}(f^\pre) + 3 \bigl\|\Phi_S'' \Phi_S''^\top (B^{(\ood)} - B^{(\id)}) \Sigma_x^{1/2} \bigr\|_\F^2\\
        &\quad+ 3 \|B^{(\id)} \Sigma_x^{1/2} (I -  (\uW^\pre_{\ell-1} \Sigma_x^{1/2})^\dag \uW_{\ell-1}^\pre \Sigma_x^{1/2})) \|_\F^2.
    \end{align*}
\end{remark}

\begin{proof}[Proof of Theorem~\ref{thm: sft ood}]
    With a slight modification to Lemma~\ref{lem: excess risk sft}, we obtain
    \begin{align*}
        \mathcal{E}^{(\ood)}(f_{\ell,U_S^\sft,V_\infty^\sft}) &= \tr\biggl(\qty(B^{(\ood)} - W^\pre - \oW^\pre_{\ell+1} U_S^\sft (\oW^\pre_{\ell+1} U_S^\sft)^\dag D \Sigma_x^{(\id)} \uW^{\pre \top}_{\ell-1} (A^\dag)^2 \uW^\pre_{\ell-1} ) \Sigma_x^{(\ood)}\\
        &\quad\cdot \qty(B^{(\ood)} - W^\pre - \oW^\pre_{\ell+1} U_S^\sft (\oW^\pre_{\ell+1} U_S^\sft)^\dag D \Sigma_x^{(\id)} \uW^{\pre \top}_{\ell-1} (A^\dag)^2 \uW^\pre_{\ell-1} )^\top\biggr)\\
        &= \norm{(B^{(\ood)} - W^\pre) \Sigma_x^{(\ood) 1/2} - \Phi_S'' \Phi_S''^\top D \Sigma_x^{(\id)} \uW^{\pre \top}_{\ell-1} (A^\dag)^2 \uW^\pre_{\ell-1} \Sigma_x^{(\ood) 1/2}}_\F^2\\
        &= \|(I - \Phi_S'' \Phi_S''^\top) (B^{(\ood)} - W^\pre) \Sigma_x^{(\ood) 1/2}\|_\F^2\\
        &\quad+ \bigl\|\underbrace{\Phi_S'' \Phi_S''^\top \qty{ (B^{(\ood)} - W^\pre) \Sigma_x^{(\ood) 1/2} - D \Sigma_x^{(\id) 1/2} (\uW^\pre_{\ell-1} \Sigma_x^{(\id) 1/2})^\dag \uW_{\ell-1}^\pre \Sigma_x^{(\ood) 1/2}}}_{=: T}\bigr\|_\F^2,
    \end{align*}
    where we used $\Sigma_x^{(\id) 1/2} \uW^{\pre \top}_{\ell-1} (A^\dag)^2 \uW^\pre_{\ell-1} \Sigma_x^{(\ood) 1/2} = (\uW^\pre_{\ell-1} \Sigma_x^{(\id) 1/2})^\dag \uW_{\ell-1}^\pre \Sigma_x^{(\ood) 1/2}$.
    Note that
    \begin{align*}
        \|T\|_\F &\leq \bigl\|\Phi_S'' \Phi_S''^\top \qty{ B^{(\ood)} \Sigma_x^{(\ood) 1/2} - B^{(\id)} \Sigma_x^{(\id) 1/2} (\uW^\pre_{\ell-1} \Sigma_x^{(\id) 1/2})^\dag \uW_{\ell-1}^\pre \Sigma_x^{(\ood) 1/2}}\bigr\|_\F\\
        &\quad+ \bigl\|\Phi_S'' \Phi_S''^\top \oW^\pre_\ell \qty{ \uW^\pre_{\ell-1} \Sigma_x^{(\ood) 1/2} - \uW^\pre_{\ell-1} \Sigma_x^{(\id) 1/2} (\uW^\pre_{\ell-1} \Sigma_x^{(\id) 1/2})^\dag \uW_{\ell-1}^\pre \Sigma_x^{(\ood) 1/2}}\bigr\|_\F\\
        &\leq \bigl\|\Phi_S'' \Phi_S''^\top (B^{(\ood)} - B^{(\id)}) \Sigma_x^{(\ood) 1/2} \bigr\|_\F\\
        &\quad+ \bigl\|\Phi_S'' \Phi_S''^\top B^{(\id)} ( \Sigma_x^{(\ood) 1/2} - \Sigma_x^{(\id) 1/2} G_{\ell-1}^{(\id,\ood)} )\bigr\|_\F\\
        &\quad+ \bigl\|\Phi_S'' \Phi_S''^\top \oW^\pre_\ell ( \uW^\pre_{\ell-1} \Sigma_x^{(\ood) 1/2} - \uW^\pre_{\ell-1} \Sigma_x^{(\id) 1/2} G_{\ell-1}^{(\id,\ood)} )\bigr\|_\F.
    \end{align*}
    Therefore,
    \begin{align*}
        \mathcal{E}^{(\ood)}(f_{\ell,U_S^\sft,V_\infty^\sft})
        &= \|(I - \Phi_S'' \Phi_S''^\top) (B^{(\ood)} - W^\pre) \Sigma_x^{(\ood) 1/2}\|_\F^2 + \|T\|_\F^2\\
        &\leq \mathcal{E}^{(\ood)}(f^\pre) + 3 \bigl\|\Phi_S'' \Phi_S''^\top (B^{(\ood)} - B^{(\id)}) \Sigma_x^{(\ood) 1/2} \bigr\|_\F^2\\
        &\quad+ 3 \|B^{(\id)} (\Sigma_x^{(\ood) 1/2} - \Sigma_x^{(\id) 1/2} G_{\ell-1}^{(\id,\ood)}) \|_\F^2\\
        &\quad+ 3 \|\oW^\pre_\ell\|_\op^2 \|\uW^\pre_{\ell-1} \Sigma_x^{(\ood) 1/2} - \uW^\pre_{\ell-1} \Sigma_x^{(\id) 1/2} G_{\ell-1}^{(\id,\ood)}\|_\F^2,
    \end{align*}
    where we used $x + y + z \leq 3x^2 + 3y^2 + 3z^2$. This concludes the proof.
\end{proof}

\subsection{Proofs for Full Fine-tuning}

Define $f_\ell^\full(x) = \oW^\pre_{\ell+1} (W^\pre_{\ell} + \Delta_\ell^\full) \uW^\pre_{\ell-1} x$ as a fine-tuned network with full fine-tuning applied to the $\ell$-th layer, evaluated under the population in-distribution risk, where $\Delta_\ell^\full$ is obtained by
\begin{align*}
    \Delta_\ell^\full \in \argmin_{\Delta' \in \R^{d_\ell\times d_{\ell-1}}} \E\qty[\qty(B^{(\id)} x^{(\id)} - \oW^\pre_{\ell+1} (W^\pre_{\ell} + \Delta') \uW^\pre_{\ell-1} x^{(\id)})^2].
\end{align*}

\begin{lemma}[In-distribution Excess Risk]\label{lem: excess risk full}
    For $f_\ell^\full$, it holds that
    \begin{align*}
        \mathcal{E}^{(\id)}(f_\ell^\full) &= \|D \Sigma_x^{(\id) 1/2} (I - \Sigma_x^{(\id) 1/2} \uW^{\pre \top}_{\ell-1} (A^2)^\dag \uW^\pre_{\ell-1} \Sigma_x^{(\id) 1/2})\|_\F^2\\
        &\quad+ \|(I - \Phi' \Phi'^\top) D \Sigma_x^{(\id)} \uW^{\pre \top}_{\ell-1} (A^2)^\dag \uW^\pre_{\ell-1} \Sigma_x^{(\id) 1/2}\|_\F^2.
    \end{align*}
\end{lemma}

\begin{proof}[Proof of Lemma~\ref{lem: excess risk full}]
    Similar to the proof of Theorem~\ref{thm: lora in-distribution}, we have
    \begin{align*}
        \mathcal{E}^{(\id)}(f_\ell^\full) &= \min_{\Delta \in \R^{d_\ell\times d_{\ell-1}}} \E\qty[\qty(B^{(\id)} x^{(\id)} - \oW^\pre_{\ell+1} (W^\pre_{\ell} + \Delta) \uW^\pre_{\ell-1} x^{(\id)})^2]\\
        &= \min_{\Delta \in \R^{d_\ell\times d_{\ell-1}}} \|D \Sigma_x^{(\id) 1/2} - \oW^\pre_{\ell+1} \Delta \uW^\pre_{\ell-1} \Sigma_x^{(\id) 1/2}\|_\F^2,
    \end{align*}
    and
    \begin{align}
        \|D \Sigma_x^{(\id) 1/2} - \oW^\pre_{\ell+1} \Delta \uW^\pre_{\ell-1} \Sigma_x^{(\id) 1/2}\|_\F^2
        &= \|\underbrace{\oW^\pre_{\ell+1} \Delta \uW^\pre_{\ell-1} \Sigma_x^{(\id) 1/2} - \Phi' \Phi'^\top D \Sigma_x^{(\id)} \uW^{\pre \top}_{\ell-1} A^\dag}_{=: T_1}\|_\F^2\label{eq: full risk first}\\
        &\quad+ \|\underbrace{D \Sigma_x^{(\id) 1/2} (I - \Sigma_x^{(\id) 1/2} \uW^{\pre \top}_{\ell-1} (A^2)^\dag \uW^\pre_{\ell-1} \Sigma_x^{(\id) 1/2})}_{=: T_2}\|_\F^2\nonumber\\
        &\quad+ \|\underbrace{(I - \Phi' \Phi'^\top) D \Sigma_x^{(\id)} \uW^{\pre \top}_{\ell-1} (A^2)^\dag \uW^\pre_{\ell-1} \Sigma_x^{(\id) 1/2}}_{=: T_3}\|_\F^2,\nonumber
    \end{align}
    where we used the fact that the inner products $\tr(T_1 T_2^\top) = \tr(T_2 T_3^\top) = \tr(T_3 T_1^\top) = 0$.
    By choosing $\Delta = (\oW^\pre_{\ell+1})^\dag D \Sigma_x^{(\id)} \uW^{\pre \top}_{\ell-1} A^\dag$ for example, the term $T_1$ becomes $0$. Thus
    \begin{align*}
        \mathcal{E}^{(\id)}(f_\ell^\full) &= \|D \Sigma_x^{(\id) 1/2} (I - \Sigma_x^{(\id) 1/2} \uW^{\pre \top}_{\ell-1} (A^2)^\dag \uW^\pre_{\ell-1} \Sigma_x^{(\id) 1/2})\|_\F^2\\
        &\quad+ \|(I - \Phi' \Phi'^\top) D \Sigma_x^{(\id)} \uW^{\pre \top}_{\ell-1} (A^2)^\dag \uW^\pre_{\ell-1} \Sigma_x^{(\id) 1/2}\|_\F^2.
    \end{align*}
    This gives the desired result.
\end{proof}

We obtain the following corollary as a direct consequence of Lemma~\ref{lem: excess risk full}.
\begin{corollary}
    For $f_\ell^\full$, it holds that
    \begin{align}
        \mathcal{E}^{(\id)}(f_\ell^\full) &\leq \|\Psi_*^\top(D \Sigma_x^{(\id) 1/2}) (I - \Sigma_x^{(\id) 1/2} \uW^{\pre \top}_{\ell-1} (A^2)^\dag \uW^\pre_{\ell-1} \Sigma_x^{(\id) 1/2})\|_\op \mathcal{E}^{(\id)}(f^\pre)\nonumber\\
        &\quad+ \|(I - \Phi' \Phi'^\top) \Phi_*(D \Sigma_x^{(\id) 1/2})\|_\op \mathcal{E}^{(\id)}(f^\pre).\label{eq: risk full ub}
    \end{align}
\end{corollary}
The first term on the right hand side of \eqref{eq: risk full ub} measures the distance between two subspaces spanned by $\Psi_*(D \Sigma_x^{(\id) 1/2})$ and $\Psi_*(\uW_{\ell-1}^\pre \Sigma_x^{(\id) 1/2})$. 
Intuitively, this quantifies the information coded at the $\ell$-th layer, and the necessary information to predict residuals. Thus, it bounds the maximum improvement by the $\ell$-th layer fine-tuning.
The second term measures the subspace distance between the subspace where prediction residuals reside, and the subspace predictable by the $\ell$-th layer fine-tuning.

\section{Auxiliary Results for Proofs}

\begin{lemma}\label{lem: rank s approx}
    Fix $s, d_1, d_2 \in \N$.
    For any $A, B \in \R^{d_1 \times d_2}$, if $\|B - A\|_\op \leq \|A\|_\op$ and $\lambda_s(A) > \lambda_{s+1}(A)$ hold, then, 
    \begin{align*}
        \|\SVD_s(B) - \SVD_s(A)\|_\F &\lesssim \kappa_*^2(A) \frac{\lambda_s(A)}{\lambda_s(A) - \lambda_{s+1}(A)} \qty(\sqrt{s} \|B - A\|_\op \wedge \|B-A\|_\F).
    \end{align*}
\end{lemma}

\begin{proof}
    By triangle inequality,
    \begin{align*}
        \|\SVD_s(B) - \SVD_s(A)\|_\F &= \|\Phi_s(B) \Phi_s^\top(B) B - \Phi_s(A) \Phi_s^\top(A) A\|_\F\\
        &\leq \|\Phi_s(B) \Phi_s^\top(B) (B - A)\|_\F + \|(\Phi_s(B) \Phi_s^\top(B) - \Phi_s(A) \Phi_s^\top(A)) A\|_\F\\
        &\leq \sqrt{s} \|B - A\|_\op + \|\Phi_s(B) \Phi_s^\top(B) - \Phi_s(A) \Phi_s^\top(A)\|_\F \|A\|_\op.
    \end{align*}
    Using Davis-Kahan theorem (Theorem 4 from \cite{yu2015useful}), and Lemma 2.6 from \cite{chen2021spectral},
    \begin{align*}
        \|\Phi_s(B) \Phi_s^\top(B) - \Phi_s(A) \Phi_s^\top(A)\|_\F &\leq \frac{6\sqrt{2} \|A\|_\op (\sqrt{s} \|B - A\|_\op \wedge \|B-A\|_\F)}{\lambda_s^2(A) - \lambda_{s+1}^2(A)}.
    \end{align*}
    Thus
    \begin{align*}
        \|\SVD_s(B) - \SVD_s(A)\|_\F &\lesssim \frac{\|A\|_\op^2}{\lambda_s^2(A)} \frac{\lambda_s^2(A)}{\lambda_s^2(A) - \lambda_{s+1}^2(A)} (\sqrt{s} \|B - A\|_\op \wedge \|B - A\|_\F)\\
        &\lesssim \frac{\|A\|_\op^2}{\lambda_s^2(A)} \frac{\lambda_s(A)}{\lambda_s(A) - \lambda_{s+1}(A)} (\sqrt{s} \|B - A\|_\op \wedge \|B - A\|_\F).
    \end{align*}
    This concludes the proof.
\end{proof}

We cite the concentration inequality for cross-covariance matrices from \cite{nakada2023understanding}.
\begin{lemma}[Proposition 9.1 from \cite{nakada2023understanding}]\label{lem: cross covariance concentration}
    Let $Z$ and $\tilde Z$ be mean zero random vectors taking values in $\R^{d_1}$ and $\R^{d_2}$, respectively. Denote covariance matrices of $Z$ and $\tilde Z$ by $\Sigma_Z$ and $\Sigma_{\tilde Z}$, respectively.
    Fix any $t > 0$.
    Assume that there exist constants $c_1, c_2 > 0$ such that
    \begin{align}
        \gamma^\top \Sigma_Z \gamma \geq c_1 \|\gamma^\top Z\|_{\psi_2}^2 \ \ \text{ and } \ \ \gamma'^\top \Sigma_{\tilde Z} \gamma' \geq c_2 \|\gamma'^\top \tilde Z\|_{\psi_2}^2\label{eq: good sub gaussian}
    \end{align}
    holds for any $\gamma \in \R^{d_1}$ and $\gamma' \in \R^{d_2}$. 
    Choose $n \gg (r_e(\Sigma_Z) \wedge r_e(\Sigma_{\tilde Z})) (t + \log(d_1 + d_2))$.
    Let $(Z_i, \tilde Z_i)_{i \in [n]}$ be $n$ independent copies of $(Z, \tilde Z)$.
    Then, there exists a constant $C = C(c_1, c_2) > 0$ such that with probability at least $1 - e^{-t}$,
    \begin{align*}
        \norm{\frac{1}{n} \sum_{i\in [n]} Z_i \tilde Z_i^\top - \E[Z \tilde Z^\top]}_\op \leq C \|\Sigma_Z\|_\op^{1/2} \|\Sigma_{\tilde Z}\|_\op^{1/2} \sqrt{\frac{(r_e(\Sigma_Z) + r_e(\Sigma_{\tilde Z}) (t + \log(d_1 + d_2))}{n}}
    \end{align*}
    hold.
\end{lemma}
Note that if a random variable $Z$ taking values in $\R^d$ satisfies $\gamma^\top \Sigma_Z \gamma \geq c \|\gamma^\top Z\|_{\psi_2}^2$ for any $\gamma \in \R^d$ with some constant $c > 0$, $AZ$ also satisfies $\gamma'^\top \Sigma_{AZ} \gamma' \geq c \|\gamma'^\top A Z\|_{\psi_2}^2$ for any $\gamma' \in \R^{d'}$ and any matrix $A \in \R^{d' \times d}$ and arbitrary $d' \in \N$, where $\Sigma_{AZ} = A \Sigma_Z A^\top$.

We then prove the following lemma to show the existance of a `good' high probability event to bound multiple inequalities.
\begin{lemma}\label{lem: good event}
    Suppose that Assumptions~\ref{asm: sub gaussian x e ap} and \ref{asm: regime ap} hold. Fix any $S \subset [d_\ell]$.
    Then, there exists an event $\mathcal{F}$ with $\P(\mathcal{F}) = 1 - \exp(-\Omega(\log^2 (n+p+q)))$ such that on the event $\mathcal{F}$, for $\Phi \in \{\Phi', \Phi_S''\}$,
    \begin{align}
        \|\Phi^\top \hat D \hat \Sigma_x^{(\id)} \uW^{\pre \top}_{\ell-1}\|_\op \lesssim \|D \Sigma_x^{(\id) 1/2}\|_\op \|A\|_\op,\ \ \|\hat A^\dag\|_\op \lesssim \|A^\dag\|_\op,\label{eq: good 1}
    \end{align}
    and
    \begin{align}
        \|(\hat A^2)^\dag - (A^2)^\dag\|_\op &\lesssim \frac{\kappa_*^2(A)}{\lambda_*^2(A)} \sqrt{\frac{r_e(A^2) \log^2(n+p+q)}{n}},\label{eq: good 3}\\
        \|\hat A - A\|_\op &\lesssim \kappa_*^2(A) \|A\|_\op \sqrt{\frac{r_e(A^2) \log^2(n+p+q)}{n}},\label{eq: good 4}\\
        \|\hat A^\dag - A^\dag\|_\op &\lesssim \frac{\kappa_*(A)}{\lambda_*(A)} \sqrt{\frac{r_e(A^2) \log^2(n+p+q)}{n}}\label{eq: good 4 inv}
    \end{align}
    hold.
    Furthermore,
    \begin{align}
        &\|\Phi^\top (\hat D \hat \Sigma_x^{(\id) 1/2} - D \Sigma_x^{(\id) 1/2}) \uW^{\pre \top}_{\ell-1}\|_\op\nonumber\\
        &\quad\lesssim \|\Sigma_\epsilon^{(\id)}\|_\op^{1/2} \|A\|_\op \sqrt{\frac{(r_e(\Phi^\top \Sigma_\epsilon^{(\id)} \Phi) + r_e(A^2)) \log^2(n+p+q)}{n}}\nonumber\\
        &\quad\quad+ \|D \Sigma_x^{(\id)} D^\top\|_\op^{1/2} \|A\|_\op \sqrt{\frac{(r_e(\Phi^\top D \Sigma_x^{(\id)} D^\top \Phi) + r_e(A^2)) \log^2(n+p+q)}{n}}\label{eq: good 2}
    \end{align}
    holds on the event $\mathcal{F}$.
\end{lemma}

\begin{proof}
    We only prove for $\Phi = \Phi'$ without loss of generality.
    Before proving Lemma~\ref{lem: good event}, we first derive several concentration inequalities.
    Assumption~\ref{asm: regime ap} implies
    \begin{align*}
        n &\gg r_e(A^2) \log^2(n+p+q),\\
        n &\gg r_e(\Sigma_x^{(\id)}) \log^2(n+p+q),\\
        n &\gg (r_e(\Sigma_\epsilon^{(\id)}) \wedge r_e(\Sigma_x^{(\id)})) \log^2(n+p+q),\\
        n &\gg (r_e(\Phi^\top \Sigma_\epsilon^{(\id)} \Phi) \wedge r_e(A^2)) \log^2(n+p+q),\\
        n &\gg (r_e(\Phi^\top D \Sigma_x^{(\id)} D^\top \Phi) \wedge r_e(A^2)) \log^2(n+p+q).
    \end{align*}
    Using Lemma~\ref{lem: cross covariance concentration}, we obtain
    \begin{align}
        \|\hat A^2 - A^2\|_\op &= \|\uW^\pre_{\ell-1} \hat\Sigma_x^{(\id)} \uW^{\pre \top}_{\ell-1} - \uW^\pre_{\ell-1} \Sigma_x^{(\id)} \uW^{\pre \top}_{\ell-1}\|_\op\nonumber\\
        &\lesssim \|A\|_\op^2 \sqrt{\frac{r_e(A^2) \log^2(n+p+q)}{n}},\label{eq: w hat sigma w x}
    \end{align}
    and
    \begin{small}
    \begin{align}
        \|\hat \Sigma_{\epsilon,x}^{(\id)}\|_\op &\lesssim \|\Sigma_\epsilon^{(\id)}\|_\op^{1/2} \|\Sigma_x^{(\id)}\|_\op^{1/2} \sqrt{\frac{(r_e(\Sigma_\epsilon^{(\id)}) + r_e(\Sigma_x^{(\id)})) \log^2(n+p+q)}{n}},\label{eq: good part 2 1}\\
        \|\hat \Sigma_x^{(\id)} - \Sigma_x^{(\id)}\|_\op &\lesssim \|\Sigma_x^{(\id)}\|_\op \sqrt{\frac{r_e(\Sigma_x^{(\id)}) \log^2(n+p+q)}{n}},\label{eq: good part 2 2}\\
        \norm{\Phi^\top \hat \Sigma_{\epsilon,x}^{(\id)} (\Sigma_x^{(\id)})^\dag \Sigma_x^{(\id)} \uW^{\pre \top}_{\ell-1}}_\op &\lesssim \|\Sigma_\epsilon^{(\id)}\|_\op^{1/2} \|A\|_\op \sqrt{\frac{(r_e(\Phi^\top \Sigma_\epsilon^{(\id)} \Phi) + r_e(A^2)) \log^2(n+p+q)}{n}},\label{eq: good part 2 3}\\
        \norm{\Phi^\top D (\hat \Sigma_x^{(\id)} - \Sigma_x^{(\id)}) \uW^{\pre \top}_{\ell-1}}_\op &\lesssim \|D \Sigma_x^{(\id)} D^\top\|_\op^{1/2} \|A\|_\op \sqrt{\frac{(r_e(\Phi^\top D \Sigma_x^{(\id)} D^\top \Phi) + r_e(A^2)) \log^2(n+p+q)}{n}},\label{eq: good part 2 4}
    \end{align}
    \end{small}
    with high probability.
    Hereafter we only focus on the event $\mathcal{F}$ where these inequalities hold.
    We divide the proof into $2$ parts.
    
    \paragraph{Part 1.}    
    In this part we derive \eqref{eq: good 3}, \eqref{eq: good 4} and \eqref{eq: good 4 inv}.
    Note that $\|\hat A^2 - A^2\|_\op \leq \lambda_*(A^2)/2$ holds on the event $\mathcal{F}$ since $n \gg \kappa_*^4(A) r_e(A^2) \log^2(n+d+p)$ by Assumption~\ref{asm: regime ap}, and hence $\rank(\hat A^2) = \rank(A^2)$.
    Using Theorem 5.2 from \cite{stewart1969continuity},
    \begin{align*}
        \frac{\|(\hat A^2)^\dag - (A^2)^\dag\|_\op}{\|(A^2)^\dag\|_\op} \lesssim \qty(1 - \frac{\kappa_*(A^2) \|\hat A^2 - A^2\|_\op}{\|A\|_\op^2})^{-1} \frac{\kappa_*(A^2) \|\hat A^2 - A^2\|_\op}{\|A\|_\op^2}.
    \end{align*}
    Again from Assumption~\ref{asm: regime ap}, \eqref{eq: w hat sigma w x} gives
    \begin{align*}
        \|(\hat A^2)^\dag - (A^2)^\dag\|_\op \lesssim \frac{\kappa_*(A^2)}{\lambda_*(A^2)} \sqrt{\frac{r_e(A^2) \log^2(n+p+q)}{n}}.
    \end{align*}
    This yields \eqref{eq: good 3}. 
    Proposition 3.2 from \cite{van1980inequality} and \eqref{eq: w hat sigma w x} yield,
    \begin{align*}
        \|(\Phi'''^\top \hat A^2 \Phi''')^{1/2} - (\Phi'''^\top A^2 \Phi''')^{1/2}\|_\op \leq \frac{\|\Phi'''^\top (\hat A^2 - A^2) \Phi'''\|_\op}{\lambda_*^{1/2}(\Phi'''^\top A^2 \Phi''')} \lesssim \frac{\|A\|_\op^2}{\lambda_*(A)} \sqrt{\frac{r_e(A^2) \log^2(n+p+q)}{n}},
    \end{align*}
    where $\Phi''' := \Phi_*(A^2)$, and we used $\lambda_*(\Phi'''^\top A^2 \Phi''') \geq \lambda_*(A^2)$. 
    Since $\hat A = \Phi''' (\Phi'''^\top \hat A^2 \Phi''')^{1/2} \Phi'''^\top$ and $A^{1/2} = \Phi''' (\Phi'''^\top A^2 \Phi''')^{1/2} \Phi'''^\top$, we obtain \eqref{eq: good 4} as
    \begin{align}
        \|\hat A - A\|_\op \lesssim \kappa_*(A) \|A\|_\op \sqrt{\frac{r_e(A^2) \log^2(n+p+q)}{n}}.\label{eq: hat sigma x 1/2}
    \end{align}
    Again using Theorem 5.2 from \cite{stewart1969continuity} combined with 
    Assumption~\ref{asm: regime ap}, we obtain \eqref{eq: good 4 inv} as
    \begin{align*}
        \|\hat A^\dag - A^\dag\|_\op &\lesssim \frac{\kappa_*^2(A)}{\lambda_*(A)} \sqrt{\frac{r_e(A^2) \log^2(n+p+q)}{n}}.
    \end{align*}
    This yields $\|\hat A^\dag\|_\op \lesssim \|A^\dag\|_\op$.
    
    \paragraph{Part 2.}
    Next we derive \eqref{eq: good 2}.
    By a similar argument as Part 1, \eqref{eq: good part 2 2} and Assumption~\ref{asm: regime ap},
    \begin{align}
        \|(\hat \Sigma_x^{(\id)})^\dag - (\Sigma_x^{(\id)})^\dag\|_\op &\lesssim \frac{\|\Sigma_x^{(\id)}\|_\op}{\lambda_*^2(\Sigma_x^{(\id)})} \sqrt{\frac{r_e(\Sigma_x^{(\id)}) \log^2(n+d+p)}{n}}.\label{eq: good part 2 3b}
    \end{align}
    Since $\hat D - D = \check \Sigma_{\epsilon,x}^{(\id)} = \hat \Sigma_{\epsilon,x}^{(\id)} (\hat \Sigma_x^{(\id)})^\dag$,
    \begin{align*}
        &\|\Phi^\top (\hat D \hat \Sigma_x^{(\id)} - D \Sigma_x^{(\id)}) \uW^{\pre \top}_{\ell-1}\|_\op\\
        &\quad\leq \norm{\Phi^\top (\hat D - D) \Sigma_x^{(\id)} \uW^{\pre \top}_{\ell-1}}_\op + \norm{\Phi^\top D (\hat \Sigma_x^{(\id)} - \Sigma_x^{(\id)}) \uW^{\pre \top}_{\ell-1}}_\op + \norm{\Phi^\top (\hat D - D) (\hat \Sigma_x^{(\id)} - \Sigma_x^{(\id)}) \uW^{\pre \top}_{\ell-1}}_\op\\
        &\quad= \norm{\Phi^\top \hat \Sigma_{\epsilon,x}^{(\id)} (\hat \Sigma_x^{(\id)})^\dag \Sigma_x^{(\id)} \uW^{\pre \top}_{\ell-1}}_\op + \norm{\Phi^\top D (\hat \Sigma_x^{(\id)} - \Sigma_x^{(\id)}) \uW^{\pre \top}_{\ell-1}}_\op + \norm{\Phi^\top \hat \Sigma_{\epsilon,x}^{(\id)} (\hat \Sigma_x^{(\id)})^\dag (\hat \Sigma_x^{(\id)} - \Sigma_x^{(\id)}) \uW^{\pre \top}_{\ell-1}}_\op\\
        &\quad\leq \norm{\Phi^\top \hat \Sigma_{\epsilon,x}^{(\id)} (\Sigma_x^{(\id)})^\dag \Sigma_x^{(\id)} \uW^{\pre \top}_{\ell-1}}_\op + \norm{\Phi^\top \hat \Sigma_{\epsilon,x}^{(\id)} \qty((\Sigma_x^{(\id)})^\dag \Sigma_x^{(\id)} - (\hat \Sigma_x^{(\id)})^\dag \Sigma_x^{(\id)}) \uW^{\pre \top}_{\ell-1}}_\op\\
        &\quad\quad+ \norm{\Phi^\top D (\hat \Sigma_x^{(\id)} - \Sigma_x^{(\id)}) \uW^{\pre \top}_{\ell-1}}_\op + \norm{\Phi^\top \hat \Sigma_{\epsilon,x}^{(\id)} (\hat \Sigma_x^{(\id)})^\dag (\hat \Sigma_x^{(\id)} - \Sigma_x^{(\id)}) \uW^{\pre \top}_{\ell-1}}_\op\\
        &\quad=: Q_1 + R_1 + Q_2 + R_2.
    \end{align*}
    We bound $Q_1$, $Q_2$, $R_1$ and $R_2$ separately.
    For the terms $Q_1$ and $Q_2$, \eqref{eq: good part 2 3} and \eqref{eq: good part 2 4} give
    \begin{align}
        Q_1 &\lesssim \|\Sigma_\epsilon^{(\id)}\|_\op^{1/2} \|A\|_\op \sqrt{\frac{(r_e(\Phi^\top \Sigma_\epsilon^{(\id)} \Phi) + r_e(A^2)) \log^2(n+p+q)}{n}},\label{eq: Q1}\\
        Q_2 &\lesssim \|D \Sigma_x^{(\id)} D^\top\|_\op^{1/2} \|A\|_\op \sqrt{\frac{(r_e(\Phi^\top D \Sigma_x^{(\id)} D^\top \Phi) + r_e(A^2)) \log^2(n+p+q)}{n}}.\label{eq: Q2}
    \end{align}
    For the term $R_1$, using \eqref{eq: good part 2 1} and \eqref{eq: good part 2 3b},
    \begin{align*}
        R_1 &\leq \|\hat \Sigma_{\epsilon,x}^{(\id)}\|_\op \|(\Sigma_x^{(\id)})^\dag - (\hat \Sigma_x^{(\id)})^\dag\|_\op \|\Sigma_x^{(\id)}\|_\op^{1/2} \|\Sigma_x^{(\id) 1/2} \uW^{\pre \top}_{\ell-1}\|_\op\\
        &\lesssim \frac{\|\Sigma_x^{(\id)}\|_\op^2 \|\Sigma_\epsilon^{(\id)}\|_\op^{1/2} \|A\|_\op}{\lambda_*^2(\Sigma_x^{(\id)})} \sqrt{\frac{(r_e(\Sigma_\epsilon^{(\id)}) + r_e(\Sigma_x^{(\id)})) \log^2(n+p+q)}{n}} \sqrt{\frac{r_e(\Sigma_x^{(\id)}) \log^2(n+p+q)}{n}}\\
        &\lesssim \kappa_*^2(\Sigma_x^{(\id)}) \|\Sigma_\epsilon^{(\id)}\|_\op^{1/2} \|A\|_\op \frac{\sqrt{r_e(\Sigma_x^{(\id)}) (r_e(\Sigma_\epsilon^{(\id)}) + r_e(\Sigma_x^{(\id)}))} \log^2(n+p+q)}{n}
    \end{align*}
    For the term $R_2$, using \eqref{eq: good part 2 1} and \eqref{eq: good part 2 2}, 
    \begin{align*}
        R_2 &\leq \|\hat \Sigma_{\epsilon,x}^{(\id)}\|_\op \|(\hat \Sigma_x^{(\id)})^\dag\|_\op \|\hat \Sigma_x^{(\id)} - \Sigma_x^{(\id)}\|_\op \|(\Sigma_x^{(\id)})^\dag\|_\op^{1/2} \|\Sigma_x^{(\id) 1/2} \uW^{\pre \top}_{\ell-1}\|_\op\\
        &\lesssim \|(\Sigma_x^{(\id)})^\dag\|_\op^{3/2} \|A\|_\op \|\Sigma_\epsilon^{(\id)}\|_\op^{1/2} \|\Sigma_x^{(\id)}\|_\op^{3/2} \sqrt{\frac{(r_e(\Sigma_\epsilon^{(\id)}) + r_e(\Sigma_x^{(\id)})) \log^2(n+p+q)}{n}} \sqrt{\frac{r_e(\Sigma_x^{(\id)})  \log^2(n+p+q)}{n}}\\
        &\lesssim \kappa_*^{3/2}(\Sigma_x^{(\id)}) \|\Sigma_\epsilon^{(\id)}\|_\op^{1/2} \|A\|_\op \frac{\sqrt{r_e(\Sigma_x^{(\id)}) (r_e(\Sigma_\epsilon^{(\id)}) + r_e(\Sigma_x^{(\id)}))} \log^2(n+p+q)}{n},
    \end{align*}
    where we used $\|(\hat \Sigma_x^{(\id)})^\dag\|_\op \lesssim \|(\Sigma_x^{(\id)})^\dag\|_\op$ by Assumption~\ref{asm: regime ap} combined with \eqref{eq: good part 2 3b}.
    Again from Assumption~\ref{asm: regime ap}, $R_1 + R_2$ is bounded by the right hand side of \eqref{eq: Q1}. Therefore,
    \begin{align*}
        &\|\Phi^\top (\hat D \hat \Sigma_x^{(\id)} - D \Sigma_x^{(\id)}) \uW^{\pre \top}_{\ell-1}\|_\op\\
        &\quad\lesssim \|\Sigma_\epsilon^{(\id)}\|_\op^{1/2} \|A\|_\op \sqrt{\frac{(r_e(\Phi^\top \Sigma_\epsilon^{(\id)} \Phi) + r_e(A^2)) \log^2(n+p+q)}{n}}\\
        &\quad\quad+ \|D \Sigma_x^{(\id)} D^\top\|_\op^{1/2} \|A\|_\op \sqrt{\frac{(r_e(\Phi^\top D \Sigma_x^{(\id)} D^\top \Phi) + r_e(A^2)) \log^2(n+p+q)}{n}}.
    \end{align*}
    Finally, from Assumption~\ref{asm: regime ap}, we obtain $\|\Phi^\top \hat D \hat \Sigma_x^{(\id)} \uW^{\pre \top}_{\ell-1}\|_\op \lesssim \|D \Sigma_x^{(\id) 1/2}\|_\op \|\Sigma_x^{(\id) 1/2} \uW^{\pre \top}_{\ell-1}\|_\op$.
    This concludes the proof.
\end{proof}